\documentclass[runningheads]{llncs}

\usepackage{eccv}

\usepackage{eccvabbrv}

\definecolor{dblue}{RGB}{52, 104, 192}
\definecolor{dgreen}{RGB}{65, 171, 93}
\definecolor{dred}{RGB}{210, 69, 69}
\definecolor{dred2}{RGB}{169, 68, 56}
\definecolor{lavender}{HTML}{E6E6FA}

\usepackage{url}
\usepackage{booktabs}
\usepackage{times}
\usepackage{soul}
\usepackage[numbers]{natbib}

\usepackage{nicefrac}
\usepackage{microtype}
\usepackage{xcolor}
\usepackage{amsfonts}
\usepackage{amsmath}
\usepackage{amssymb}
\usepackage{dsfont}
\usepackage{arydshln}

\usepackage{graphicx}
\usepackage{colortbl}
\usepackage{tabulary}
\usepackage{adjustbox}
\usepackage{multirow}       

\usepackage{makecell}
\usepackage{caption}
\usepackage{subcaption}
\usepackage{enumitem}
\usepackage{pifont}
\usepackage{lipsum}

\usepackage{wrapfig}
\usepackage{stackengine}
\usepackage{comment}
\usepackage{mathtools}
\usepackage[flushleft]{threeparttable}
\usepackage[normalem]{ulem}
\usepackage[ruled]{algorithm2e}
\usepackage{cuted}

\usepackage[accsupp]{axessibility}  

\usepackage{bm, dsfont}

\let\tilde\widetilde

\newcommand{\bh}{\bm{h}}

\newcommand{\bw}{\bm{w}}
\newcommand{\bx}{\bm{x}}

\newcommand{\bz}{\bm{z}}

\newcommand{\Ib}{\mathbf{I}}

\newcommand{\Rb}{\mathbf{R}}

\newcommand{\bU}{\bm{U}}

\newcommand{\cD}{\mathcal{D}}

\newcommand{\cL}{\mathcal{L}}

\newcommand{\cN}{\mathcal{N}}

\newcommand{\EE}{\mathbb{E}}

\newcommand{\RR}{\mathbb{R}}

\newcommand{\bmu}{\bm{\mu}}

\newcommand{\bSigma}{\bm{\Sigma}}

\newcommand{\tr}{\mathop{\mathrm{tr}}}

\makeatletter
\newcommand{\ve}{\@ifnextchar\bgroup{\velong}{{\bm{e}}}}
\newcommand{\velong}[1]{{\bm{#1}}}
\makeatother

\definecolor{tabgray}{HTML}{f2f2f2}
\definecolor{tablavender}{HTML}{E6F7E6}
\definecolor{tabnavy}{HTML}{00008B}
\definecolor{tabwhite}{HTML}{FFFFFF}
\definecolor{linf}{HTML}{8172b2}
\definecolor{l2}{HTML}{ccb974}
\definecolor{l1}{HTML}{64b5cd}

\newcommand{\std}[1]{{\footnotesize$\pm$}{\scriptsize$#1$}}

\definecolor{lavender}{HTML}{9ad756}
\definecolor{DBlue}{RGB}{31, 119, 180}
\definecolor{DGreen}{RGB}{44,160,44}
\definecolor{DRed}{RGB}{214,39,40}
\definecolor{dorange}{RGB}{255, 127, 14}
\definecolor{dorange2}{RGB}{255, 165, 0}

\newcommand{\fold}{\textsc{Fold}\xspace}
\newcommand{\foldr}{\textsc{Fold-r}\xspace}
\newcommand{\folda}{\textsc{Fold-a}\xspace}
\newcommand{\autofold}{\textsc{AutoFold}\xspace}

\SetKwInput{KwInput}{Input}

\NewDocumentCommand{\alignednum}{m o}{%
  \makebox[6em][c]{%
    \makebox[0pt][r]{#1}%
    \makebox[0pt][l]{%
        \IfValueTF{#2}
        {\std{#2}}
        {\textcolor{black!60}{\std{0.0}}}
    }
  }
}

\usepackage{hyperref}

\usepackage{orcidlink}

\begin{document}
\title{Exploiting Local Flatness for Efficient Out-of-Distribution Detection} 

\titlerunning{Exploiting Local Flatness for Efficient Out-of-Distribution Detection}

\author{Seonghwan Park\inst{1, 2}\orcidlink{0009-0009-6322-6959} \and
Hyunji Jung\inst{2}\orcidlink{0009-0002-6062-0642} \and
Dongyeop Lee\inst{2}\orcidlink{0009-0000-2164-4508} \and
Namhoon Lee\inst{2}\orcidlink{0009-0001-5208-2007}}

\authorrunning{Park et al.}

\institute{Korea Electronics Technology Institute (KETI) \and
Pohang University of Science and Technology (POSTECH)\\
\email{\{seonghwan.park, namhoon.lee\}@postech.ac.kr}}

\maketitle

\vspace{-1em}
\begin{abstract}
    Detecting out-of-distribution (OOD) data is crucial for reliable machine learning deployment.
    Among detection strategies, post-hoc methods are particularly attractive due to their efficiency, as they operate directly on pre-trained networks without requiring retraining.
    Within this paradigm, one promising direction exploits loss-landscape curvature to estimate model uncertainty; however, such methods incur substantial computational cost and rely on implicit assumptions about how landscape flatness differs between in-distribution (ID) and OOD data.
    In this work, we provide the first systematic investigation of this curvature discrepancy and show that OOD inputs exhibit larger Hessian curvature than ID data, with the gap widening under stronger distributional shifts.
    Motivated by these observations, we propose \fold, a lightweight flatness-modulated OOD detector that leverages the feature Hessian and partial feature normalization to improve ID-OOD separability while avoiding costly parameter-space curvature approximations.
    To optimally adapt this normalization across diverse datasets, we further introduce \autofold, a self-supervised tuning scheme that synthesizes pseudo-OOD samples via ID logit masking for automatic calibration without requiring external data.
    Experiments on OOD benchmarks show that \fold outperforms prior methods, improving the average AUROC by 1.63\% and reducing FPR95 by 2.30\%, while maintaining computational efficiency comparable to a standard forward pass.
    Supported by theoretical analysis and extensive ablations, \fold provides a principled and practical solution for robust real-world deployment.
    \vspace{-0.5em}
\end{abstract}
\begin{figure}[t]
    \centering
    \begin{subfigure}{0.49\linewidth}
        \centering
        \includegraphics[width=\linewidth]{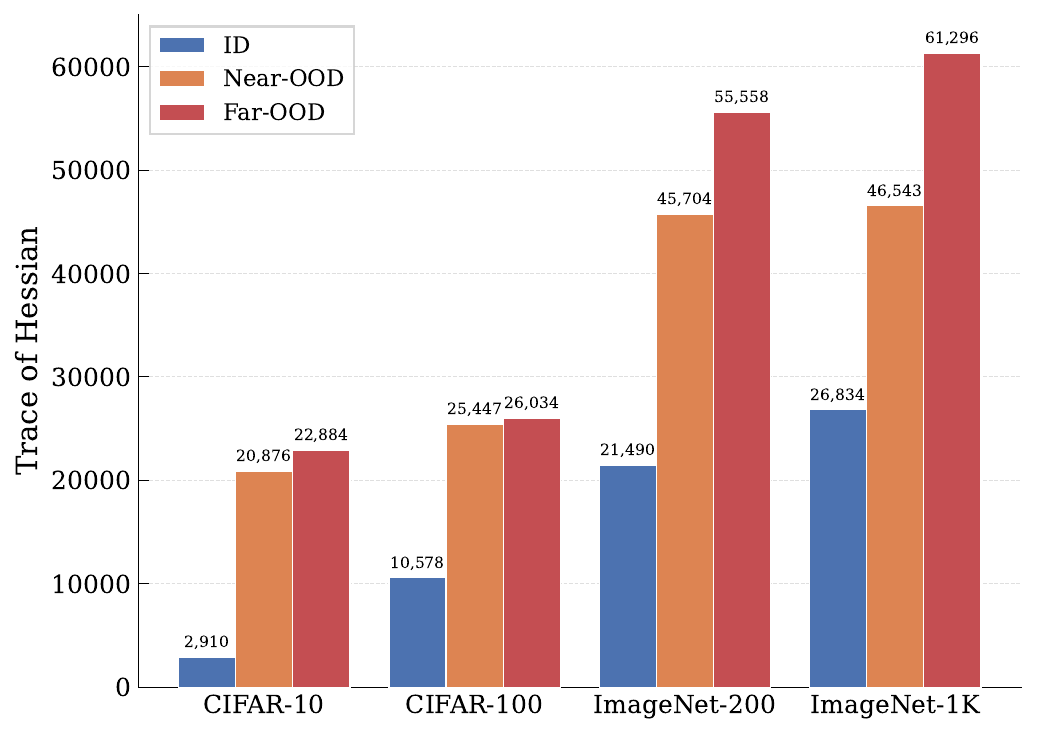}
        \subcaption{Sharpness across distribution shifts}
        \label{fig:pitch 1}
    \end{subfigure}
    \hfill
    \begin{subfigure}{0.49\linewidth}
        \centering
        \includegraphics[width=\linewidth]{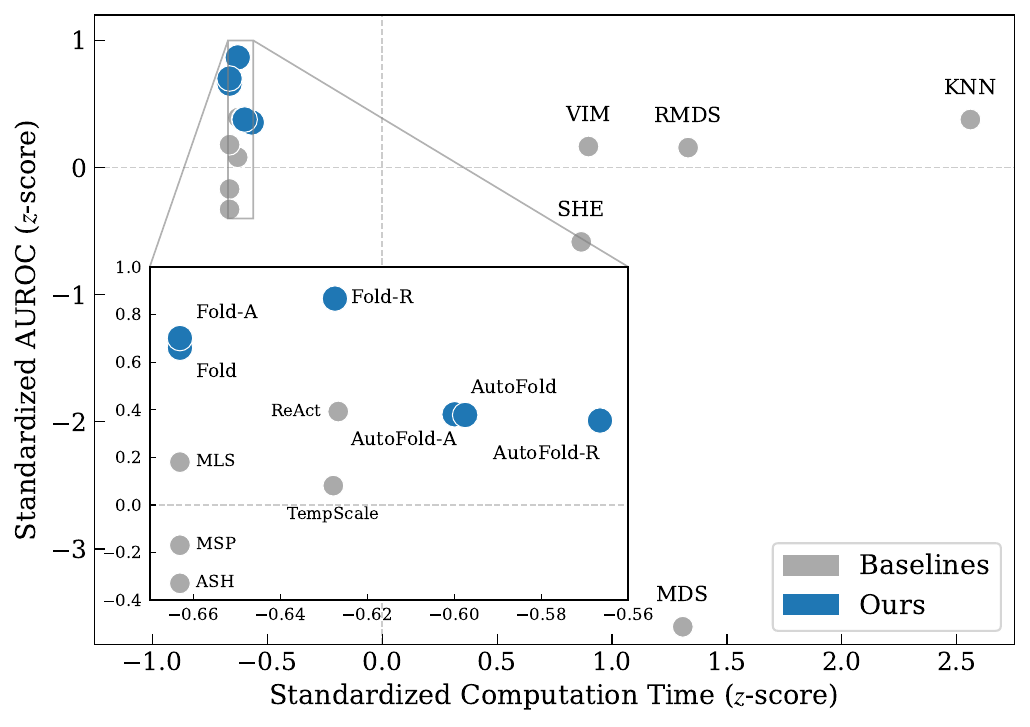}
        \subcaption{Performance-efficiency trade-off}
        \label{fig:pitch 2}
    \end{subfigure}
    \vspace{-0.75em}
    \caption{
        (a) The trace of Hessian, serving as a metric for loss landscape sharpness, progressively increases as the data shifts from ID to near-OOD and far-OOD. 
        (b) The trade-off between detection performance (\ie, standardized AUROC) and computational cost (\ie, standardized computation time). 
        The top-left corner represents the ideal region of high accuracy and low latency. 
        Our proposed \fold and its automatically tuned variant, \autofold, consistently occupy this optimal region, demonstrating outstanding efficiency and performance compared to the baselines.
        }
    \label{fig:pitch performance}
    \vspace{-1.5em}
\end{figure}

\vspace{-2em}
\section{Introduction}
\vspace{-0.5em}
Real-world machine learning systems inevitably encounter out-of-distribution (OOD) data, \ie, samples drawn from distributions whose label sets are disjoint from those seen during training.
Since the model has no prior knowledge of these unseen classes, it should ideally abstain from prediction or express high uncertainty.
However, modern deep neural networks often produce erroneous predictions on such inputs with unwarranted overconfidence~\citep{szegedy2013intriguing, moosavi2017universal, MSP, nguyen2015deep}.
This vulnerability makes reliable OOD identification a critical challenge and has motivated extensive research on detection methods.

Among these efforts, \emph{post-hoc} detection methods have emerged as an appealing paradigm due to their computational efficiency and practical versatility.
Rather than requiring model retraining to suppress overconfidence, these approaches extract discriminative signals directly from pre-trained networks~\citep{OutlierExposure, lee2018training, vyas2018out, devries2018learning, dhamija2018reducing}.
Driven by these advantages, a rich body of work has explored post-hoc OOD detection along multiple directions, including output-level uncertainty estimation~\citep{godin20cvpr, lakshminarayanan2017simple, MDS, ODIN, Energy, tempscale, RMDS}, informative intermediate activations~\citep{sun2021react, ASH}, and neuron- or gradient-level statistics~\citep{NAC, GradNorm, gradientpruningood}.

Within this line of work, recent studies on post-hoc uncertainty estimation exploit the curvature of a pre-trained model’s loss landscape near its optimized parameters~\citep{madras2019detecting, ritter2018scalable}.
Through the Laplace approximation, second-order parameter sensitivity is used to approximate the otherwise intractable Bayesian posterior, enabling a principled estimate of epistemic uncertainty for a fixed model.
However, computing the exact curvature matrix scales quadratically with the number of parameters, making it infeasible for modern deep networks. 
Although several scalable approximations have been proposed \citep{sharma2021sketching, nguyen2026fisher, kristiadi2020being}, they remain computationally demanding and difficult to deploy in real-time settings.
Furthermore, these approaches implicitly rely on the \textit{assumption} that the loss landscape induced by in-distribution (ID) samples is flatter than that of OOD data.

In this work, we provide the first systematic investigation of this ID-OOD curvature discrepancy, culminating in a simple, computationally efficient method that exploits these structural properties.

Specifically, our empirical analysis reveals that key Hessian statistics (\eg, the principal eigenvalue and the trace), are consistently smaller for ID data than for OOD samples (see~\Cref{fig:pitch 1}). 
Furthermore, this curvature discrepancy grows with the severity of the distributional shift, becoming increasingly pronounced from semantically near- to far-OOD regimes.
Motivated by these observations, we introduce \fold, a lightweight \underline{F}latness-m\underline{o}du\underline{l}ated OOD \underline{d}etection algorithm that leverages the \textit{feature Hessian} for OOD discrimination while avoiding the prohibitive computational overhead typical of Bayesian curvature approximations.
Our approach incorporates \textit{partial feature normalization} as a simple structural modification that significantly amplifies the curvature gap between ID and OOD representations.
To further mitigate hyperparameter sensitivity without relying on external data, we introduce AutoFold, a self-supervised tuning strategy that generates pseudo-OOD samples via logit masking using only ID data to enable automatic and robust hyperparameter selection.
Extensive experiments demonstrate that Fold achieves highly competitive detection performance while maintaining the computational efficiency required for practical real-world deployment (see~\Cref{fig:pitch 2}).

\noindent Our key contributions are summarized as follows:
\begin{description}[font={\tiny$\bullet$}~~\normalfont, leftmargin=1.3em, labelindent=0.5em, labelsep=1em, topsep=0.3em, itemsep=0.3em]
    \item[\textbf{Systematic curvature analysis}]
    We present the first systematic analysis of loss landscape curvature, showing Hessian statistics are smaller for ID than for OOD data.
    
    \item[\textbf{Efficient detection method}]
    We propose \fold, a lightweight OOD detection algorithm leveraging the \textit{feature Hessian} and \textit{partial feature normalization} to enhance ID–OOD separability without the overhead of Bayesian curvature-based methods.
    
    \item[\textbf{Automatic hyperparameter tuning}]
    We introduce \autofold, a self-supervised tuning strategy that synthesizes pseudo-OOD samples from ID data via logit masking, enabling automatic and reliable hyperparameter selection without external data.
    
    \item[\textbf{Theoretical and empirical validation}]
    We provide theoretical insights and extensive experiments demonstrating that \fold achieves strong detection performance with substantial computational efficiency across diverse architectures.
\end{description}
\vspace{-0.6em}
\section{Observations}
\label{sec:observations}
\vspace{-0.5em}
In this section, we present a systematic empirical analysis of loss landscape curvature induced by ID and OOD samples. 
We first describe our experimental protocol based on scalable Hessian estimation for computing key curvature metrics, and then examine the geometric discrepancy between the two distributions at both the dataset and sample levels, establishing the empirical foundation for our approach.
\vspace{-0.5em}

\subsection{Experimental Setup}
\label{subsec:experimental setup}
\textbf{Curvature analysis protocol.}\quad
We employ two metrics to characterize the curvature of the loss landscape, where $H$ is the Hessian of the loss: the largest eigenvalue $\lambda_{\mathrm{max}}(H)$, capturing worst-case curvature, and the trace $\mathrm{tr}(H)$, reflecting average curvature~\citep{yao2020pyhessian}. 
Both quantities are estimated using Hutchinson’s method, which relies on Hessian–vector products to avoid explicit construction of the full Hessian matrix~\citep{hutchinson1989stochastic, Hutchinson1, Hutchinson2}. 
We analyze curvature from two complementary perspectives. 
First, to examine our hypothesis that flatness differs across distributions, we compute Hessian statistics of the average loss over the entire dataset, which we refer to as the \textit{dataset-level curvature}. 
Second, since practical OOD detection requires identifying distribution shifts per sample, we investigate whether the \textit{sample-level curvature} provides a reliable signal for this purpose.

\vspace{1em}\noindent\textbf{Experimental details.}\quad
We adopt the OpenOOD framework~\citep{yang2022openood,zhang2023openood}, which categorizes evaluation datasets into near- and far-OOD according to dataset-specific criteria. 
For CIFAR benchmarks~\citep{CIFAR-10}, this distinction is defined strictly by image content and semantic similarity: near-OOD datasets~\citep{CIFAR-10, tin} share object-level similarities with the ID data, whereas far-OOD datasets~\citep{deng2012mnist, svhn, zhou2017places, cimpoi2014describing} contain numerical digits, texture patterns, or scene imagery that differ substantially in both semantics and low-level statistics~\citep{ahmed2020detecting}. 
In contrast, the large visual diversity of ImageNet~\citep{deng2009imagenet} classes makes such semantic categorization less clear; 
thus the near-OOD~\citep{SSB-Hard, ninco} and far-OOD~\citep{van2018inaturalist, cimpoi2014describing, openimages} splits are determined empirically based on detection difficulty.
For all experiments, we use a pretrained ResNet-50 for ImageNet-1K and ResNet-18 for CIFAR-10, CIFAR-100, and ImageNet-200; 
results are averaged over three runs except for ImageNet-1K, where we use the single provided pretrained checkpoint~\citep{zhang2023openood}. 
Additional implementation details and dataset specifications are provided in Appendix~\ref{app:experimental details}.

\begin{table*}[t!]
    \centering
    \caption{
        Comparison of the largest eigenvalue and the Hessian trace across distribution regimes (\eg, ID, near-OOD, far-OOD).
        The reported values are averaged over datasets within each regime.
        ID samples exhibit lower curvature, as reflected by smaller largest eigenvalues and Hessian traces, whereas both metrics increase progressively from ID to near-OOD and further to far-OOD.
        This trend indicates heightened model sensitivity under increasing degrees of distributional shift.
        }
    \label{tab:full parameter hessian statistics}
    \vspace{-0.75em}
    \resizebox{\linewidth}{!}{%
        \begin{sc}
        \begin{tabular}{l @{\hspace{1.5em}} c @{\hspace{1em}} c @{\hspace{1.5em}} c @{\hspace{1em}} c @{\hspace{1.5em}} c @{\hspace{1em}} c @{\hspace{1.5em}} c @{\hspace{1em}} c}
            \toprule
            & \multicolumn{2}{c}{\textbf{CIFAR-10}} & \multicolumn{2}{c}{\textbf{CIFAR-100}} & \multicolumn{2}{c}{\textbf{ImageNet-200}} & \multicolumn{2}{c}{\textbf{ImageNet-1K}} \\ 
            \cmidrule(lr){2-3} \cmidrule(lr){4-5} \cmidrule(lr){6-7} \cmidrule(lr){8-9}
            \textbf{Dataset} &  $\lambda_{\max}(H)$ &  $\tr(H)$ &  $\lambda_{\max}(H)$ &  $\tr(H)$ &  $\lambda_{\max}(H)$ &  $\tr(H)$ &  $\lambda_{\max}(H)$ &  $\tr(H)$ \\ 
            \midrule
            
            ID
            & \alignednum{202.8}[23.0] & \alignednum{2910.5}[717.5] & \alignednum{316.9}[3.4] & \alignednum{10577.8}[751.7] & \alignednum{548.7}[40.2] & \alignednum{21489.6}[2558.6] & \alignednum{702.4} & \alignednum{26834.2} \\
            
            Near-OOD
            & \alignednum{1228.1}[65.2] & \alignednum{20875.7}[917.9] & \alignednum{681.7}[6.7] & \alignednum{25447.0}[704.3] & \alignednum{941.6}[24.3] & \alignednum{45703.7}[2327.0] & \alignednum{1170.5} & \alignednum{46543.2} \\
            
            Far-OOD
            & \alignednum{3417.3}[443.2] & \alignednum{22883.5}[1032.5] & \alignednum{2347.6}[34.3] & \alignednum{26033.5}[709.3] & \alignednum{1147.9}[11.7] & \alignednum{55558.4}[1435.6] & \alignednum{1971.9} & \alignednum{61296.1} \\
            \bottomrule
        \end{tabular}
        \end{sc}
    }
\end{table*}

\begin{figure*}[t]
    \vspace{-0.75em}
    \centering
    \begin{subfigure}{0.245\linewidth}
        \centering
        \includegraphics[width=\linewidth]{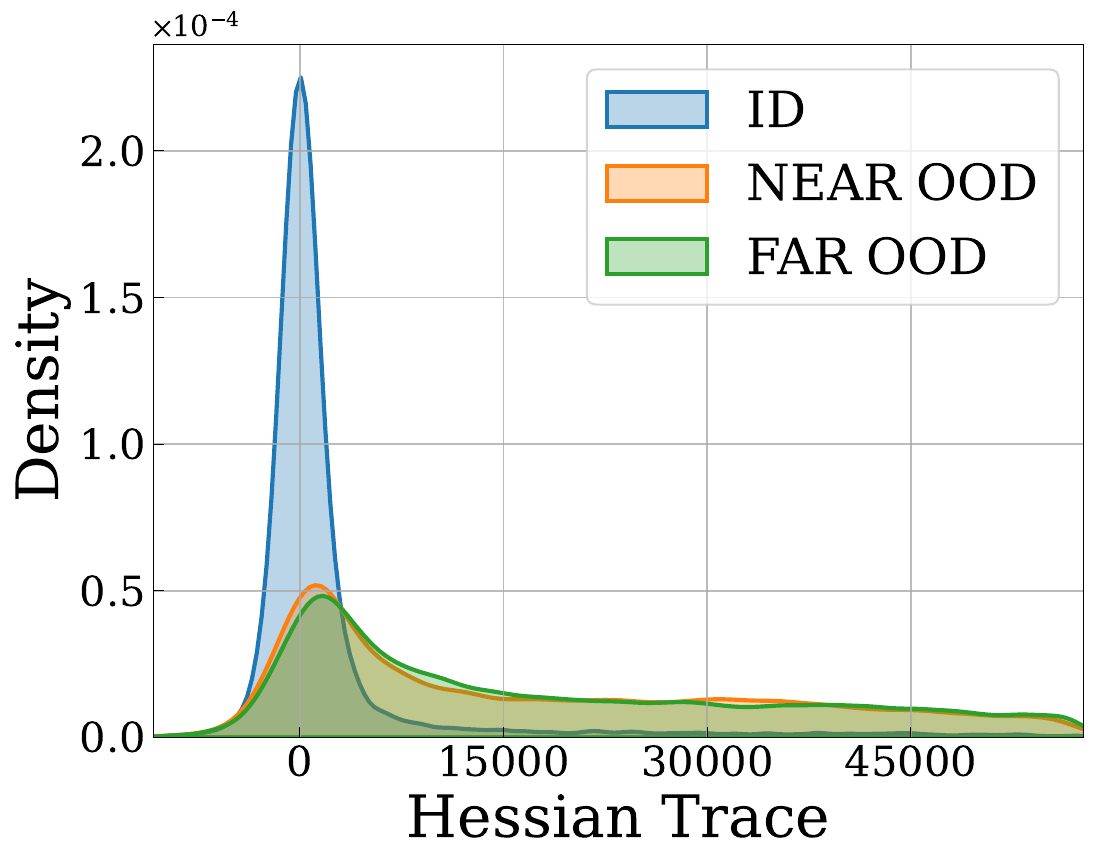}
        \subcaption{CIFAR-10}
    \end{subfigure}
    \begin{subfigure}{0.245\linewidth}
        \centering
        \includegraphics[width=\linewidth]{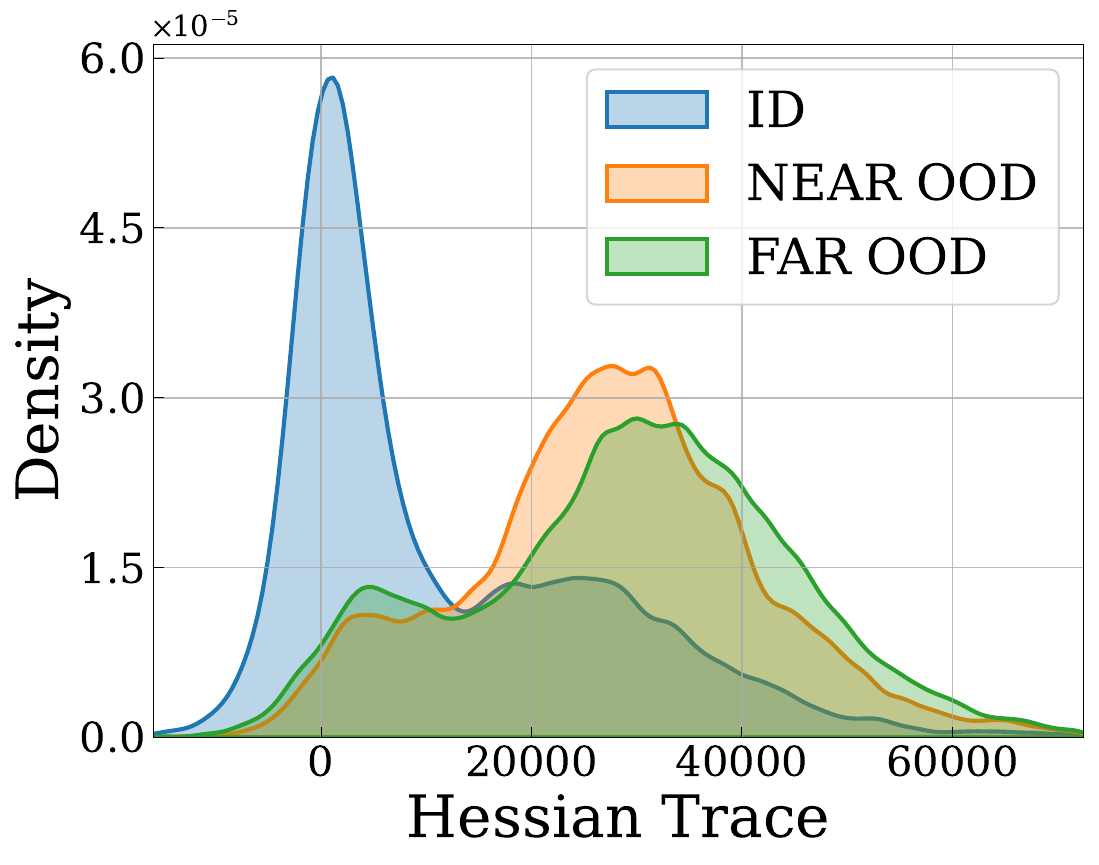}
        \subcaption{CIFAR-100}
    \end{subfigure}
    \begin{subfigure}{0.245\linewidth}
        \centering
        \includegraphics[width=\linewidth]{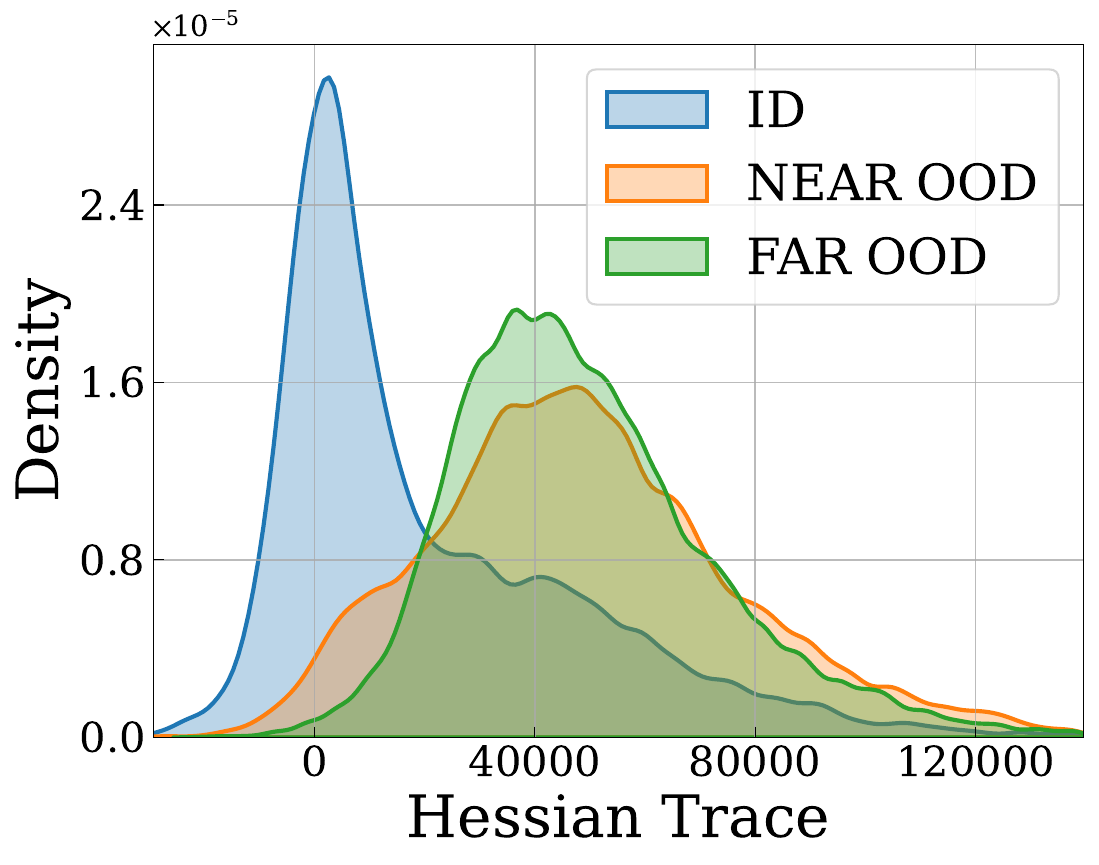}
        \subcaption{ImageNet-200}
    \end{subfigure}
    \begin{subfigure}{0.245\linewidth}
        \centering
        \includegraphics[width=\linewidth]{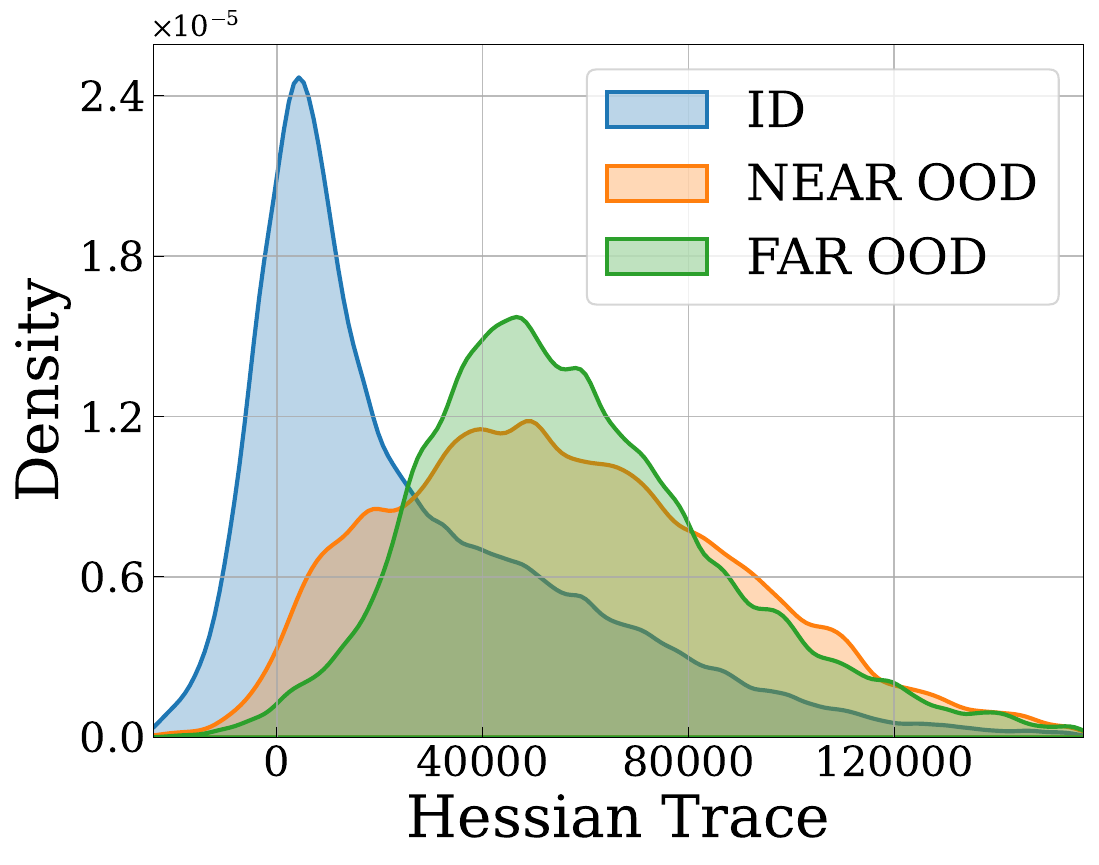}
        \subcaption{ImageNet-1K}
    \end{subfigure}
    \vspace{-1.75em}
    \caption{
        Distribution of per-sample Hessian trace for ID and OOD data, encompassing both near- and far-OOD settings.
        For models trained on CIFAR, we use TIN~\citep{tin} as the near-OOD dataset and SVHN~\citep{svhn} as the far-OOD dataset, while for models trained on ImageNet, we use NINCO~\citep{ninco} and iNaturalist~\citep{van2018inaturalist} as the near- and far-OOD datasets, respectively.
        Across datasets, ID samples are concentrated at small values, whereas OOD samples shift toward larger values, exhibit broader distributions, and often display multimodality, reflecting separability under distribution shift.
        }
    \label{fig:sample-wise hessian}
    \vspace{-1.25em}
\end{figure*}

\subsection{Results}
\label{subsec:results}
\vspace{-0.75em}
\textbf{Dataset-level curvature.}\quad
We begin by comparing dataset-level curvature between ID and OOD inputs.
As shown in \Cref{tab:full parameter hessian statistics}, ID samples consistently exhibit lower curvature than OOD samples across all datasets, with the gap generally widening as the distributional shift becomes more severe.
For instance, when trained on CIFAR-100, the largest eigenvalue for near- and far-OOD inputs exceeds that of ID inputs by factors of $2.2$ and $7.4$, respectively.
For CIFAR-10, the differences are even larger, with corresponding factors of $6.1$ and $16.9$.
A similar pattern is observed for the Hessian trace: in ImageNet-200 and ImageNet-1K, the trace for near-OOD inputs is $2.1$ and $1.7$ times higher than that of ID data, respectively.
Overall, these results indicate that curvature increases with stronger distributional shifts, reflecting heightened sensitivity of the model parameters.

\vspace{0.75em}\noindent\textbf{Sample-level curvature.}\quad
Next, we examine curvature at the sample level.
\Cref{fig:sample-wise hessian} shows the distributions of per-sample Hessian trace values.
Across datasets, ID samples are clearly separated from OOD samples.
For CIFAR-10, the ID density is concentrated near zero, whereas OOD samples shift toward larger values, with far-OOD exhibiting the heaviest tail.
On CIFAR-100, the two OOD distributions exhibit pronounced bimodality, reflecting distinct curvature regimes under increasing distributional shift.
A similar pattern is observed for ImageNet-trained models, where ID samples consistently produce smaller Hessian traces with narrower dispersion than OOD samples.

Overall, the sample-level analysis corroborates the dataset-level observations: curvature increases and becomes more dispersed as the distributional mismatch grows. 
This consistency across scales suggests that curvature is a reliable indicator of OOD behavior.
\section{Proposed Method}
\label{sec:proposed_method}
\vspace{-0.75em}
In this section, we propose \fold, a flatness-modulated OOD detector motivated by the observation that ID samples lie in flatter loss regions than OOD samples. 
We derive a feature Hessian by transporting logit-space curvature through the classifier, enabling flatness estimation in the feature space (Section~\ref{subsec:feature_hessian}). 
We then apply partial feature normalization and use the trace of the normalized feature Hessian as the detection score (Section~\ref{subsec:partial_normalization}).
Finally, we introduce \autofold, a self-supervised scheme that selects the normalization exponent via ID logit masking without requiring OOD data (Section~\ref{subsec:autofold}).

\vspace{-0.75em}
\subsection{Efficient Geometric Curvature via the Feature Hessian}
\label{subsec:feature_hessian}
\vspace{-0.5em}
We consider a deep neural network $f(\cdot;\boldsymbol{\theta})$ decomposed into a feature encoder $h(\cdot;\boldsymbol{\theta}_h)$ and a linear classifier $g(\cdot;\boldsymbol{\theta}_g)$. 
Given an input $\mathbf{x}$, the network produces a latent representation $\mathbf{h}=h(\mathbf{x}) \in \mathbb{R}^d$ and corresponding logits $\mathbf{z}=g(\mathbf{h}) \in \mathbb{R}^C$, where $C$ denotes the number of classes. 
Following prior works~\citep{sun2021dice, sun2021react}, we adopt the energy function~\citep{Energy}
\begin{equation}
    \mathcal{L}(\mathbf{z}) = \log \left( \sum_{i=1}^{C} \exp(z_i) \right).
\end{equation}

As discussed in~\Cref{sec:observations}, curvature contains rich information distinguishing ID and OOD samples. 
However, computing the full parameter Hessian remains computationally expensive for modern deep networks, even with efficient approximations. 
To address this limitation, we draw on recent insights into neural network dynamics. 
In particular, Lee~\etal~\citep{lee2023implicit} show that the logit-space Hessian $\nabla_{\mathbf{z}}^2 \mathcal{L}$ exhibits strong spectral correspondence with the full Hessian, providing a tractable surrogate for curvature analysis.

Nevertheless, curvature measured purely in logit space ignores the geometric structure induced by the classifier. 
Since the linear mapping $g$ determines class decision boundaries in the learned representation space, its Jacobian encodes the discriminative geometry of ID features, \ie, a structure known to be critical for OOD detection~\citep{haoqi2022vim}. 
To jointly capture predictive uncertainty and feature-space geometry, we therefore transport the logit curvature back to the latent space through the classifier mapping. 

Formally, applying the second-order chain rule to $\mathcal{L}(\mathbf{z})$ yields
\begin{equation}
    \nabla_{\mathbf{h}}^{2} \mathcal{L}
    =
    \left( \nabla_{\mathbf{h}} g \right)
    \left( \nabla_{\mathbf{z}}^2 \mathcal{L} \right)
    \left( \nabla_{\mathbf{h}} g \right)^{\!\top}
    +
    \sum_{i=1}^{C}
    \frac{\partial \mathcal{L}}{\partial z_i}
    \nabla_{\mathbf{h}}^{2} g_i.
    \label{eq:feature_chain_rule}
\end{equation}
Because $g$ is linear, $\nabla_{\mathbf{h}}^{2} g_i = 0$ for all $i$, and the second term vanishes, yielding
\begin{equation}
    \nabla_{\mathbf{h}}^{2} \mathcal{L}
    =
    \left( \nabla_{\mathbf{h}} g \right)
    \left( \nabla_{\mathbf{z}}^2 \mathcal{L} \right)
    \left( \nabla_{\mathbf{h}} g \right)^{\!\top}.
    \label{eq:feature_hessian}
\end{equation}
Here, $\nabla_{\mathbf{h}} g \in \mathbb{R}^{C \times d}$ corresponds to the classifier weight matrix, characterizing the sensitivity of logits to perturbations in feature space.

Equation~\eqref{eq:feature_hessian} indicates that the feature Hessian functions as a geometry-aware curvature operator, projecting predictive uncertainty from logit space onto the classifier-induced feature manifold.
Consequently, the resulting curvature reflects directional sharpness aligned with class-discriminative subspaces, capturing geometric structures that are otherwise obscured in the purely logit-space Hessian.

\subsection{Calibrated Geometric Curvature via Partial Normalization}
\label{subsec:partial_normalization}
\vspace{-0.3em}
One limitation of curvature evaluation in feature space is its dependence on the representation norm $\|\bh\|$. 
As the feature magnitude increases, the logits scale proportionally, pushing the softmax distribution toward saturation. 
This effect flattens the local loss landscape and causes the feature Hessian to vanish (\ie, $\nabla_{\bh}^2 \cL \rightarrow \mathbf{0}$), potentially producing a degenerate curvature estimate that fails to reflect meaningful predictive uncertainty.

To mitigate this issue, we introduce partial feature normalization, transforming $\bh$ into a scale-adjusted representation $\tilde{\bh}$:
\vspace{-0.25em}
\begin{equation}
    \tilde{\bh} = \frac{\bh}{\|\bh\|^{\alpha}}, \qquad 0 < \alpha \leq 1,
\end{equation}
where the hyperparameter $\alpha$ controls the degree of magnitude suppression and implicitly induces a sample-dependent temperature scaling of the logits. 
By balancing directional information with feature magnitude in this manner, the final \fold score is defined as
\vspace{-0.25em}
\begin{equation}
    S_{\mathrm{\fold}}(\bx) = \tr\!\left(\nabla_{\tilde{\bh}}^2 \cL\right),
\end{equation}
which effectively measures curvature in the normalized feature space while mitigating magnitude-induced degeneracy.

Empirically, the optimal $\alpha$ varies with dataset complexity. 
Simple datasets (\eg, CIFAR) benefit from near-full normalization ($\alpha\!\approx\!1$), whereas more diverse datasets (\eg, ImageNet) favor milder suppression ($\alpha\!\approx\!0.2$) to preserve magnitude cues (see~\Cref{subsec:partial feature normalization}). 
This motivates the automated calibration strategy introduced in the next section.

\vspace{-0.8em}
\subsection{Self-Supervised Parameter Calibration via ID Logit Masking}
\label{subsec:autofold}
\vspace{-0.3em}

As discussed in the previous section, the optimal $\alpha$ depends on the dataset’s intrinsic feature scale. 
Conventional OOD detection methods often tune hyperparameters using an auxiliary OOD validation set, assuming prior knowledge of the OOD distribution and limiting realistic deployment~\citep{yang2022openood, zhang2023openood}. 
To avoid this dependency, we propose \autofold, a self-supervised framework that selects $\alpha$ using only a small ID validation set.

\autofold generates pseudo-OOD signals via ID logit masking by suppressing the $k$-th logit of a sample $\bx$ with label $y=k$, simulating an unknown-class scenario:
\vspace{-0.25em}
\begin{equation}
    \tilde{z}_j =
    \begin{cases}
        z_j & \text{if } j \neq k, \\
        -\infty & \text{if } j = k.
    \end{cases}
\end{equation}
Removing the dominant class evidence forces the model to rely on residual logits, inducing predictive ambiguity similar to that observed in OOD inputs.

Unlike leave-one-class-out training~\citep{vyas2018out, loco2}, which requires $K$ retraining cycles for $K$ classes, \autofold operates purely at inference time with negligible overhead.
The optimal parameter $\alpha^*$ is obtained by maximizing the separability between the scores of ID validation samples and their masked counterparts:
\vspace{-0.25em}
\begin{equation}
    \alpha^*
    =
    \arg\max_{\alpha}
    \frac{1}{K}\sum_{k=1}^{K}
    \mathrm{AUROC}\!\left(
    \mathcal{S}_{ID}^{(k)},
    \mathcal{S}_{pseudo}^{(k)}
    \right),
\end{equation}
where $\mathcal{S}_{ID}^{(k)}$ and $\mathcal{S}_{pseudo}^{(k)}$ denote the \fold score sets for the validation samples of class $k$ and their logit-masked versions, respectively.
This procedure calibrates $\alpha$ in a strictly in-distribution manner, providing a practical and deployment-friendly tuning criterion.
\section{Evaluations}
\label{sec:evaluations}
\subsection{Experimental Setup}
\textbf{Baselines and implementation details.}\quad
To evaluate \fold, we compare it with a comprehensive set of post-hoc OOD detection methods operating exclusively at inference time and typically assuming classifiers trained with the standard cross-entropy loss~\citep{openmax16cvpr, MSP, tempscale, ODIN, MDS, Energy, RMDS, sun2021react, MLS, haoqi2022vim, sun2022out, ASH, SHE}. 
As \fold is orthogonal to activation thresholding and feature reshaping techniques, we further combine it with ReAct~\citep{sun2021react} and ASH~\citep{ASH}, denoted as \foldr and \folda, respectively. 
Additional implementation details are provided in Appendix~\ref{app:experimental details}. 
The code is available at \href{https://github.com/shpark97/Fold}{https://github.com/shpark97/Fold}.

\vspace{1em}\noindent\textbf{Evaluation metrics.}\quad
We report the area under the receiver operating characteristic curve (AUROC) and the false positive rate at 95\% true positive rate (FPR95) as the primary evaluation metrics. 
All pretrained models are kept fixed during evaluation to preserve their original in-distribution classification performance.


\subsection{Main Results}
\textbf{Standard OOD detection benchmarks.}\quad
\Cref{tab:ood benchmarks main} reports the results on CIFAR and ImageNet benchmarks, comparing \fold with multiple existing baselines. 
Across all evaluations, the proposed methods achieve the best average performance in terms of both AUROC and FPR95, demonstrating the effectiveness of incorporating curvature-based information for OOD detection. 
In particular, the standard \fold attains an average AUROC of 87.20\%, improving upon strong prior baselines including KNN (+0.97\%) and ReAct (+0.92\%). 
For FPR95, \fold achieves a competitive score of 42.42\%, matching VIM and improving upon KNN by 0.82\%.

Performance further improves when \fold is combined with feature shaping techniques. 
Specifically, \fold-\textsc{r} achieves the best overall results, reaching an average AUROC of 87.91\% and an FPR95 of 40.12\%, outperforming prior methods by 1.63\% and 2.30\%, respectively. 
These gains arise because ReAct's activation thresholding stabilizes feature representations, sharpening curvature cues for more reliable uncertainty estimation. 
Additionally, \fold-\textsc{a} performs particularly well on ImageNet benchmarks, achieving among the best results on both ImageNet-200 and ImageNet-1K.

Moreover, unlike prior approaches that often perform well on one benchmark but degrade on others, \fold and its variants maintain consistently strong performance across both CIFAR and ImageNet benchmarks, indicating that curvature-based cues remain effective across varying dataset scales and distribution shifts.

\begin{table*}[t]
    \centering
    \caption{
        Comparison on standard OOD detection benchmarks.
        AUROC ($\uparrow$) and FPR95 ($\downarrow$) are averaged across all OOD datasets for each ID dataset.
        \textbf{Bold} entries denote the best performance, and \underline{underlined} entries indicate the second- and third-best results.
        $\dagger$ marks results reported in \citep{zhang2023openood}.
        Overall, \fold consistently outperforms the competing baselines in terms of average performance.
    }
    \label{tab:ood benchmarks main}
    \vspace{-0.75em}
    \resizebox{\linewidth}{!}{%
        \begin{sc}
        \begin{tabular}{l @{\hspace{0.5em}}|@{\hspace{0.5em}} c@{\hspace{1em}}c@{\hspace{1em}}c@{\hspace{1em}}c @{\hspace{0.5em}}|@{\hspace{0.5em}} c}
            \toprule
            \multirow{2}{*}{\textbf{Method}} & \textbf{CIFAR-10} & \textbf{CIFAR-100} & \textbf{ImageNet-200} & \textbf{ImageNet-1K} & \textbf{Average} \\
            \cmidrule(lr){2-2} \cmidrule(lr){3-3} \cmidrule(lr){4-4} \cmidrule(lr){5-5} \cmidrule(lr){6-6}
            & AUROC($\uparrow$) / FPR95($\downarrow$) & AUROC($\uparrow$) / FPR95($\downarrow$) & AUROC($\uparrow$) / FPR95($\downarrow$) & AUROC($\uparrow$) / FPR95($\downarrow$) & AUROC($\uparrow$) / FPR95($\downarrow$) \\ 
            \midrule
            
            \underline{\textbf{Baselines}}  \\
            
            OpenMax$^\dagger$~\citep{openmax16cvpr}
            & 88.95 \,/\, 34.33 & 78.46 \,/\, 55.19 & 86.23 \,/\, 45.26 & 83.46 \,/\, 48.22 & 84.28 \,/\, 45.75 \\
            
            MSP$^\dagger$~\citep{MSP}
            & 89.83 \,/\, 37.21 & 78.60 \,/\, 57.40 & 87.41 \,/\, 43.19 & 81.55 \,/\, 57.14 & 84.35 \,/\, 48.73 \\
            
            TempScale$^\dagger$~\citep{tempscale}
            & 90.01 \,/\, 39.31 & 79.46 \,/\, 56.79 & 87.97 \,/\, 42.33 & 83.39 \,/\, 53.79 & 85.21 \,/\, 48.05 \\
            
            ODIN$^\dagger$~\citep{ODIN}
            & 86.26 \,/\, 63.81 & 79.49 \,/\, 58.54 & 87.13 \,/\, 47.24 & 83.58 \,/\, 55.38 & 84.12 \,/\, 56.24 \\
            
            MDS$^\dagger$~\citep{MDS}
            & 87.88 \,/\, 38.11 & 65.82 \,/\, 76.01 & 69.60 \,/\, 68.64 & 66.72 \,/\, 71.93 & 72.51 \,/\, 63.67 \\
            
            RMDS$^\dagger$~\citep{RMDS}
            & 91.40 \,/\, \underline{29.86} & \underline{82.00} \,/\, \underline{53.70} & 85.86 \,/\, 41.08 & 82.63 \,/\, 50.56 & 85.47 \,/\, 43.80 \\
            
            EBO$^\dagger$~\citep{Energy}
            & 90.00 \,/\, 48.24 & 80.15 \,/\, 56.26 & 87.51 \,/\, 45.01 & 84.03 \,/\, 50.46 & 85.42 \,/\, 49.99 \\
            
            ReAct$^\dagger$~\citep{sun2021react}
            & 89.31 \,/\, 51.12 & 80.52 \,/\, 54.93 & 88.13 \,/\, 42.10 & 87.15 \,/\, 42.46 & 86.28 \,/\, 47.65 \\
            
            MLS$^\dagger$~\citep{MLS}
            & 89.90 \,/\, 48.23 & 80.14 \,/\, 56.31 & 87.83 \,/\, 44.32 & 84.33 \,/\, 50.06 & 85.55 \,/\, 49.73 \\
            
            VIM$^\dagger$~\citep{haoqi2022vim}
            & 91.88 \,/\, 31.65 & 79.46 \,/\, 54.70 & 86.23 \,/\, 39.99 & 84.44 \,/\, 43.34 & 85.50 \,/\, \underline{42.42} \\
            
            KNN$^\dagger$~\citep{sun2022out}
            & \underline{92.19} \,/\, \textbf{27.52} & 81.66 \,/\, 56.17 & 88.52 \,/\, 40.43 & 82.55 \,/\, 48.83 & 86.23 \,/\, 43.24 \\
            
            ASH$^\dagger$~\citep{ASH}
            & 77.42 \,/\, 81.61 & 79.79 \,/\, 61.37 & \underline{89.29} \,/\, 42.33 & \textbf{88.71} \,/\, \underline{37.02} & 83.80 \,/\, 55.58 \\
            
            SHE$^\dagger$~\citep{SHE}
            & 84.06 \,/\, 70.87 & 77.59 \,/\, 62.44 & 85.96 \,/\, 52.02 & 84.07 \,/\, 54.07 & 82.92 \,/\, 59.85 \\
            
            \midrule
            \textbf{\underline{Fixed $\alpha$}}\\
            
            \fold  
            & \textbf{92.46} \,/\, \underline{29.43} & 81.94 \,/\, \underline{51.62} & 88.65 \,/\, 40.30 & 85.74 \,/\, 48.35 & \underline{87.20} \,/\, \underline{42.42} \\
            
            \fold-r  
            & \underline{92.30} \,/\, 30.65 & \textbf{82.22} \,/\, \textbf{50.55} & 89.07 \,/\, \underline{38.92} & \underline{88.04} \,/\, 40.36 & \textbf{87.91} \,/\, \textbf{40.12} \\
            
            \fold-a  
            & 89.43 \,/\, 49.18 & \underline{81.98} \,/\, 54.26 & \textbf{89.68} \,/\, \textbf{37.06} & \underline{88.28} \,/\, \underline{37.03} & \underline{87.34} \,/\, 44.38 \\
            
            \midrule
            \textbf{\underline{Auto $\alpha$ tuning}}\\
            
            \autofold  
            & \textbf{92.46} \,/\, \underline{29.43} & 79.76 \,/\, 57.52 & 88.37 \,/\, 40.82 & 84.38 \,/\, 46.93 & 86.24 \,/\, 43.67 \\
            
            \autofold-r
            & \underline{92.30} \,/\, 30.65 & 79.40 \,/\, 55.71 & 89.03 \,/\, 39.45 & 83.85 \,/\, 43.39 & 86.15 \,/\, \underline{42.30} \\
            
            \autofold-a  
            & 89.43 \,/\, 49.18 & 78.44 \,/\, 63.84 & \underline{89.44} \,/\, \underline{37.97} & 87.60 \,/\, \textbf{36.81} & 86.23 \,/\, 46.95 \\
            
            \bottomrule
        \end{tabular}
        \end{sc}
    }
\end{table*}
\begin{figure*}[t]
    \centering
    \begin{subfigure}{0.495\linewidth}
        \centering
        \includegraphics[width=\linewidth]{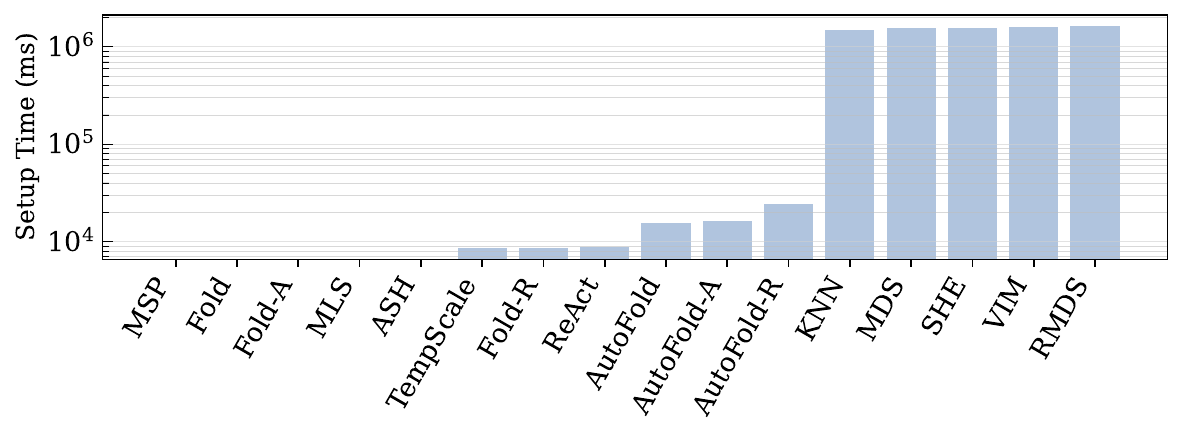}
        \subcaption{Setup time (ms)}
    \end{subfigure}
    \begin{subfigure}{0.495\linewidth}
        \centering
        \includegraphics[width=\linewidth]{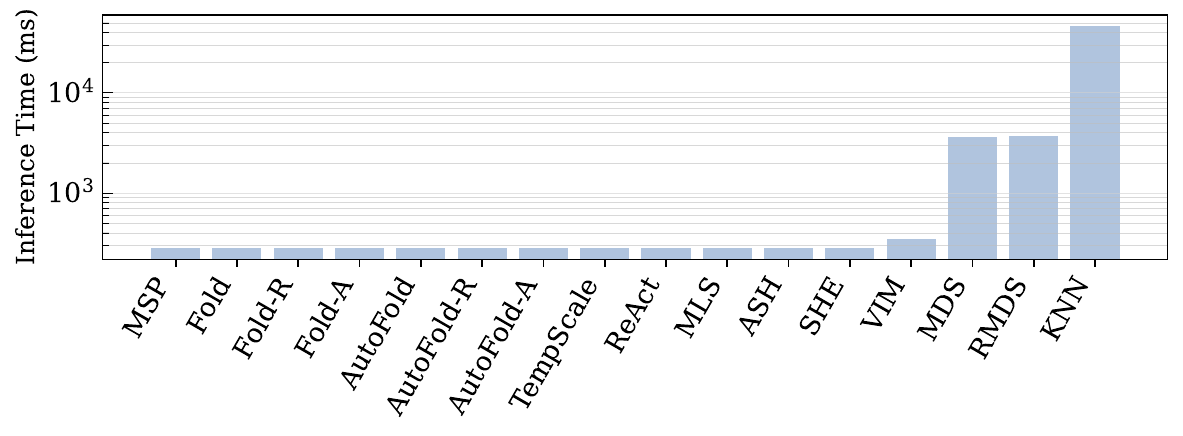}
        \subcaption{Batch-wise inference time (ms)}
    \end{subfigure}
    \vspace{-2em}
    \caption{
        Setup and inference times of baseline methods on ImageNet-1K, ordered left to right.
        All measurements are conducted on a single A100 GPU with a batch size of 256.
        \fold incurs near-minimal setup and inference overhead.
        \autofold, which automatically determines the optimal $\alpha$, introduces only a marginal increase in setup time and no additional inference cost.
        }
    \label{fig:cost}
    \vspace{-1.5em}
\end{figure*}

\vspace{1em}\noindent\textbf{Inference cost.}\quad
To evaluate the efficiency of \fold, we measure the setup and inference times of various baseline methods on ImageNet-1K. 
Here, setup time denotes the preparatory stage required before inference; for instance, KNN requires collecting feature vectors, whereas MDS and RMDS involve estimating class-conditional Gaussian distributions. 
Inference time refers to the time required to compute the OOD detection score for each method.

\vspace{1em}

As shown in \Cref{fig:cost}, the \fold framework incurs minimal overhead in both setup and inference. 
Specifically, the base \fold and \fold-\textsc{a} require virtually no setup time, while the additional computation introduced by \fold-\textsc{r} remains negligible. 
During inference, all \fold variants operate with latency comparable to a standard forward pass, indicating high computational efficiency. 
In contrast, competing methods such as RMDS, VIM, and KNN often incur substantially higher computational costs despite achieving competitive performance in certain settings. 
These overheads arise either from lengthy setup procedures (\eg, VIM) or significant computation during both setup and inference (\eg, RMDS and KNN).

To illustrate the performance–efficiency trade-off, we plot standardized computational cost against standardized AUROC (see \Cref{fig:pitch 2}). 
The \fold variants cluster in the upper-left region of the plot, indicating strong detection performance with minimal computational overhead. 
Overall, these results show that \fold achieves a favorable balance between detection accuracy and computational efficiency, making it well suited for practical deployment.

\begin{figure}[t]
    \centering
    \begin{subfigure}{0.245\linewidth}
        \centering
        \includegraphics[width=\linewidth]{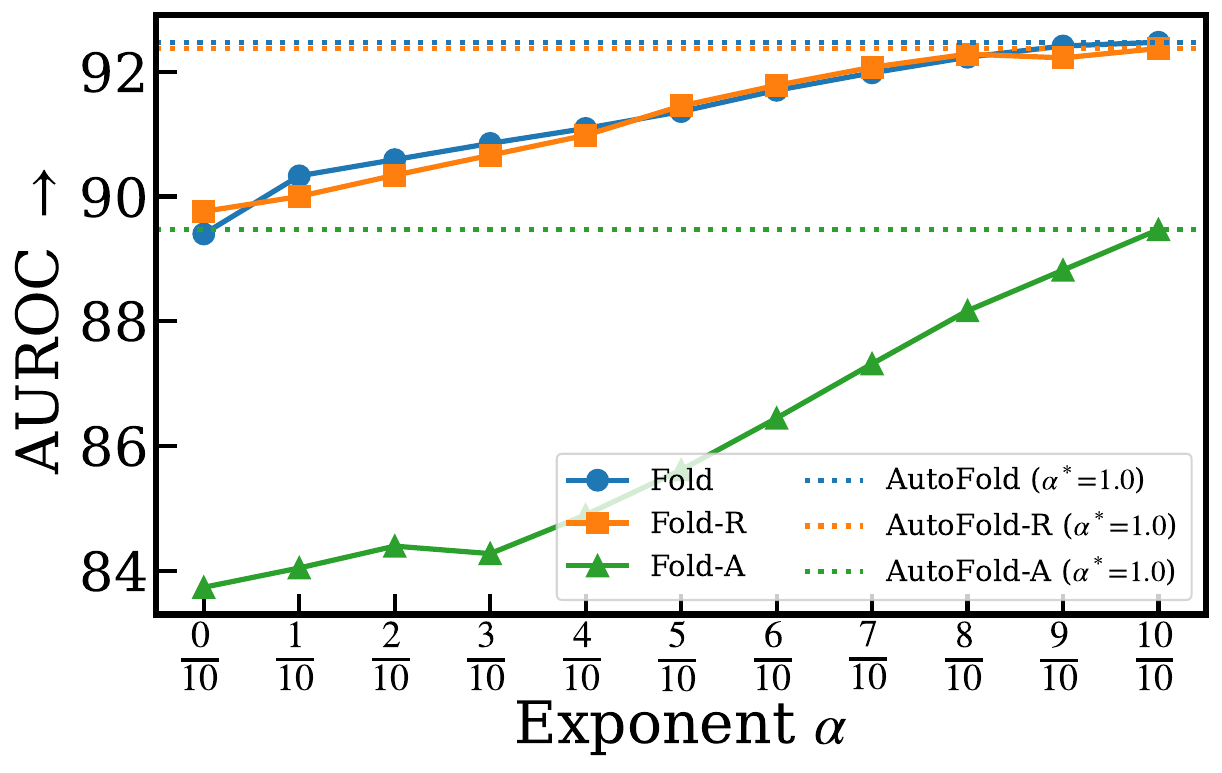}
        \subcaption{CIFAR-10, AUROC($\uparrow$)}
    \end{subfigure}
    \begin{subfigure}{0.245\linewidth}
        \centering
        \includegraphics[width=\linewidth]{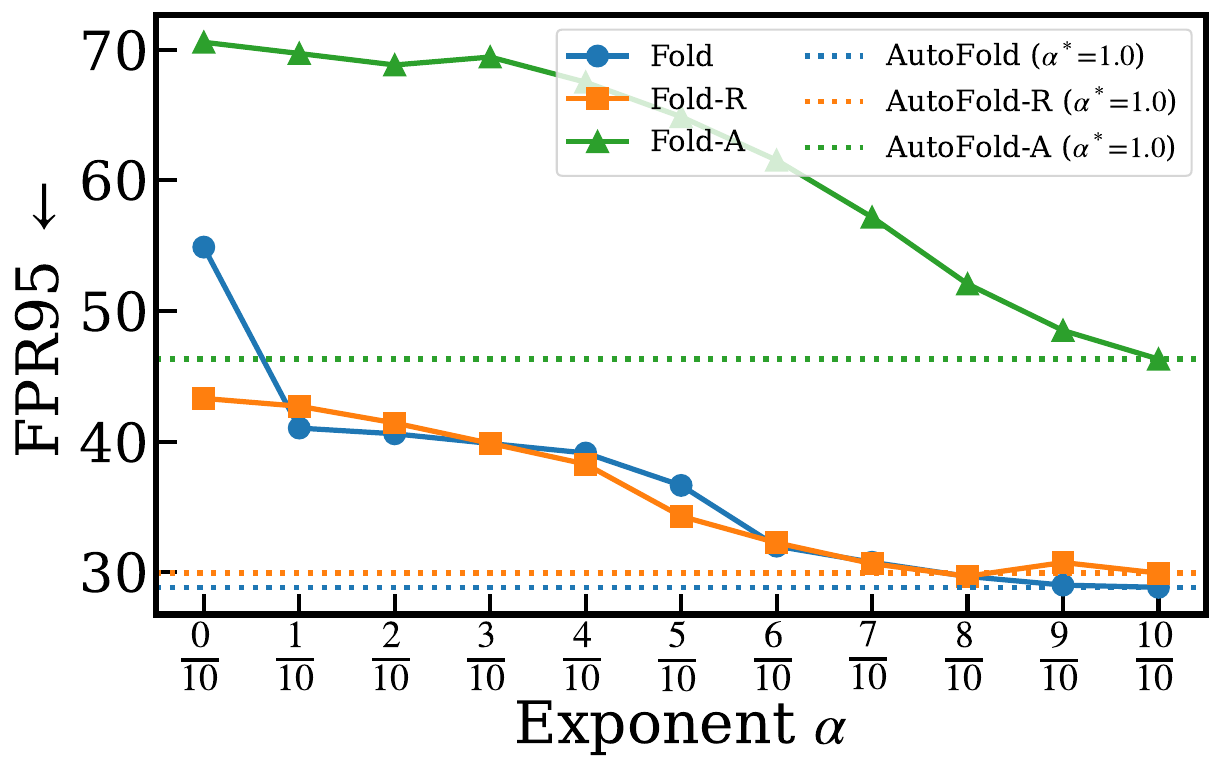}
        \subcaption{CIFAR-10, FPR95($\downarrow$)}
    \end{subfigure}
    \begin{subfigure}{0.245\linewidth}
        \centering
        \includegraphics[width=\linewidth]{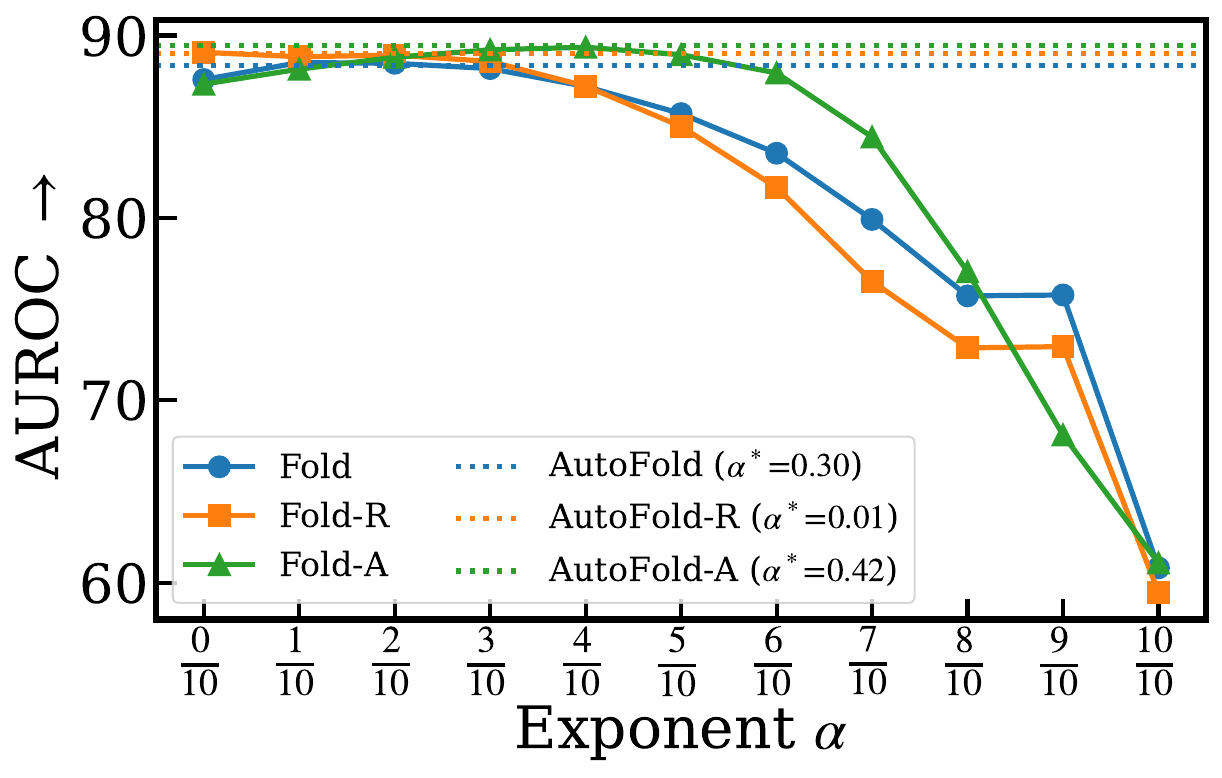}
        \subcaption{ImageNet-200, AUROC($\uparrow$)}
    \end{subfigure}
    \begin{subfigure}{0.245\linewidth}
        \centering
        \includegraphics[width=\linewidth]{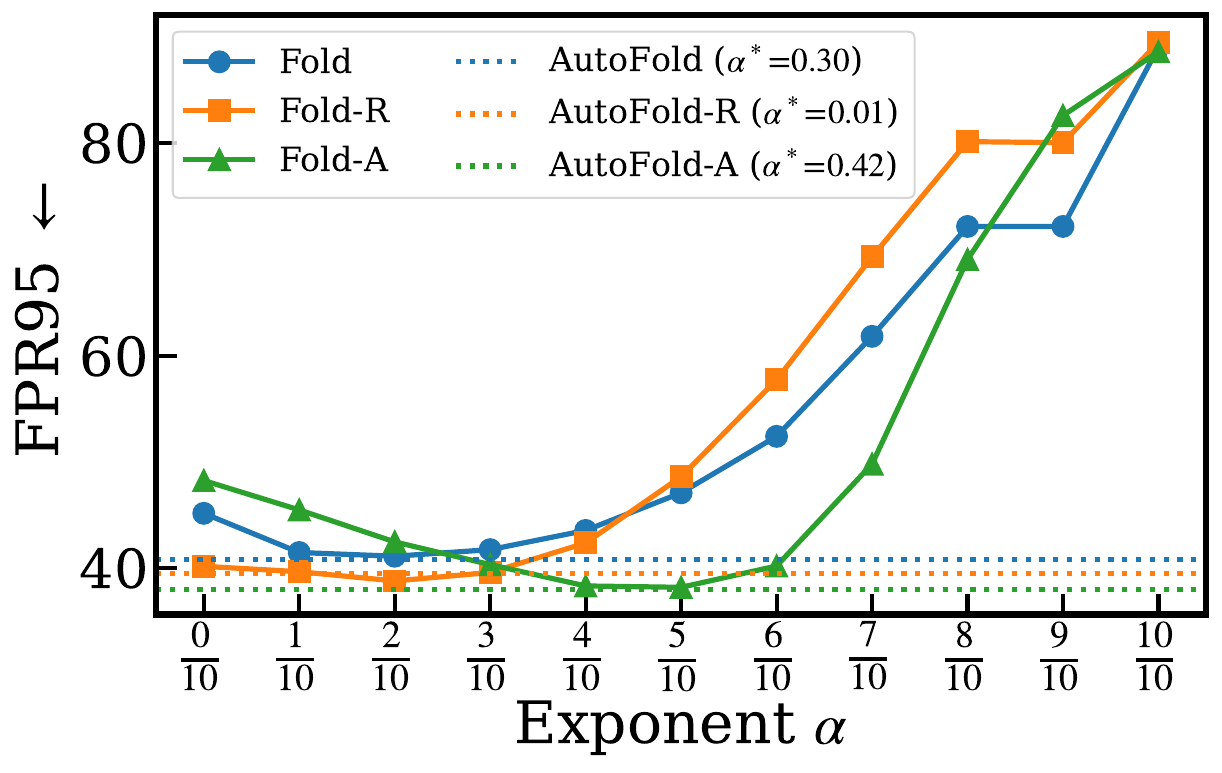}
        \subcaption{ImageNet-200, FPR95($\downarrow$)}
    \end{subfigure}
    \vspace{-1.5em}
    \caption{
        Ablation study on the normalization parameter $\alpha$.
        On CIFAR-10, larger $\alpha$ values consistently improve AUROC and FPR95 across all \fold variants, whereas relatively smaller $\alpha$ values yield better performance on ImageNet-200, indicating dataset-dependent sensitivity to $\alpha$.
        In contrast, \autofold automatically identifies a near-optimal $\alpha$ for each dataset, and the resulting performance (dotted line) closely matches the best achievable results.
        }
    \label{fig:parameter analysis alpha}
    \vspace{-1.5em}
\end{figure}

\subsection{Sensitivity Analysis and Automated Calibration}
\label{subsec:partial feature normalization}
\textbf{Dataset-dependent sensitivity of $\alpha$.}\quad
\Cref{fig:parameter analysis alpha} presents a systematic analysis of the effect of the norm parameter $\alpha$. 
We vary $\alpha$ within the range $[0.0, 1.0]$ for all \fold variants.
For CIFAR-10, larger values of $\alpha$ lead to clear performance gains, whereas for ImageNet-200, smaller values yield superior results. 
This contrast suggests that the optimal influence of $\alpha$ is dataset-dependent, likely reflecting differences in semantic complexity and intrinsic feature distributions.

\vspace{1em}\noindent\textbf{OOD-free calibration via \autofold.}\quad
As indicated by the dotted lines in \Cref{fig:parameter analysis alpha}, \autofold identifies a near-optimal parameter setting that closely matches the peak detection performance without manual tuning. 
Consistent with this observation, \Cref{tab:ood benchmarks main} shows that \autofold performs comparably to \fold evaluated with a fixed $\alpha$, and occasionally achieves slightly better results. 
This improvement arises from the search resolution: while the manual \fold evaluation discretizes $\alpha$ into 10 intervals within $[0.0, 1.0]$, \autofold explores the parameter space with a finer resolution of 100 intervals. 
Consequently, \autofold achieves competitive performance without relying on any external OOD validation data.

Moreover, \autofold operates with setup times that are orders of magnitude lower (on a logarithmic scale) than those of competing methods (\eg, KNN, MDS, SHE, VIM, RMDS). 
Notably, the reported setup costs of these baselines do not account for the additional hyperparameter tuning typically performed on auxiliary OOD validation sets. 
In contrast, \autofold removes this requirement entirely, highlighting its practicality and efficiency in realistic scenarios where OOD data are unavailable.
\vspace{-1em}
\section{Further Analysis}
\label{sec:further analysis}
\vspace{-0.5em}
\subsection{Partial Normalization}
\label{subsec:further partial normalization}
\vspace{-0.3em}
\textbf{Generalizability to existing baselines.}\quad
To assess whether the proposed partial feature normalization generalizes to existing methods, we evaluate three normalization schemes (no, partial, and full) across multiple baselines on CIFAR-10 and ImageNet-200. 
As shown in \Cref{tab:norm_method_comparison}, partial normalization consistently improves prior approaches. 
In contrast to previous works that adopt full feature normalization~\citep{mueller2025mahalanobis++, park2023understanding, RMDS}, our results reveal clear dataset-dependent behavior. 
On CIFAR-10, partial and full normalization yield comparable improvements. 
However, on more complex datasets (\ie, ImageNet-200), full normalization substantially degrades detection performance by discarding informative magnitude cues. 
In contrast, partial normalization improves performance by preserving informative feature variance while suppressing noisy magnitude effects. 
These findings suggest that partial normalization serves as an effective plug-and-play component for enhancing general OOD detection frameworks.

\begin{table*}[!t]
    \centering
    \caption{
        Effects of partial feature normalization on baselines, evaluated using AUROC ($\uparrow$) and FPR95 ($\downarrow$) on CIFAR-10 and ImageNet-200.
        On CIFAR-10, normalization improves performance, with partial and full normalization both competitive.
        On ImageNet-200, partial normalization yields consistent gains, whereas full normalization substantially degrades detection performance.
    }
    \label{tab:norm_method_comparison}
    \vspace{-0.75em}
    \resizebox{\linewidth}{!}{%
        \begin{sc}
        \begin{tabular}{l @{\hspace{1.5em}} c @{\hspace{1.5em}} c@{\hspace{1.5em}}c@{\hspace{1.5em}}c@{\hspace{1.5em}}c@{\hspace{1.5em}}c@{\hspace{1.5em}}c@{\hspace{1.5em}}c @{\hspace{0.75em}}@{\hspace{0.75em}} c}
            \toprule
            & & \multicolumn{7}{c@{\hspace{1.5em}}}{\textbf{Method}} \\
            \cmidrule(lr){3-9}
            \textbf{Dataset} & \textbf{Norm.} & \textbf{MSP} & \textbf{EBO} & \textbf{ReAct} & \textbf{ASH} & \textbf{\fold} & \textbf{\fold-R} & \textbf{\fold-A} & \textbf{Average} \\
            \midrule
            
            \multirow{3}{*}{CIFAR-10} 
            & $\alpha = 0$    & 89.83 \,/\, 37.21 & 90.00 \,/\, 48.24 & 89.31 \,/\, 51.12 & 77.42 \,/\, 81.61 & 89.72 \,/\, 51.40 & 89.83 \,/\, 43.46 & 83.82 \,/\, 69.63 & 87.13 \,/\, 54.67 \\
            & $0 < \alpha < 1$ & 91.92 \,/\, 28.22 & \textbf{92.37} \,/\, \textbf{31.61} & 92.06 \,/\, \textbf{33.74} & 88.02 \,/\, 56.70 & 92.42 \,/\, \textbf{29.22} & 92.28 \,/\, \textbf{30.14} & 88.82 \,/\, 51.00 & 91.13 \,/\, 37.23 \\
            & $\alpha = 1$    & \textbf{92.01} \,/\, \textbf{27.84} & 92.30 \,/\, 32.37 & \textbf{92.10} \,/\, 33.75 & \textbf{88.82} \,/\, \textbf{53.26} & \textbf{92.46} \,/\, 29.43 & \textbf{92.29} \,/\, 31.17 & \textbf{89.43} \,/\, \textbf{49.18} & \textbf{91.34} \,/\, \textbf{36.71} \\
            
            \midrule
            
            \multirow{3}{*}{ImageNet-200} 
            & $\alpha = 0$    & 87.41 \,/\, 43.19 & \textbf{87.51} \,/\, 45.01 & \textbf{88.13} \,/\, 42.10 & 89.29 \,/\, 42.33 & 87.81 \,/\, 44.10 & \textbf{89.22} \,/\, 39.24 & 87.61 \,/\, 47.14 & 87.64 \,/\, 43.30 \\
            & $0 < \alpha < 1$ & \textbf{88.31} \,/\, 40.62 & 87.38 \,/\, \textbf{44.82} & 88.02 \,/\, \textbf{41.84} & \textbf{89.67} \,/\, \textbf{39.83} & \textbf{88.65} \,/\, \textbf{40.16} & 89.12 \,/\, \textbf{37.90} & \textbf{89.64} \,/\, \textbf{37.47} & \textbf{89.91} \,/\, \textbf{40.38} \\
            & $\alpha = 1$    & 87.65 \,/\, \textbf{40.39} & 64.92 \,/\, 87.96 & 62.34 \,/\, 88.98 & 66.41 \,/\, 85.33 & 62.23 \,/\, 88.56 & 60.89 \,/\, 89.04 & 62.51 \,/\, 88.08 & 79.02 \,/\, 81.19 \\
            
            \bottomrule
        \end{tabular}
        \end{sc}
    }
    \vspace{-2em}
\end{table*}

\begin{wrapfigure}{r}{0.6\linewidth}
    \vspace{-2.5em}
    \centering
    \includegraphics[width=\linewidth]{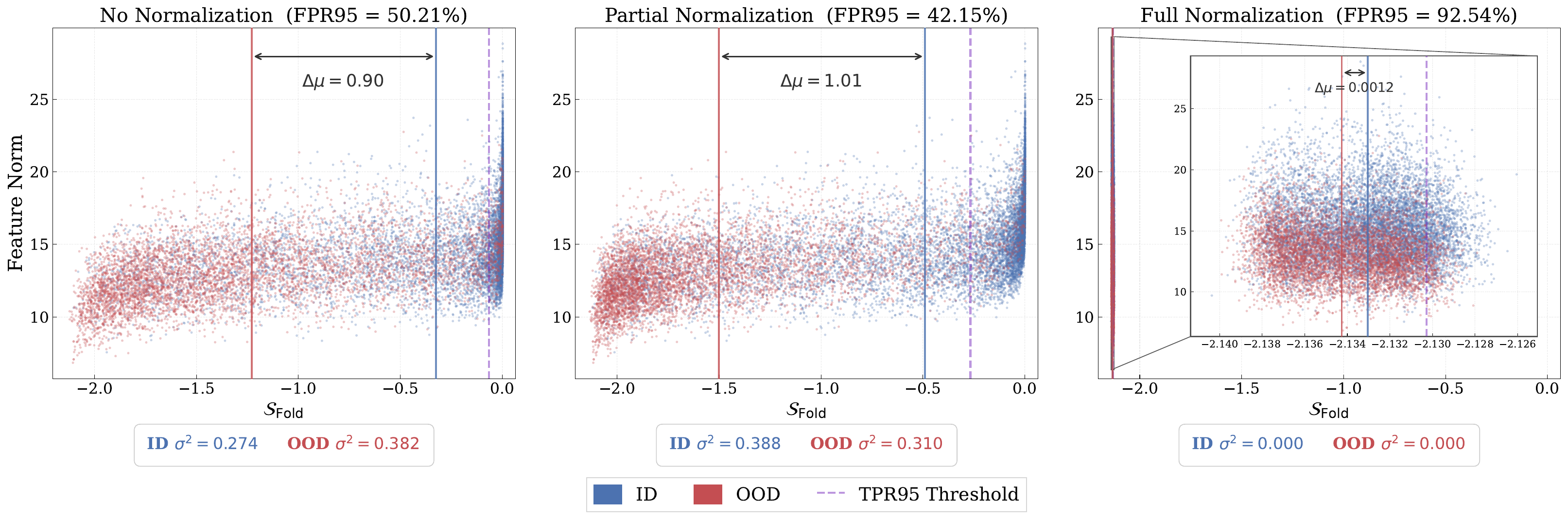}
    \vspace{-1.5em}
    \caption{
        Scatter plots illustrating the relationship between feature norm and \fold score under three normalization regimes.
        Partial normalization enhances ID–OOD separability by increasing the mean score gap and reducing OOD variance, yielding a higher Cohen's $d$ ($1.605 \rightarrow 1.689$)~\citep{cohen2013statistical}.
        In contrast, full normalization removes magnitude information, collapsing score variance and markedly reducing discriminability (Cohen's $d = 0.520$).
        }
    \label{fig:normalization visualization}
    \vspace{-2em}
\end{wrapfigure}

\vspace{1em}\noindent\textbf{Effect of feature normalization.}\quad
To further analyze the effect of partial normalization on the \fold score, we examine the relationship between feature norm magnitude and detection score on ImageNet-200. 
As shown in \Cref{fig:normalization visualization}, partial normalization increases the mean score gap (\ie, $\Delta \mu$) between ID and OOD samples from $0.90$ to $1.01$ while substantially reducing the variance of the OOD distribution relative to the unnormalized baseline. 
Consequently, Cohen's $d$~\citep{cohen2013statistical}, a standard measure of statistical separability, improves from $1.605$ to $1.689$, yielding an absolute $8.06\%$ improvement in FPR95. 
In contrast, full normalization collapses score variance and severely degrades discriminability (Cohen's $d\!=\!0.520$), rendering two distributions nearly indistinguishable (\ie, $\text{FPR95}\!=\!92.54\%$).

\vspace{1em}\noindent\textbf{Spectral analysis of feature curvature.}\quad
To better understand this behavior, we analyze the top 50 absolute eigenvalues of the feature Hessian for ID and OOD samples under three normalization regimes. 
As shown in \Cref{fig:eigenvalues}, full normalization collapses the spectral gap across the spectrum, rendering the ID and OOD spectra nearly identical and thereby explaining the severe degradation in detection performance. 
Without normalization, the spectral difference is largely concentrated in the top 10 eigenvalues, indicating that the score primarily depends on worst-case curvature. 
In contrast, partial normalization slightly reduces the gap among the largest eigenvalues while enlarging the spectral separation across a broader portion of the spectrum. 
This shift suggests that partial normalization exploits more distributed, average-case curvature information, capturing richer geometric signals that improve the robustness of OOD detection.

\begin{figure}[t]
    \centering
    \includegraphics[width=\linewidth]{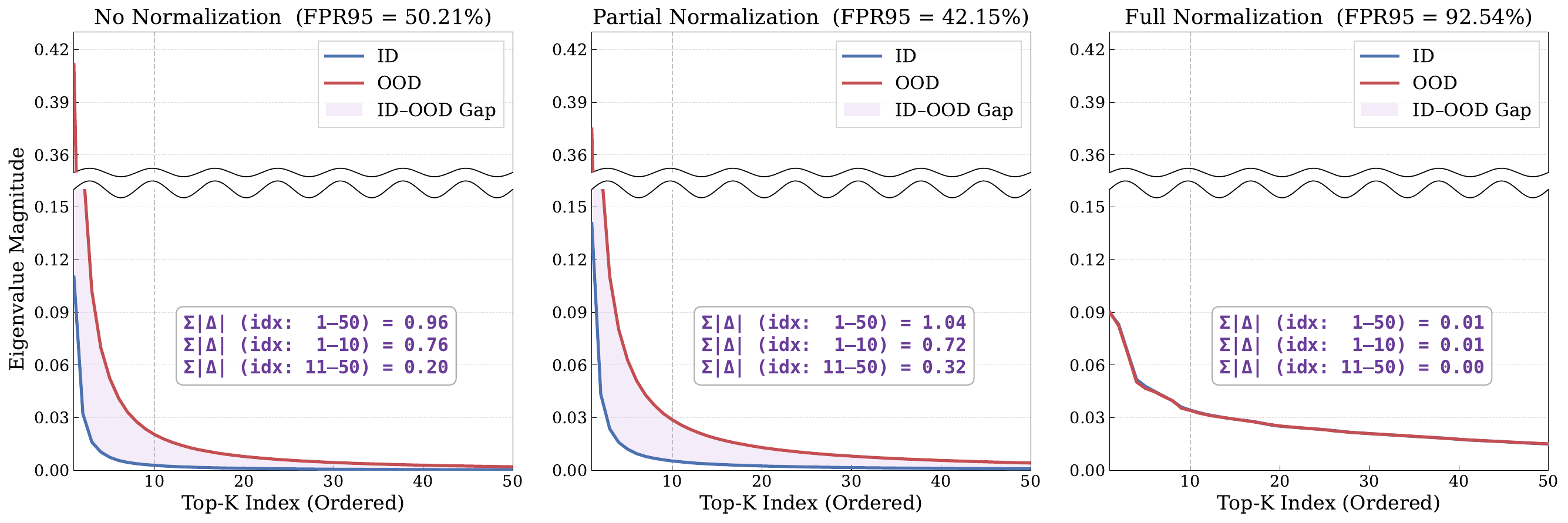}
    \vspace{-1em}
    \caption{
        Top-50 absolute eigenvalues of the feature Hessian for ID and OOD samples on ImageNet-200 under no, partial, and full normalization.
        Full normalization collapses the spectral gap, making the ID and OOD spectra nearly identical.
        Without normalization, the difference is concentrated in only the top portion of the spectrum.
        In contrast, partial normalization enlarges the spectral gap across a broader range of eigenvalues, aligning with its improved OOD detection performance.
        }
    \label{fig:eigenvalues}
\end{figure}

\begin{wraptable}{r}{0.475\linewidth}
    \vspace{-3.5em}
    \centering
    \caption{
        Comparison of logit and feature Hessians in AUROC ($\uparrow$) across datasets.
        The feature Hessian achieves higher performance, indicating richer signals for OOD detection.
    }
    \label{tab:logit hessian vs feature hessian}
    \resizebox{\linewidth}{!}{%
        \begin{sc}
        \begin{tabular}{l @{\hspace{0.5em}}|@{\hspace{0.5em}} c@{\hspace{1em}}c@{\hspace{1em}}c@{\hspace{1em}}c @{\hspace{0.5em}}|@{\hspace{0.5em}} c}
            \toprule
            \textbf{Hessian} & \textbf{C-10} & \textbf{C-100} & \textbf{IN-200} & \textbf{IN-1K} & \textbf{Avg.} \\
            \midrule
            
            Logit & 92.52 & 81.27 & 86.35 & 84.10 & 86.02 \\
            Feature & 92.46 & 81.94 & 88.65 & 85.74 & 86.81 \\
            
            \bottomrule
        \end{tabular}
        \end{sc}
    }
    \vspace{-2em}
\end{wraptable}

\subsection{Feature Hessian}
\label{subsec:feature hessian}
\textbf{Empirical efficacy of the feature Hessian.}\quad
To empirically validate the advantage of evaluating curvature in the latent space, we compare the detection performance of the feature Hessian with that of the logit Hessian. 
As shown in \Cref{tab:logit hessian vs feature hessian}, the feature Hessian consistently improves the average AUROC across standard benchmarks. 
While the performance remains competitive on simpler datasets such as CIFAR-10, the gains become more pronounced on more complex domains (\eg, ImageNet-200 and ImageNet-1K). 
These empirical findings align with our methodological motivation. 
By transporting curvature through the classifier mapping, the feature Hessian leverages the discriminative geometry encoded in the decision boundaries. 
Consequently, it captures both predictive uncertainty and intrinsic feature-space structure, yielding richer and more robust representations for OOD detection.

\vspace{1em}\noindent\textbf{Discriminative power of the feature Hessian.}\quad
To examine whether the feature Hessian exhibits curvature behavior similar to that of the parameter-space Hessian, we analyze the empirical spectral density (ESD)~\citep{yao2020pyhessian}. 
The ESD provides a comprehensive characterization of curvature by capturing the distribution of eigenvalues across the spectrum. 
As illustrated in \Cref{fig: Feature Hessian ESD}, although the absolute scales differ due to dimensionality and parameterization, both the spectral density and the largest eigenvalue consistently increase under distribution shifts. 
Notably, the feature Hessian exhibits a tightly concentrated spectrum for ID data and a more dispersed spectrum for OOD data. 
These observations indicate that the feature Hessian serves as a reliable and computationally efficient surrogate for the full parameter-space Hessian. 
Moreover, it can offer a sharper and more discriminative signal for separating ID from OOD inputs.

\begin{figure*}[!t]
    \centering
    \begin{subfigure}{0.245\linewidth}
        \centering
        \includegraphics[width=\linewidth]{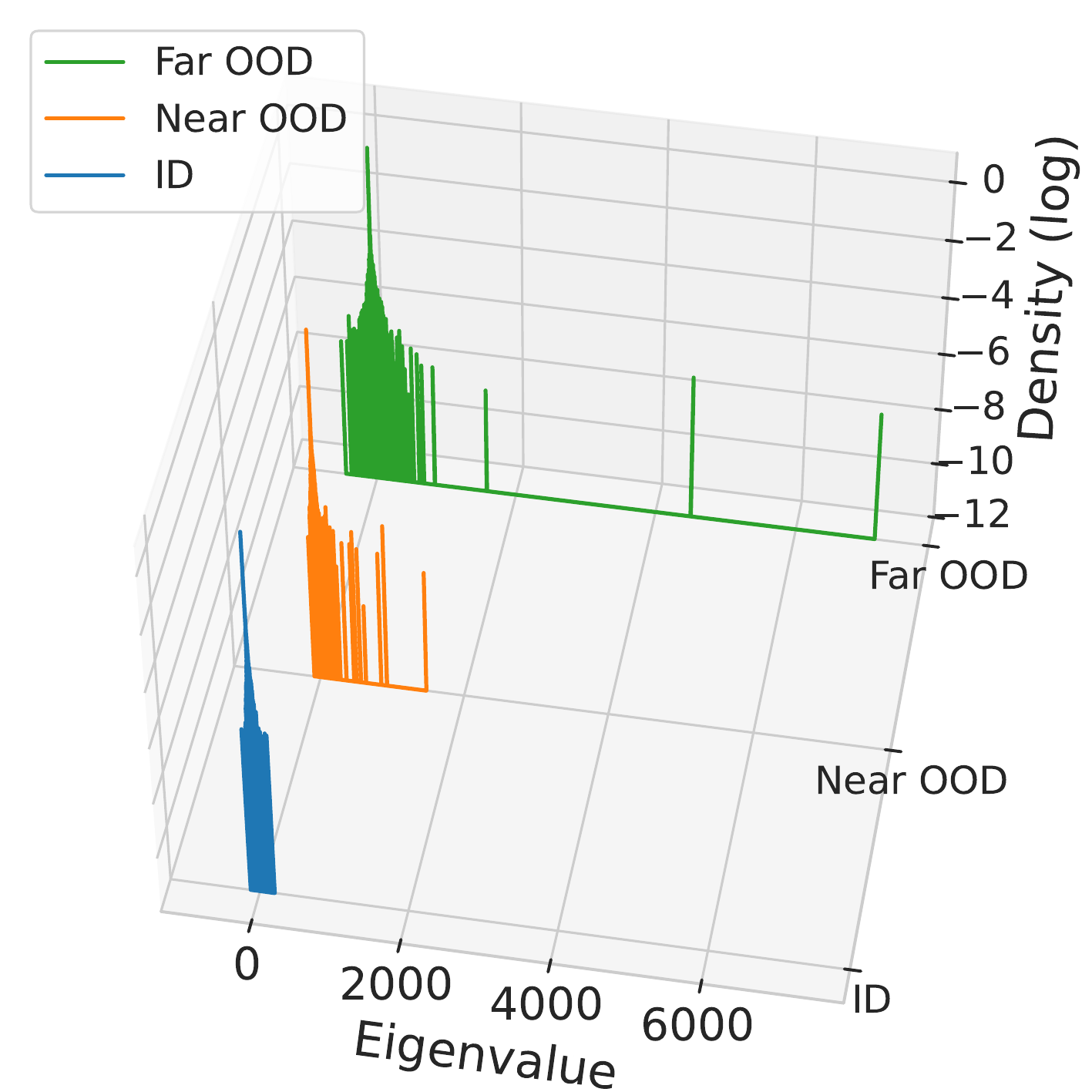}
        \subcaption{CIFAR-10, Full}
    \end{subfigure}
    \begin{subfigure}{0.245\linewidth}
        \centering
        \includegraphics[width=\linewidth]{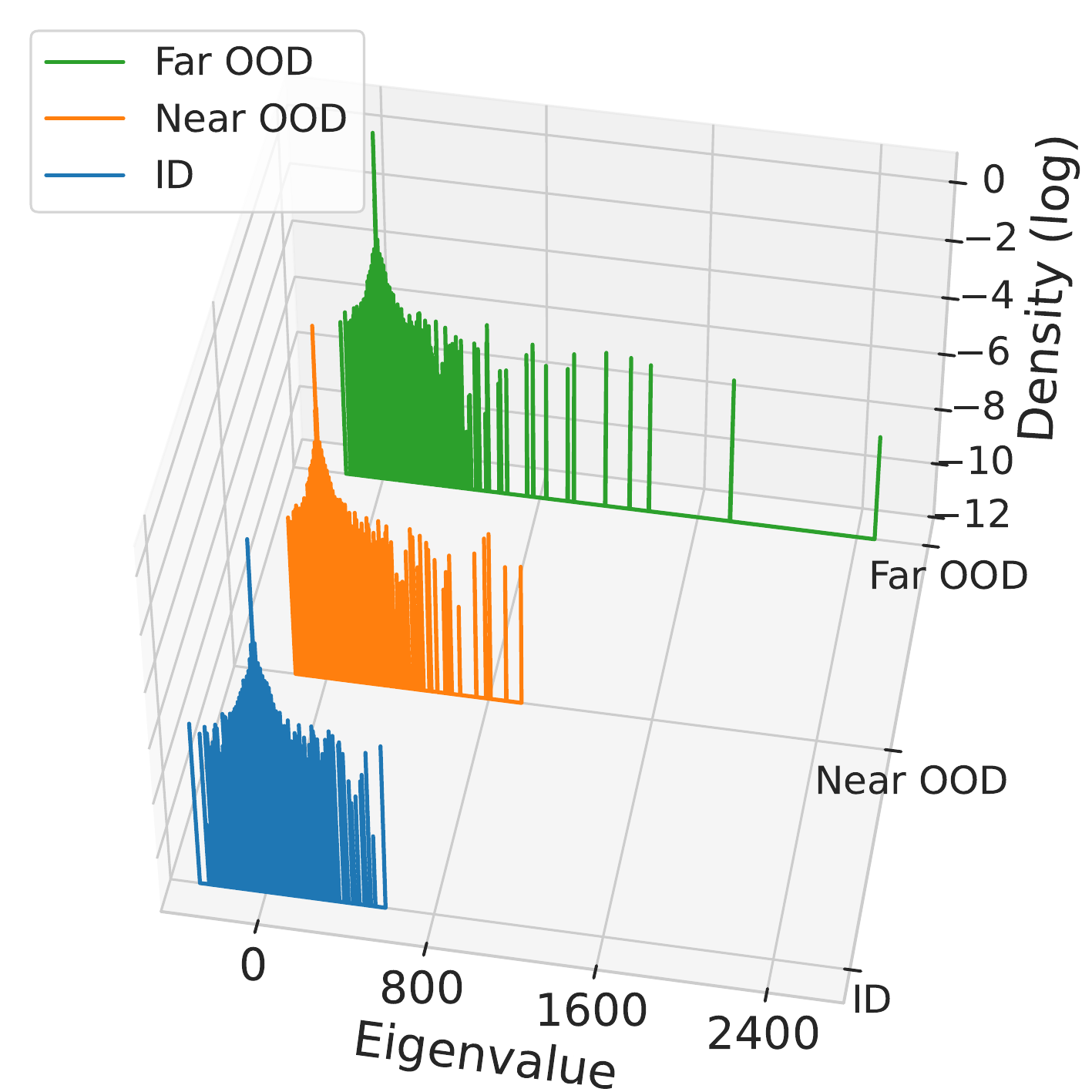}
        \subcaption{ImageNet-200, Full}
    \end{subfigure}
    \begin{subfigure}{0.245\linewidth}
        \centering
        \includegraphics[width=\linewidth]{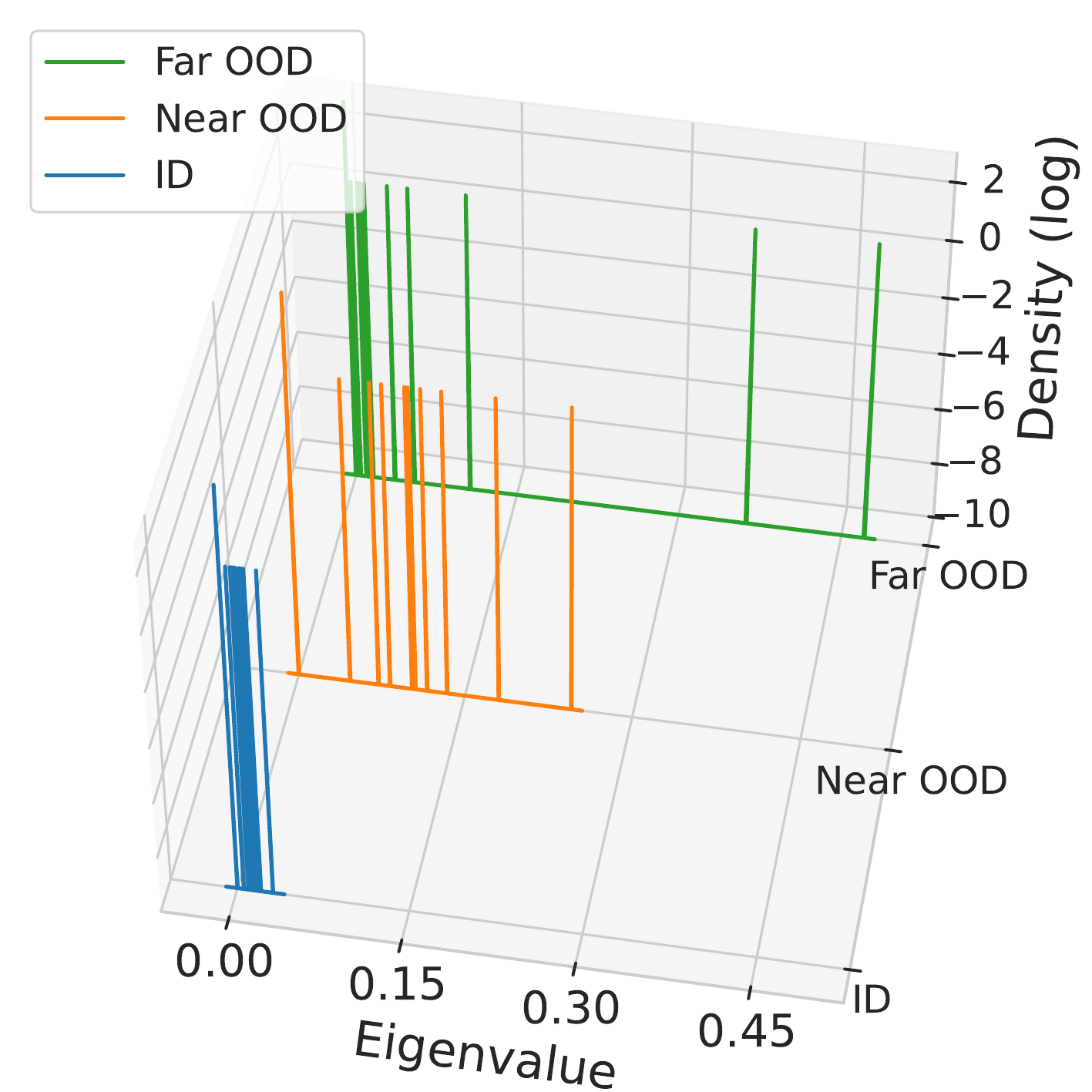}
        \subcaption{CIFAR-10, Feature}
    \end{subfigure}
    \begin{subfigure}{0.245\linewidth}
        \centering
        \includegraphics[width=\linewidth]{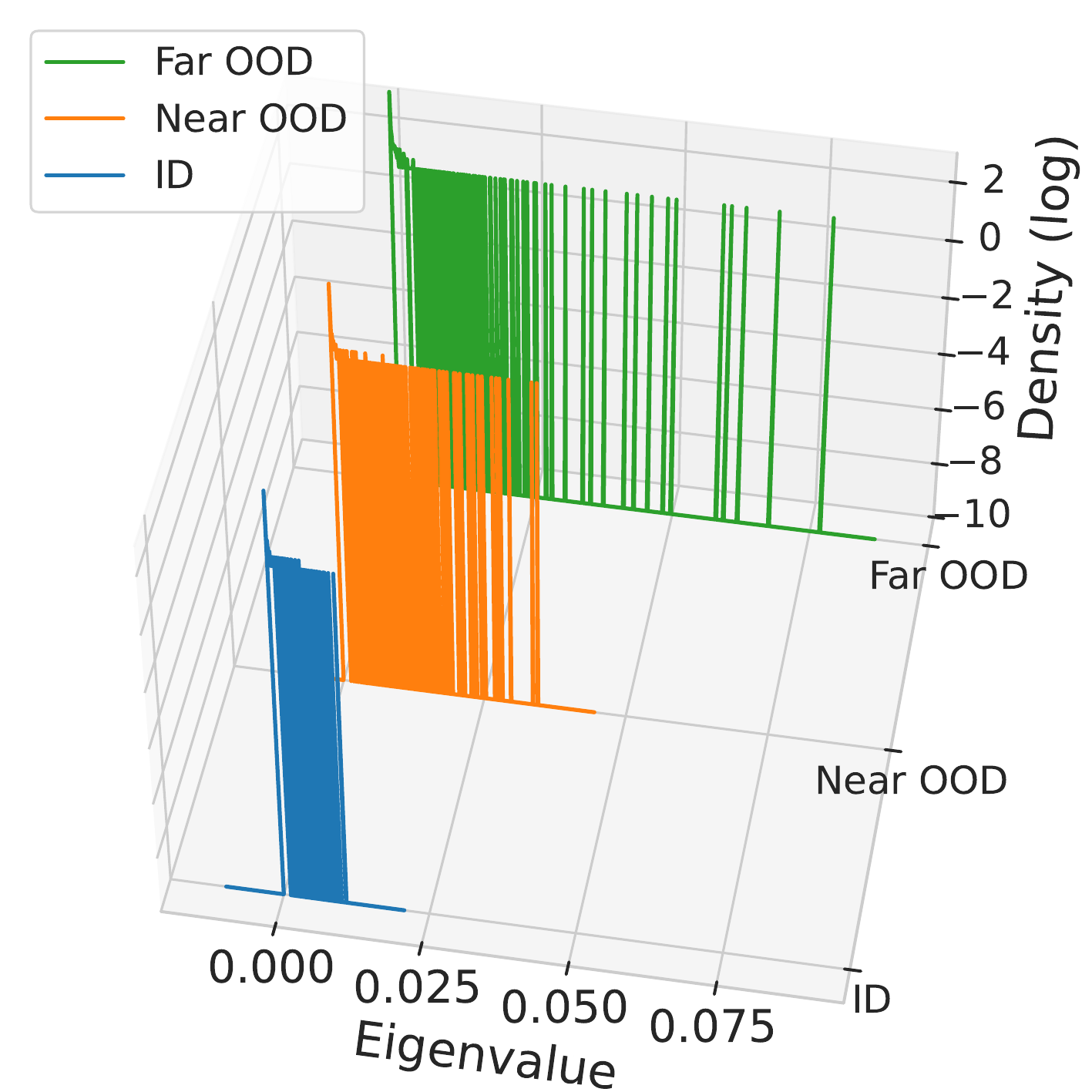}
        \subcaption{ImageNet-200, Feature}
    \end{subfigure}
    \vspace{-1.6em}
    \caption{
        Empirical spectral densities of the parameter and feature Hessians for ID and OOD data on CIFAR-10 and ImageNet-200.
        Consistent with the full-parameter Hessian, the feature Hessian demonstrates a monotonic increase in the largest eigenvalue when transitioning from ID to OOD samples.
        Furthermore, the feature Hessian exhibits a broader spectral distribution than the parameter Hessian, suggesting a more pronounced separability signal for OOD detection.
        }
    \label{fig: Feature Hessian ESD}
    \vspace{-0.5em}
\end{figure*}

\begin{table*}[t]
    \centering
    \caption{
        Comparison of OOD detection performance on the ImageNet-1K benchmark across diverse backbone architectures~\citep{regnet, densenet, wrn, resnext}.
        Across architectures, \fold yields consistent but modest improvements in average performance, indicating stable and architecture-agnostic gains.
    }
    \label{tab:ood benchmarks architecture}
    \vspace{-0.75em}
    \resizebox{\linewidth}{!}{%
        \begin{sc}
        \begin{tabular}{l @{\hspace{0.5em}}|@{\hspace{0.5em}} c@{\hspace{1.5em}}c@{\hspace{1.5em}}c@{\hspace{1.5em}}c @{\hspace{0.5em}}|@{\hspace{0.5em}} c}
            \toprule
            \multirow{2}{*}{\textbf{Method}} & \textbf{RegNet}~\citep{regnet} & \textbf{DenseNet}~\citep{densenet} & \textbf{WRN}~\citep{wrn} & \textbf{ResNeXt}~\citep{resnext} & \textbf{Average} \\
            \cmidrule(lr){2-2} \cmidrule(lr){3-3} \cmidrule(lr){4-4} \cmidrule(lr){5-5} \cmidrule(lr){6-6}
            & AUROC($\uparrow$) / FPR95($\downarrow$) & AUROC($\uparrow$) / FPR95($\downarrow$) & AUROC($\uparrow$) / FPR95($\downarrow$) & AUROC($\uparrow$) / FPR95($\downarrow$) & AUROC($\uparrow$) / FPR95($\downarrow$) \\ 
            \midrule
            
            \underline{\textbf{Baselines}}  \\
            
            MSP~\citep{MSP}
            & 88.05 \,/\, 44.60 & 79.80 \,/\, 59.20 & 83.19 \,/\, 54.78 & 82.00 \,/\, 56.96 & 83.26 \,/\, 53.89 \\
            
            TempScale~\citep{tempscale}
            & 90.79 \,/\, 38.80 & 82.08 \,/\, 55.83 & 84.86 \,/\, 52.00 & 84.09 \,/\, 53.87 & 85.45 \,/\, 50.13 \\
            
            MDS~\citep{MDS}
            & 88.32 \,/\, 37.66 & 66.44 \,/\, 75.31 & 71.71 \,/\, 64.41 & 65.85 \,/\, 71.96 & 73.08 \,/\, 62.33 \\
            
            RMDS~\citep{RMDS}
            & 90.66 \,/\, 34.60 & 79.34 \,/\, 61.99 & 85.31 \,/\, 43.25 & 83.86 \,/\, 45.81 & 84.79 \,/\, 46.41 \\

            EBO~\citep{RMDS}
            & \underline{92.82} \,/\, 35.26 & 82.51 \,/\, 52.13 & 84.43 \,/\, 52.21 & 84.83 \,/\, 51.75 & 86.15 \,/\, 47.84 \\
            
            ReAct~\citep{sun2021react}
            & 84.93 \,/\, 47.75 & 83.80 \,/\, 46.63 & 85.69 \,/\, 49.32 & 87.40 \,/\, 42.64 & 85.46 \,/\, 46.59 \\
            
            MLS~\citep{MLS}
            & 92.59 \,/\, 35.96 & 82.80 \,/\, 51.92 & 84.89 \,/\, 51.88 & 85.02 \,/\, 51.48 & 86.33 \,/\, 47.81 \\
            
            VIM~\citep{haoqi2022vim}
            & 91.50 \,/\, \underline{31.28} & 80.06 \,/\, 50.06 & 83.62 \,/\, 45.53 & 84.62 \,/\, 42.90 & 84.95 \,/\, 42.44 \\
            
            KNN~\citep{sun2022out}
            & 91.72 \,/\, \underline{30.72} & 72.88 \,/\, 64.20 & 85.39 \,/\, 44.12 & 85.38 \,/\, 43.72 & 83.84 \,/\, 45.69 \\
            
            ASH~\citep{ASH}
            & 78.22 \,/\, 61.88 & 83.00 \,/\, 50.93 & \underline{87.80} \,/\, 46.40 & \textbf{88.93} \,/\, \underline{40.32} & 84.49 \,/\, 49.88 \\
            
            SHE~\citep{SHE}
            & 87.23 \,/\, 50.81 & 81.21 \,/\, 57.61 & 77.67 \,/\, 65.85 & 79.97 \,/\, 65.03 & 81.52 \,/\, 59.82 \\
            
            \midrule
            \textbf{\underline{Fixed $\alpha$}}\\
            
            \fold  
            & \textbf{93.78} \,/\, \textbf{28.93} & 83.52 \,/\, 47.21 & 85.53 \,/\, 45.22 & 85.12 \,/\, 47.00 & \underline{86.99} \,/\, \underline{42.09} \\
            
            \fold-r  
            & 89.39 \,/\, 38.37 & \textbf{85.41} \,/\, \underline{45.58} & 86.59 \,/\, \underline{42.95} & 86.38 \,/\, 42.68 & \underline{86.94} \,/\, \underline{42.40} \\
            
            \fold-a  
            & 86.87 \,/\, 48.44 & \underline{84.91} \,/\, \textbf{44.59} & \underline{87.95} \,/\, \textbf{37.60} & \underline{88.61} \,/\, \textbf{37.54} & \textbf{87.08} \,/\, \textbf{42.05} \\
            
            \midrule
            \textbf{\underline{Auto $\alpha$ tuning}}\\
            
            \autofold  
            & \underline{92.80} \,/\, 32.76 & 83.69 \,/\, 47.76 & 85.62 \,/\, 49.05 & 85.28 \,/\, 48.48 & 86.85 \,/\, 44.51 \\
            
            \autofold-r
            & 90.02 \,/\, 37.12 & 83.22 \,/\, 51.80 & 84.76 \,/\, 46.45 & 84.67 \,/\, 45.69 & 85.67 \,/\, 45.27 \\
            
            \autofold-a  
            & 85.59 \,/\, 51.88 & \underline{85.02} \,/\, \underline{45.50} & \textbf{88.60} \,/\, \underline{42.53} & \underline{88.15} \,/\, \underline{37.71} & 86.84 \,/\, 44.40 \\
            
            \bottomrule
        \end{tabular}
        \end{sc}
    }
    \vspace{-1.5em}
\end{table*}
\subsection{Generalization across diverse architectures}
\vspace{-0.4em}
To evaluate architectural robustness, we compare OOD detection performance on the ImageNet-1K benchmark using four backbone networks: RegNet \citep{regnet}, DenseNet \citep{densenet}, WRN \citep{wrn}, and ResNeXt \citep{resnext}. 
As shown in \Cref{tab:ood benchmarks architecture}, the \fold framework consistently demonstrates strong and stable performance across all evaluated architectures. 
In particular, \fold-\textsc{a} achieves the highest average AUROC of 87.08\% and the lowest average FPR95 of 42.05\%, corresponding to a 0.75\% absolute AUROC improvement over the strongest baseline (\ie, MLS) and a 0.39\% reduction in FPR95 compared to the most competitive baseline for that metric (\ie, VIM). 
Moreover, the base \fold model also attains a competitive average AUROC of 86.99\%.

These results indicate that the geometric advantages of curvature-based OOD scoring are not limited to a particular architecture family but generalize across diverse models. 
Even without an external validation set, the \autofold variants maintain robust detection performance. 
Overall, these findings establish \fold as a versatile OOD detection approach capable of achieving reliable and stable performance across architectures.
\section{Theoretical Analysis}
\label{sec:theoretical_analysis}

In this section, we provide theoretical justification for the proposed \fold score by analyzing the curvature of the energy function under a simplified binary classification setting. 
In particular, we show that the expected trace of the feature Hessian is provably larger for OOD data than for ID data.

Let the feature representation $\tilde{\bx}=h(\bx)\in\RR^d$ follow a class-conditional distribution $\cD_{\mathrm{id}}$ defined as
\begin{equation}
    \tilde{\bx}\mid y \sim \cD_{\mathrm{id}}
    =
    \begin{cases}
        \cN(\bmu,\bSigma) & \text{if } y=1,\\
        \cN(-\bmu,\bSigma) & \text{if } y=-1,
    \end{cases}
\end{equation}
where $\bSigma$ is a positive definite covariance matrix.
Consider a linear classifier $g(\tilde{\bx};\bw)=[\tilde{\bx}^\top\bw,0]^\top$ trained by minimizing the expected logistic loss $\EE[-\log(\sum_{i=1}^2 \exp(yz_i))]$, where $\bz=g(\tilde{\bx};\bw)$ denotes the logits.
To model OOD inputs, we assume features are drawn from a shifted and rotated variant of $\cD_{\mathrm{id}}$:
\begin{equation}
    \tilde{\bx}\mid y \sim \cD_{\mathrm{ood}}
    =
    \begin{cases}
        \cN(c\Rb\bmu,\bSigma) & \text{if } y=1,\\
        \cN(-c\Rb\bmu,\bSigma) & \text{if } y=-1,
    \end{cases}
\end{equation}
where $\Rb$ is a rotation matrix and $c>0$ controls the magnitude of the distributional shift.

\begin{theorem}
    \label{theorem:id_less_than_ood}
    If $c|\bmu^\top \bSigma^{-1} \Rb\bmu| \le \bmu^\top \bSigma^{-1} \bmu$, then the expected trace of the feature Hessian is smaller for ID data than for OOD data, \ie, 
    \begin{equation}
        \EE_{\cD_{\mathrm{id}}}\!\left[\tr(\nabla^2_{\tilde{\bx}}\cL(\bw))\right]
        \le
        \EE_{\cD_{\mathrm{ood}}}\!\left[\tr(\nabla^2_{\tilde{\bx}}\cL(\bw))\right].
    \end{equation}
\end{theorem}

\begin{proof}
For $\bz=g(\tilde{\bx};\bw)=[\tilde{\bx}^\top\bw,0]^\top$, the trace of the logit-space Hessian satisfies
\begin{equation}
    \tr(\nabla_{\bz}^2\cL(\bw))
    =
    1-\sum_{i=1}^2
    \left(\frac{\exp z_i}{\exp z_1+1}\right)^2.
\end{equation}

Taking expectation yields
\begin{equation}
    \EE_{\cD}[\tr(\nabla_{\bz}^2\cL(\bw))]
    =
    2\EE_{\cD}\!\left[
    \frac{\exp z_1}{(\exp z_1+1)^2}
    \right].
\end{equation}

A fundamental property of binary classification is that the expected energy loss $\mathbb{E}[\cL(\bw)]$ admits a unique minimizer whose direction aligns with that of the Bayes-optimal classifier~\citep{soudry2018implicit, ji2018risk, nacson2019stochasticgradientdescentseparable, 10.5555/3618408.3619509}. 
In particular, Oh \etal~\citep[Thm.~3.1]{10.5555/3618408.3619509} show that training on ID data converges to the population-risk minimizer $\bw^* = \alpha \bSigma^{-1} \bmu$ for some $\alpha > 0$. 
Without loss of generality, we set $\alpha = 1$.

Under $\cD_{\mathrm{id}}$, $\tilde{\bx}$ follows a balanced mixture of $\cN(\bmu,\bSigma)$ and $\cN(-\bmu,\bSigma)$; 
under $\cD_{\mathrm{ood}}$, the same structure holds with shifted means. 
Hence
\begin{align}
    \label{proof:eqns}
    \EE_{\cD_{\mathrm{id}}}[\tr(\nabla_{\bz}^2\cL(\bw))]
    &=
    2\EE_{\tilde{\bx}\sim\cN(\bmu,\bSigma)}
    \!\left[
    \frac{\exp(\tilde{\bx}^\top\bSigma^{-1}\bmu)}
    {(\exp(\tilde{\bx}^\top\bSigma^{-1}\bmu)+1)^2}
    \right], \\
    \EE_{\cD_{\mathrm{ood}}}[\tr(\nabla_{\bz}^2\cL(\bw))]
    &=
    2\EE_{\tilde{\bx}\sim\cN(c\Rb\bmu,\bSigma)}
    \!\left[
    \frac{\exp(\tilde{\bx}^\top\bSigma^{-1}\bmu)}
    {(\exp(\tilde{\bx}^\top\bSigma^{-1}\bmu)+1)^2}
    \right].
\end{align}

By the affine invariance of the Gaussian distribution, the scalar random variable $x := \tilde{\bx}^\top \bSigma^{-1} \bmu$ follows a normal distribution with variance $\sigma^2 = \bmu^\top \bSigma^{-1} \bmu$. 
Its mean is $\mu_{\mathrm{id}}=\bmu^\top\bSigma^{-1}\bmu$ under $\cD_{\mathrm{id}}$ and $\mu_{\mathrm{ood}}=c\bmu^\top\bSigma^{-1}\Rb\bmu$ under $\cD_{\mathrm{ood}}$. 
Therefore,
\begin{align}
    \EE_{\cD_{\mathrm{id}}}[\tr(\nabla_{\bz}^2\cL(\bw))]
    &=
    2\EE_{x\sim\cN(\mu_{\mathrm{id}},\sigma^2)}[\varphi(x)],\\
    \EE_{\cD_{\mathrm{ood}}}[\tr(\nabla_{\bz}^2\cL(\bw))]
    &=
    2\EE_{x\sim\cN(\mu_{\mathrm{ood}},\sigma^2)}[\varphi(x)],
\end{align}
where $\varphi(x)=\frac{\exp x}{(\exp x+1)^2}$.

Since $\varphi(x)$ is strictly decreasing in $|x|$ for $x\ge0$, the condition
$c|\bmu^\top\bSigma^{-1}\Rb\bmu|\le\bmu^\top\bSigma^{-1}\bmu$
implies
\begin{equation}
    \EE_{\cD_{\mathrm{id}}}[\tr(\nabla_{\bz}^2\cL(\bw))]
    \le
    \EE_{\cD_{\mathrm{ood}}}[\tr(\nabla_{\bz}^2\cL(\bw))].
\end{equation}

Applying the second-order chain rule,
\begin{equation}
    \nabla_{\tilde{\bx}}^2\cL(\bw)
    =
    \bU\nabla_{\bz}^2\cL(\bw)\bU^\top,
    \quad
    \bU=
    \begin{bmatrix}\bw^*&0\end{bmatrix}.
\end{equation}
Using the cyclic trace identity,
\begin{equation}
    \tr(\nabla_{\tilde{\bx}}^2\cL(\bw))
    =
    \|\bw^*\|^2
    \tr(\nabla_{\bz}^2\cL(\bw))/2.
\end{equation}

Thus the inequality also holds in feature space:
\begin{equation}
    \EE_{\cD_{\mathrm{id}}}[\tr(\nabla_{\tilde{\bx}}^2\cL)]
    \le
    \EE_{\cD_{\mathrm{ood}}}[\tr(\nabla_{\tilde{\bx}}^2\cL(\bw))].
\end{equation}
\end{proof}

Theorem~\ref{theorem:id_less_than_ood} formalizes the intuition that OOD samples, which lie in regions of lower model confidence due to their absence during training, tend to exhibit larger traces of the feature Hessian. 
For example, when $\bSigma=\Ib$, the inequality naturally holds for any $c<1$ and rotation $\Rb$, corresponding to shifts toward low-confidence regions.

Overall, this analysis provides theoretical support for using the feature-Hessian trace as an OOD detection score in \fold. 
Despite the simplifying assumptions on the in- and out-of-distribution data (\eg, Gaussian mixtures and rotational shifts) and the omission of partial feature normalization, it offers an intuitive explanation for why OOD samples tend to exhibit larger Hessian traces than ID samples. 
Consistent with this intuition, our empirical results (see~\Cref{tab:feature hessian} in the Appendix) show substantially larger feature-Hessian traces for OOD data, supporting the practical effectiveness of the proposed approach.
\section{Discussion}
\subsection{Feature Magnitude}

The role of feature magnitude depends strongly on the complexity of the underlying recognition task.
Feature magnitude reflects semantic activation intensity, but its discriminative value varies across datasets.
For simpler datasets (\eg, CIFAR-10), feature direction alone is often sufficient for class separation, and excessively large magnitudes mainly introduce overconfident noise, favoring strong suppression ($\alpha\!\approx\!1$).
In contrast, for more complex datasets (\eg, ImageNet) with substantial intra-class variation, feature magnitude contains fine-grained semantic information that complements feature direction, making weaker normalization ($\alpha\!\approx\!0.2$) preferable to avoid discarding useful cues.
These observations explain why the optimal $\alpha$ varies with dataset complexity and suggest that adaptive feature calibration is a promising direction for future work.

\subsection{Theoretical Limitation}
Our theoretical analysis provides intuition rather than a complete characterization.
It relies on a simplified setting and therefore does not fully capture the deep, multiclass regime considered in our experiments.
Accordingly, it motivates the connection between the feature-Hessian trace and curvature rather than providing a formal guarantee.
A tighter characterization of the ID-OOD curvature gap is left for future work.

\section{Related Work}
\emph{Post-hoc methods} derive uncertainty scores from a fixed pre-trained model, including logit-based~\citep{MSP, ODIN, Energy, MLS, tempscale}, distance-based~\citep{MDS, RMDS, haoqi2022vim, sun2022out, sun2021react, ASH, kernelpca, actsub}, and gradient-based methods~\citep{GradNorm, gradpca}. 
Related curvature-based methods estimate parameter-space sharpness via Laplace approximations~\citep{ritter2018scalable, kristiadi2020being}, but incur substantially higher computational cost. 
\emph{Training-time methods} instead improve OOD detection through modified optimization, including outlier exposure~\citep{OutlierExposure}, synthetic outlier generation~\citep{Vos, LargeSamples}, confidence calibration~\citep{lee2018training, confbranch2018arxiv}, logit normalization~\citep{LogitNorm}, contrastive learning~\citep{csi20nips}, regularization without OOD data~\citep{godin20cvpr, yu2019unsupervised}, and loss-landscape regularization~\citep{ravikumar2024curvature}. 

\section{Conclusion}
\label{sec:conclusion}
In this work, we investigated the curvature discrepancy between ID and OOD data to enable efficient post-hoc OOD detection, leading to three main contributions:
\begin{enumerate}[leftmargin=*, noitemsep, label=\roman*.]
    \item We show that OOD inputs exhibit larger Hessian curvature than ID data, with the discrepancy increasing under stronger distributional shifts.
    
    \item We propose \fold, which leverages the feature Hessian and partial normalization to improve ID-OOD separability without parameter-space overhead.

    \item We introduce \autofold, a self-supervised tuning scheme based on ID logit masking for automatic normalization calibration without external data.
\end{enumerate}
Together, these contributions establish a principled and efficient framework for reliable OOD detection in real-world deployment settings.

\section*{Acknowledgement}
This work was partly supported by the Institute of Information \& communications Technology Planning \& Evaluation (IITP) grants funded by the Korea government (MSIT) (RS-2019-II191906, Artificial Intelligence Graduate School Program (POSTECH); RS-2022-II220959, (part2) Few-Shot learning of Causal Inference in Vision and Language for Decision Making; RS-2025-25441838, Development of a human foundation model for human-centric universal artificial intelligence and training of personnel), and the National Research Foundation of Korea (NRF) grants funded by the Korea government (MSIT) (RS-2023-00210466, RS-2026-25500419).

%
%
\bibliographystyle{splncs04}
\bibliography{main}

\appendix
\clearpage

\begin{center}
    {\LARGE \bf Exploiting Local Flatness for Efficient Out-of-Distribution Detection} \\
    \vspace{1.0em}
    
    {\large Supplementary Material} \\
    \vspace{1.0em}
\end{center}

\section{Experimental Details}
\label{app:experimental details}
This section provides additional details of the experimental framework, expanding upon the core setup described in the main text. 
For completeness, we briefly revisit relevant concepts where necessary. 
We first describe the strictly separated validation protocol used to ensure unbiased evaluation. 
We then present the CIFAR and ImageNet benchmark datasets used in our experiments. 
Next, we summarize the suite of post-hoc OOD detection baselines considered for comparison. 
Finally, we detail the implementation settings, including model architectures, computational environment, and the hyperparameter optimization procedures for both \fold and its self-supervised variant, \autofold.

\subsection{Validation Protocol and Hyperparameter Tuning}
\label{app:validation_protocol}
Many OOD detection methods rely on tunable hyperparameters.  
In contrast to earlier studies that occasionally tune these parameters directly on test samples, which may lead to overly optimistic performance estimates, we adopt a strictly separated validation protocol following~\citep{yang2022openood, zhang2023openood}.  
Specifically, we construct independent ID and OOD validation sets, denoted as $\mathcal{D}_{\text{val}}^{\text{ID}}$ and $\mathcal{D}_{\text{val}}^{\text{OOD}}$, respectively.  
The ID validation set is formed by holding out a subset of the original ID test data, consisting of 1,000 samples for CIFAR and 5,000 samples for ImageNet.  
The OOD validation set is carefully curated to ensure that its label space is strictly disjoint from that of the OOD test set, \ie, $\mathcal{Y}_{\text{val}}^{\text{OOD}} \cap \mathcal{Y}_{\text{test}}^{\text{OOD}} = \emptyset$.  
These validation sets are used exclusively for hyperparameter tuning and model selection.  
Consequently, the final test evaluations remain strictly unbiased.

\subsection{Datasets}
\label{app:datasets}
To evaluate OOD detection performance under diverse and realistic conditions, we adopt the standard benchmark suites curated by the OpenOOD framework~\citep{yang2022openood, zhang2023openood}.

\vspace{1em}\noindent\textbf{CIFAR Benchmarks.}\quad
For the CIFAR-based evaluation~\citep{CIFAR-10}, we employ CIFAR-10 and CIFAR-100 as the ID datasets.  
The corresponding near-OOD and far-OOD datasets are defined according to their semantic similarity and visual characteristics relative to the ID data.
When CIFAR-10 serves as the ID dataset, the near-OOD group consists of CIFAR-100 and Tiny ImageNet (TIN)~\citep{tin}, which contain object categories semantically similar to those in the ID data.  
The far-OOD group includes MNIST~\citep{deng2012mnist} (grayscale handwritten digits), SVHN~\citep{svhn} (street-view house numbers), Textures~\citep{cimpoi2014describing} (describable texture patterns), and Places365~\citep{zhou2017places} (diverse scene imagery).  
These datasets differ substantially from CIFAR in both semantic content and low-level statistics.
When CIFAR-100 is used as the ID dataset, the near-OOD datasets consist of CIFAR-10 and TIN, while the far-OOD datasets remain identical to those used in the CIFAR-10 setting.
To strictly prevent evaluation leakage, we follow the protocol of~\citep{yang2022openood} to remove overlapping images between CIFAR and either TIN or Places365.  
Specifically, 1,203 images are removed from TIN due to duplication with CIFAR.  
In addition, 1,000 images from TIN spanning 20 categories are reserved as $\mathcal{D}_{\text{val}}^{\text{OOD}}$, ensuring that they are disjoint from $\mathcal{D}_{\text{test}}^{\text{OOD}}$.  
For the far-OOD datasets, an additional 1,305 images are excluded due to semantic overlap with the ID categories.

\vspace{1em}\noindent\textbf{ImageNet Benchmarks.}\quad
For large-scale evaluation, we utilize ImageNet-1K~\citep{deng2009imagenet} and its computationally tractable subset, ImageNet-200, as the ID datasets.  
From the official ImageNet validation set, 45,000 images are allocated to $\mathcal{D}_{\text{test}}^{\text{ID}}$, while the remaining 5,000 images are reserved as $\mathcal{D}_{\text{val}}^{\text{ID}}$.
Due to the large semantic diversity of the ImageNet label space, near-OOD and far-OOD splits are defined empirically based on detection difficulty rather than strict semantic similarity.  
The near-OOD group includes SSB-Hard~\citep{SSB-Hard}, which consists of semantically similar classes sampled from ImageNet-21K, and NINCO~\citep{ninco}, a manually curated dataset of challenging OOD samples.  
The far-OOD group includes iNaturalist~\citep{van2018inaturalist} (fine-grained biological categories), Textures~\citep{cimpoi2014describing} (describable texture patterns), and OpenImage-O~\citep{haoqi2022vim}, a diverse subset of Open Images.
Following the OpenOOD protocol, 1,763 images from OpenImage-O are designated as $\mathcal{D}_{\text{val}}^{\text{OOD}}$, and the remaining samples are used as $\mathcal{D}_{\text{test}}^{\text{OOD}}$.

\subsection{Baseline Methods}
We compare our method against a comprehensive suite of post-hoc OOD detection baselines following the taxonomy of the OpenOOD framework~\citep{zhang2023openood}. 
\textit{MSP}~\citep{MSP} uses the maximum softmax probability over ID classes as a confidence score. 
\textit{OpenMax}~\citep{openmax16cvpr} modifies the softmax layer to estimate the probability that an input belongs to an unknown class. 
\textit{TempScale}~\citep{tempscale} improves confidence calibration by applying temperature scaling to the softmax logits. 
\textit{ODIN}~\citep{ODIN} extends temperature scaling by adding small input perturbations to further separate ID and OOD score distributions. 
\textit{MDS}~\citep{MDS} models penultimate-layer features using class-conditional Gaussian distributions and employs the Mahalanobis distance as the detection score. 
Its variant, \textit{MDSEns}~\citep{MDS}, aggregates Mahalanobis scores computed from multiple intermediate layers. 
\textit{RMDS}~\citep{RMDS} extends the Mahalanobis framework by incorporating a background score derived from an unconditional Gaussian distribution. 
\textit{EBO}~\citep{Energy} adopts an energy-based formulation that converts logits into an energy score reflecting the likelihood of a sample belonging to the ID distribution. 
\textit{ReAct}~\citep{sun2021react} suppresses abnormal activations by truncating feature values that exceed a predefined threshold. 
\textit{MLS}~\citep{MLS} uses the maximum logit value directly as the detection score without applying the softmax transformation. 
\textit{VIM}~\citep{haoqi2022vim} combines the logit-based score with the norm of the residual feature component orthogonal to the principal ID subspace. 
\textit{KNN}~\citep{sun2022out} performs a $k$-nearest neighbor search on penultimate-layer features to estimate the distance from the ID training distribution. 
\textit{ASH}~\citep{ASH} applies activation shaping by pruning low-magnitude activations and simplifying the remaining ones to suppress spurious signals. 
\textit{SHE}~\citep{SHE} constructs template features for each ID class and computes the OOD score based on the distance between the test feature and its nearest class template.

\subsection{Implementation Details}
\label{app:implementation_details}
For models trained on CIFAR-10, CIFAR-100, and ImageNet-200, we employ a ResNet-18 backbone~\citep{resnet}. 
For large-scale evaluation on ImageNet-1K, we utilize a pre-trained ResNet-50 model provided by the OpenOOD framework~\citep{zhang2023openood}. 
To ensure statistical robustness, all results for the CIFAR datasets and ImageNet-200 are reported as the average over three independent runs with different random seeds (0, 1, and 2). 
For ImageNet-1K, we follow the standard protocol and report the performance of the publicly available pre-trained checkpoint.
Hessian-related statistics, including the trace, the largest eigenvalue, and the empirical spectral density discussed in \Cref{sec:observations} and \Cref{sec:further analysis}, are computed using the approximation procedures implemented in PyHessian~\citep{yao2020pyhessian}. 
All computational procedures, including evaluation and Hessian approximations, are conducted on NVIDIA A100 GPUs.
Our proposed methods introduce a normalization coefficient $\alpha$, whose search space is defined over the interval $(0, 1]$. 
The optimization protocol for this parameter differs depending on the specific variant. 
For \fold, we perform a grid search with a step size of 0.1 (\ie, $\alpha \in \{0.1, 0.2, \dots, 1.0\}$), and select the optimal value based on the AUROC measured on the auxiliary OOD validation set described in \Cref{app:datasets}. 
In contrast, for \autofold, we employ a finer grid with a step size of 0.01 (\ie, $\alpha \in \{0.01, 0.02, \dots, 1.00\}$). 
As detailed in \Cref{subsec:autofold}, the optimal $\alpha$ for \autofold is determined in a strictly self-supervised manner using only the ID validation set. 
Specifically, $\alpha$ is selected by maximizing the AUROC between the original ID validation samples and their logit-masked pseudo-OOD counterparts, thereby eliminating any reliance on prior OOD knowledge.

\section{Detailed Experimental Results}
\label{app:additional results}
This section presents additional experimental results that complement and further support the findings reported in the main text.
For clarity, the analysis is organized into three parts.
First, we examine the structural properties of the feature-space Hessian and show its correspondence with the parameter-space Hessian through spectral statistics and ESD visualizations.
Second, we extend the curvature analysis by presenting sample-wise Hessian trace distributions, providing a more detailed view of the geometric distinctions between ID and OOD data across multiple benchmarks.
Finally, we report per-dataset OOD detection results to assess the stability and generalization of the proposed framework relative to existing baselines.

\begin{table}[t]
    \centering
    \caption{
        Comparison of the largest eigenvalue and the trace of the feature Hessian across different distribution regimes.
        Consistent with the trends observed for the full Hessian in \Cref{tab:full parameter hessian statistics}, the feature Hessian exhibits a similar pattern: ID samples attain smaller values, while both metrics increase progressively from ID to near-OOD to far-OOD.
        This alignment indicates that the distribution-dependent sensitivity observed at the full-parameter level is likewise reflected in the feature space.
        }
    \label{tab:feature hessian}
    \vspace{-0.75em}
    \resizebox{\linewidth}{!}{%
        \begin{sc}
        \begin{tabular}{l cc cc cc cc} 
            \toprule
            & \multicolumn{2}{c}{\textbf{CIFAR-10}} & \multicolumn{2}{c}{\textbf{CIFAR-100}} & \multicolumn{2}{c}{\textbf{ImageNet-200}} & \multicolumn{2}{c}{\textbf{ImageNet-1K}} \\ 
            \cmidrule(lr){2-3} \cmidrule(lr){4-5} \cmidrule(lr){6-7} \cmidrule(lr){8-9}
            \textbf{Dataset} & $\lambda_{\max}(H)_{ \times10^{-3}}$ & $\tr(H)_{ \times10^{-3}}$ & $\lambda_{\max}(H)_{ \times10^{-3}}$ & $\tr(H)_{ \times10^{-3}}$ & $\lambda_{\max}(H)_{ \times10^{-3}}$ & $\tr(H)_{ \times10^{-3}}$ & $\lambda_{\max}(H)_{ \times10^{-3}}$ & $\tr(H)_{ \times10^{-3}}$ \\ 
            \midrule
            
            ID
            & \alignednum{36.3}[5.2] & \alignednum{161.7}[29.0] & \alignednum{16.3}[0.3] & \alignednum{546.3}[4.1] & \alignednum{10.5}[0.2] & \alignednum{484.1}[2.4] & \alignednum{6.3} & \alignednum{650.8} \\
            
            Near-OOD
            & \alignednum{263.2}[43.0] & \alignednum{1184.7}[150.3] & \alignednum{50.3}[1.1] & \alignednum{1394.5}[5.1] & \alignednum{33.2}[0.4] & \alignednum{1140.7}[6.4] & \alignednum{18.7} & \alignednum{1082.6} \\
            
            Far-OOD
            & \alignednum{441.6}[70.8] & \alignednum{1362.5}[212.0] & \alignednum{141.6}[2.2] & \alignednum{1262.4}[21.1] & \alignednum{67.4}[1.5] & \alignednum{1415.4}[5.6] & \alignednum{75.5} & \alignednum{1504.7} \\
            \bottomrule
        \end{tabular}
        \end{sc}
    }
\end{table}

\begin{table}[t!]
    \centering
    \caption{
        Hessian characteristics across OOD datasets in the CIFAR and ImageNet benchmarks.
        }
    \label{tab: comprehensive hessian all}
    \vspace{-1.5em}
    \begin{subtable}{\linewidth}
        \centering
        \label{tab: comprehensive full parameter hessian for cifar datasets}
        \resizebox{\linewidth}{!}{%
            \begin{sc}
            \begin{tabular}{l @{\hspace{1em}} l @{\hspace{1.5em}} c @{\hspace{1.5em}} c @{\hspace{1em}} c @{\hspace{1em}} c @{\hspace{1.5em}} c @{\hspace{1em}} c @{\hspace{1em}} c @{\hspace{1em}} c} 
                \toprule
                \multirow{2}{*}{\textbf{Dataset}} & \multirow{2}{*}{\textbf{Metric}} & \multirow{2}{*}{\textbf{ID}} & \multicolumn{3}{c}{\textbf{Near-OOD}} & \multicolumn{4}{c}{\textbf{Far-OOD}} \\
                \cmidrule(lr){4-6} \cmidrule(lr){7-10}
                & & & CIFAR-10 & CIFAR-100 & TIN & MNIST & SVHN & Textures & Places365 \\
                \midrule
                & $\lambda_{\max}(H)$ & \alignednum{202.8}[22.9] & \textcolor{gray}{N / A} & \alignednum{1058.7}[63.8] & \alignednum{1397.5}[68.2] & \alignednum{3087.2}[480.0] & \alignednum{6977.6}[1679.6] & \alignednum{1894.6}[132.8] & \alignednum{1709.7}[32.2] \\
                \multirow{-2}{*}{CIFAR-10} & $\tr(H)$ & \alignednum{2910.5}[717.5] & \textcolor{gray}{N / A} & \alignednum{20388.9}[761.3] & \alignednum{21362.5}[1078.1] & \alignednum{22582.2}[1869.3] & \alignednum{24676.5}[3487.8] & \alignednum{23339.6}[1825.5] & \alignednum{20935.9}[792.8] \\
                \midrule
                & $\lambda_{\max}(H)$ & \alignednum{316.9}[3.4] & \alignednum{594.6}[17.7] & \textcolor{gray}{N / A} & \alignednum{768.7}[5.0] & \alignednum{2842.5}[298.2] & \alignednum{4765.7}[236.6] & \alignednum{937.2}[90.3] & \alignednum{845.1}[26.2] \\
                \multirow{-2}{*}{CIFAR-100} & $\tr(H)$ & \alignednum{10577.8}[751.7] & \alignednum{24213.9}[949.5] & \textcolor{gray}{N / A} & \alignednum{26680.2}[1031.1] & \alignednum{22292.2}[2509.0] & \alignednum{27152.4}[2717.5] & \alignednum{29686.8}[1212.1] & \alignednum{25002.7}[812.1] \\
                \bottomrule
            \end{tabular}
            \end{sc}
        }
        \caption{\textit{Full-parameter Hessian} characteristics across OOD datasets in the CIFAR benchmarks.}
    \end{subtable}
    
    \begin{subtable}{\linewidth}
        \centering
        \label{tab: comprehensive feature hessian for cifar datasets}
        \resizebox{\linewidth}{!}{%
            \begin{sc}
            \begin{tabular}{l @{\hspace{1em}} l @{\hspace{1.5em}} c @{\hspace{1.5em}} c @{\hspace{1em}} c @{\hspace{1em}} c @{\hspace{1.5em}} c @{\hspace{1em}} c @{\hspace{1em}} c @{\hspace{1em}} c} 
                \toprule
                \multirow{2}{*}{\textbf{Dataset}} & \multirow{2}{*}{\textbf{Metric}} & \multirow{2}{*}{\textbf{ID}} & \multicolumn{3}{c}{\textbf{Near-OOD}} & \multicolumn{4}{c}{\textbf{Far-OOD}} \\
                \cmidrule(lr){4-6} \cmidrule(lr){7-10}
                & & & CIFAR-10 & CIFAR-100 & TIN & MNIST & SVHN & Textures & Places365 \\
                \midrule
                & $\lambda_{\max}(H)_{ \times10^{-3}}$ & \alignednum{36.3}[5.2] & \textcolor{gray}{N / A} & \alignednum{245.3}[41.4] & \alignednum{281.0}[44.8] & \alignednum{559.9}[131.9] & \alignednum{548.2}[80.2] & \alignednum{385.0}[37.6] & \alignednum{273.2}[36.0] \\
                \multirow{-2}{*}{CIFAR-10} & $\tr(H)_{ \times10^{-3}}$ & \alignednum{161.7}[29.0] & \textcolor{gray}{N / A} & \alignednum{1129.8}[154.2] & \alignednum{1239.6}[146.6] & \alignednum{1555.5}[278.4] & \alignednum{1338.0}[294.9] & \alignednum{1306.5}[212.9] & \alignednum{1250.0}[140.7] \\
                \midrule
                & $\lambda_{\max}(H)_{ \times10^{-3}}$ & \alignednum{16.3}[0.3] & \alignednum{54.0}[2.0] & \textcolor{gray}{N / A} & \alignednum{46.7}[0.3] & \alignednum{221.3}[7.5] & \alignednum{191.8}[5.9] & \alignednum{90.4}[5.0] & \alignednum{63.0}[2.8] \\
                \multirow{-2}{*}{CIFAR-100} & $\tr(H)_{ \times10^{-3}}$ & \alignednum{546.3}[4.1] & \alignednum{1318.0}[7.4] & \textcolor{gray}{N / A} & \alignednum{1470.9}[3.5] & \alignednum{1208.4}[82.8] & \alignednum{1291.7}[39.6] & \alignednum{1254.2}[23.6] & \alignednum{1295.2}[11.0] \\
                \bottomrule
            \end{tabular}
            \end{sc}
        }
        \caption{\textit{Feature Hessian} characteristics across OOD datasets in the CIFAR benchmarks.}
    \end{subtable}
    
    \begin{subtable}{\linewidth}
        \centering
        \label{tab: comprehensive full parameter hessian for imagenet datasets}
        \resizebox{\linewidth}{!}{%
            \begin{sc}
            \begin{tabular}{l @{\hspace{1em}} l @{\hspace{1.5em}} c @{\hspace{1.5em}} c @{\hspace{1em}} c @{\hspace{1.5em}} c @{\hspace{1em}} c @{\hspace{1em}} c} 
                \toprule
                \multirow{2}{*}{\textbf{Dataset}} & \multirow{2}{*}{\textbf{Metric}} & \multirow{2}{*}{\textbf{ID}} & \multicolumn{2}{c}{\textbf{Near-OOD}} & \multicolumn{3}{c}{\textbf{Far-OOD}} \\
                \cmidrule(lr){4-5} \cmidrule(lr){6-8}
                & & & SSB-Hard & NINCO & iNaturalist & Textures & OpenImage-O \\
                \midrule
                & $\lambda_{\max}(H)$ & \alignednum{548.7}[40.2] & \alignednum{883.9}[34.1] & \alignednum{999.3}[37.9] & \alignednum{2586.1}[208.0] & \alignednum{1136.9}[9.8] & \alignednum{1159.0}[31.6] \\
                \multirow{-2}{*}{ImageNet-200} & $\tr(H)$ & \alignednum{21489.6}[2558.6] & \alignednum{42552.4}[3515.1] & \alignednum{48855.0}[2801.2] & \alignednum{50103.7}[3418.80] & \alignednum{63892.1}[2727.9] & \alignednum{47224.6}[247.9] \\
                \midrule
                & $\lambda_{\max}(H)$ & \alignednum{702.4} & \alignednum{1105.2} & \alignednum{1235.7} & \alignednum{2920.0} & \alignednum{1475.7} & \alignednum{1520.1} \\
                \multirow{-2}{*}{ImageNet-1K} & $\tr(H)$ & \alignednum{26834.2} & \alignednum{39022.0} & \alignednum{54064.3} & \alignednum{63005.4} & \alignednum{64076.6} & \alignednum{56806.4} \\
                \bottomrule
            \end{tabular}
            \end{sc}
        }
        \caption{\textit{Full-parameter Hessian} characteristics across OOD datasets in the ImageNet benchmarks.}
    \end{subtable}
    
    \begin{subtable}{\linewidth}
        \centering
        \label{tab: comprehensive feature hessian for imagenet datasets}
        \resizebox{\linewidth}{!}{%
            \begin{sc}
            \begin{tabular}{l @{\hspace{1em}} l @{\hspace{1.5em}} c @{\hspace{1.5em}} c @{\hspace{1em}} c @{\hspace{1.5em}} c @{\hspace{1em}} c @{\hspace{1em}} c} 
                \toprule
                \multirow{2}{*}{\textbf{Dataset}} & \multirow{2}{*}{\textbf{Metric}} & \multirow{2}{*}{\textbf{ID}} & \multicolumn{2}{c}{\textbf{Near-OOD}} & \multicolumn{3}{c}{\textbf{Far-OOD}} \\
                \cmidrule(lr){4-5} \cmidrule(lr){6-8}
                & & & SSB-Hard & NINCO & iNaturalist & Textures & OpenImage-O \\
                \midrule
                & $\lambda_{\max}(H)_{ \times10^{-3}}$ & \alignednum{10.5}[0.2] & \alignednum{30.4}[0.8] & \alignednum{36.0}[1.5] & \alignednum{76.0}[2.2] & \alignednum{66.6}[3.1] & \alignednum{68.2}[1.9] \\
                \multirow{-2}{*}{ImageNet-200} & $\tr(H)_{ \times10^{-3}}$ & \alignednum{484.1}[2.4] & \alignednum{1041.6}[2.9] & \alignednum{1239.8}[10.2] & \alignednum{1580.3}[17.2] & \alignednum{1447.9}[9.1] & \alignednum{1382.9}[3.2] \\
                \midrule
                & $\lambda_{\max}(H)_{ \times10^{-3}}$ & \alignednum{6.3} & \alignednum{14.2} & \alignednum{23.2} & \alignednum{103.5} & \alignednum{52.9} & \alignednum{70.2} \\
                \multirow{-2}{*}{ImageNet-1K} & $\tr(H)_{ \times10^{-3}}$ & \alignednum{650.8} & \alignednum{930.5} & \alignednum{1234.6} & \alignednum{1529.1} & \alignednum{1511.7} & \alignednum{1473.2} \\
                \bottomrule
            \end{tabular}
            \end{sc}
        }
        \caption{\textit{Feature Hessian} characteristics across OOD datasets in the ImageNet benchmarks.}
    \end{subtable}

\end{table}

\subsection{Feature Hessian}
To assess whether the feature-space Hessian behaves similarly to the parameter-space Hessian in separating ID and OOD data, we analyze key curvature metrics (\eg, the trace and the largest eigenvalue).
As shown in \Cref{tab:feature hessian}, although their absolute magnitudes differ due to dimensionality and parameterization, the overall trends remain consistent.
Both the largest eigenvalue and the trace increase monotonically as the evaluated distribution deviates further from ID data.
This trend matches the parameter-space Hessian statistics reported in \Cref{tab:full parameter hessian statistics}.
These results suggest that the feature-space approximation preserves the essential structural properties of the full parameter-space Hessian.

Building on these observations, we provide a more detailed analysis by reporting dataset-specific statistics for both the full parameter Hessian and the feature Hessian. 
\Cref{tab: comprehensive full parameter hessian for cifar datasets} summarizes these results across the CIFAR and ImageNet benchmark suites. 
A dataset-level examination reveals consistent trends across the evaluation settings. 
In particular, the spectral curvature metrics follow a clear monotonic pattern: they are lowest for ID data, increase for near-OOD samples, and reach their highest values for far-OOD inputs. 
These results further support the use of Hessian-based curvature metrics as indicators of distribution shift.

To complement these scalar statistics, we visualize the ESD of both the parameter-space and feature-space Hessians for the CIFAR benchmarks.
The corresponding plots are shown in~\Cref{fig:Hessian ESD for CIFAR-10 across datasets} and \Cref{fig:Feature Hessian ESD for CIFAR-10 across datasets}.
Consistent with the scalar metrics, the spectral distributions exhibit strong qualitative agreement between the two Hessian representations.
Analogous ESD visualizations for the ImageNet benchmarks are presented in~\Cref{fig:Hessian ESD for ImageNet-200 across datasets} and \Cref{fig:Feature Hessian ESD for ImageNet-200 across datasets}.
These results further support the structural correspondence between the two Hessian formulations.

\vspace{5em}

\begin{figure*}[h!]
    \centering
    \begin{subfigure}{0.245\linewidth}
        \centering
        \includegraphics[width=\linewidth]{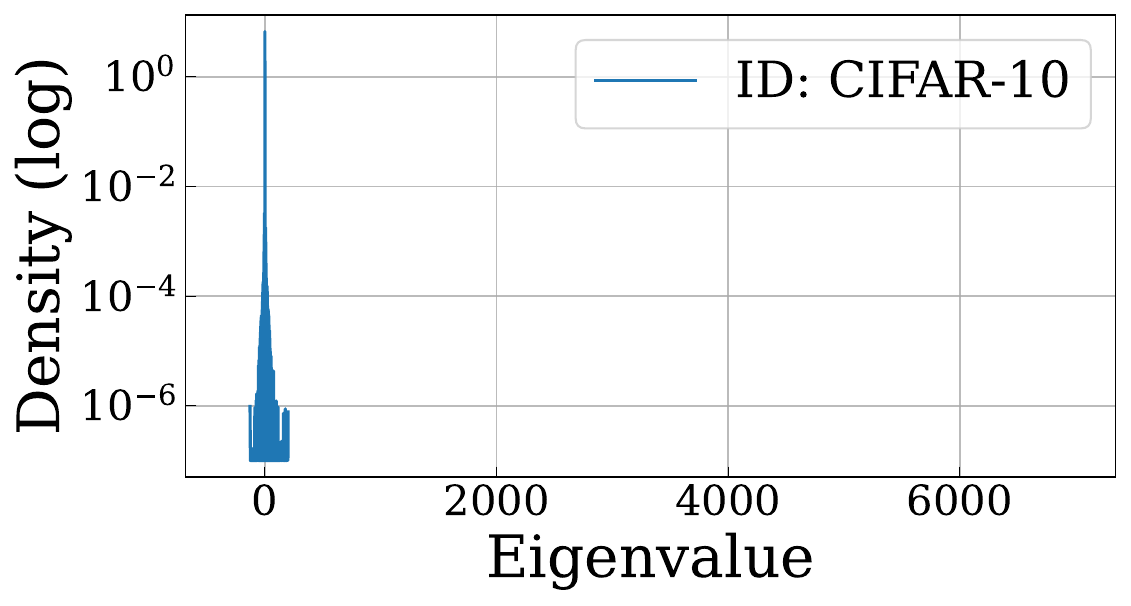}
        \subcaption{CIFAR-10}
    \end{subfigure}
    \begin{subfigure}{0.245\linewidth}
        \centering
        \includegraphics[width=\linewidth]{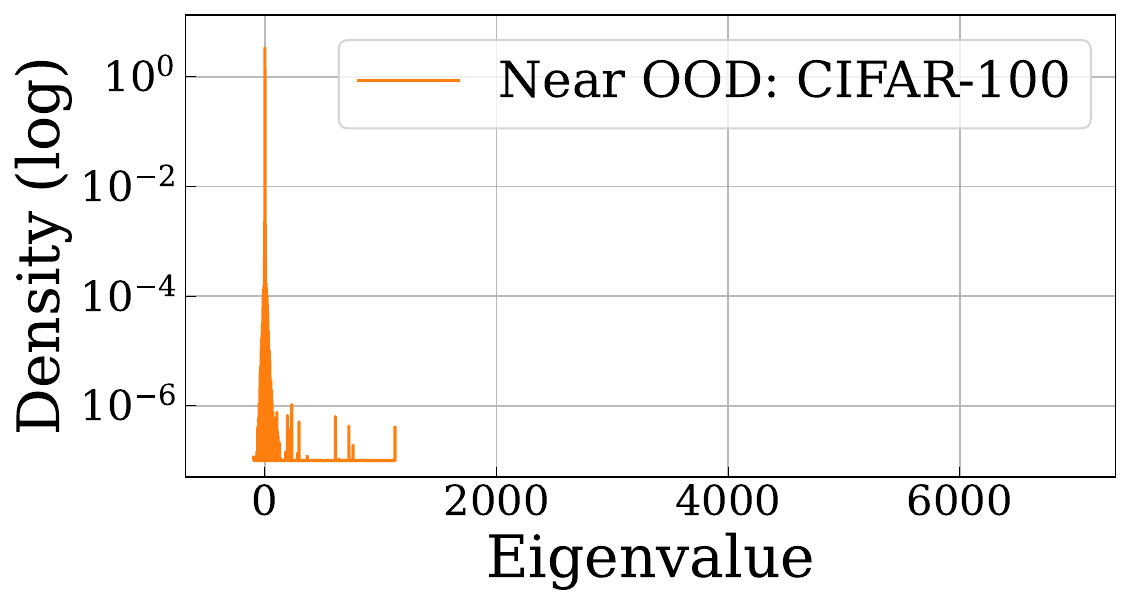}
        \subcaption{CIFAR-100}
    \end{subfigure}
    \begin{subfigure}{0.245\linewidth}
        \centering
        \includegraphics[width=\linewidth]{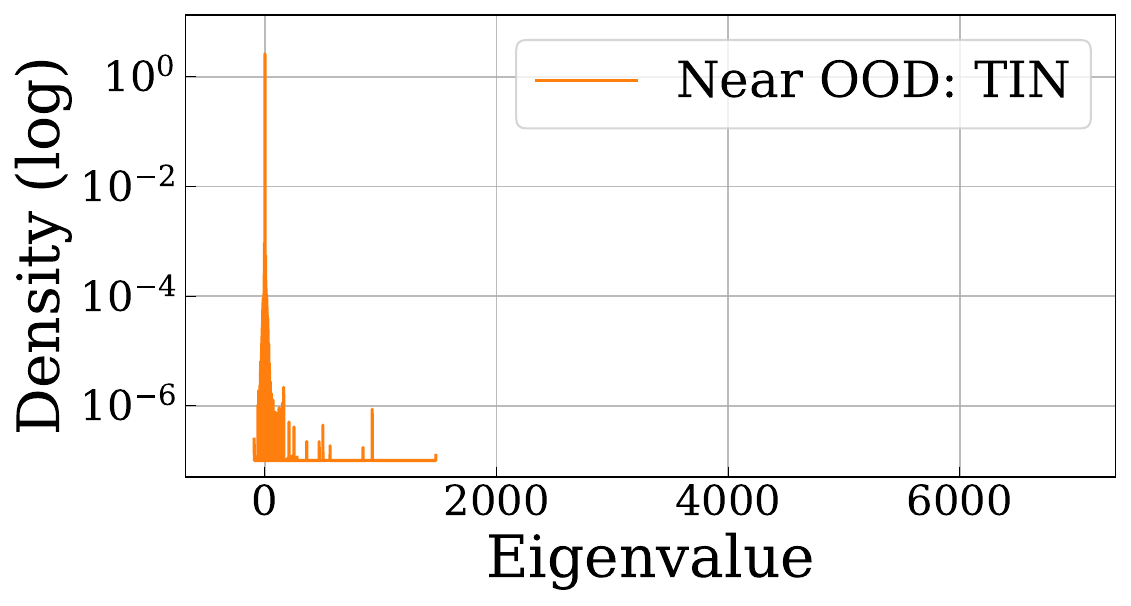}
        \subcaption{TIN}
    \end{subfigure}
    
    \begin{subfigure}{0.245\linewidth}
        \centering
        \includegraphics[width=\linewidth]{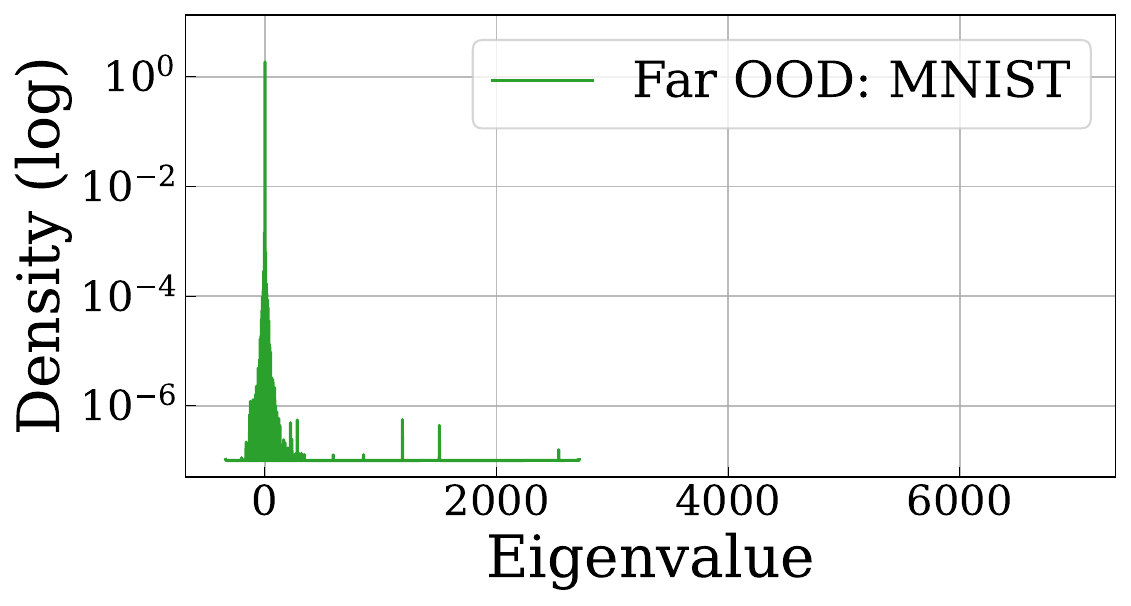}
        \subcaption{MNIST}
    \end{subfigure}
    \begin{subfigure}{0.245\linewidth}
        \centering
        \includegraphics[width=\linewidth]{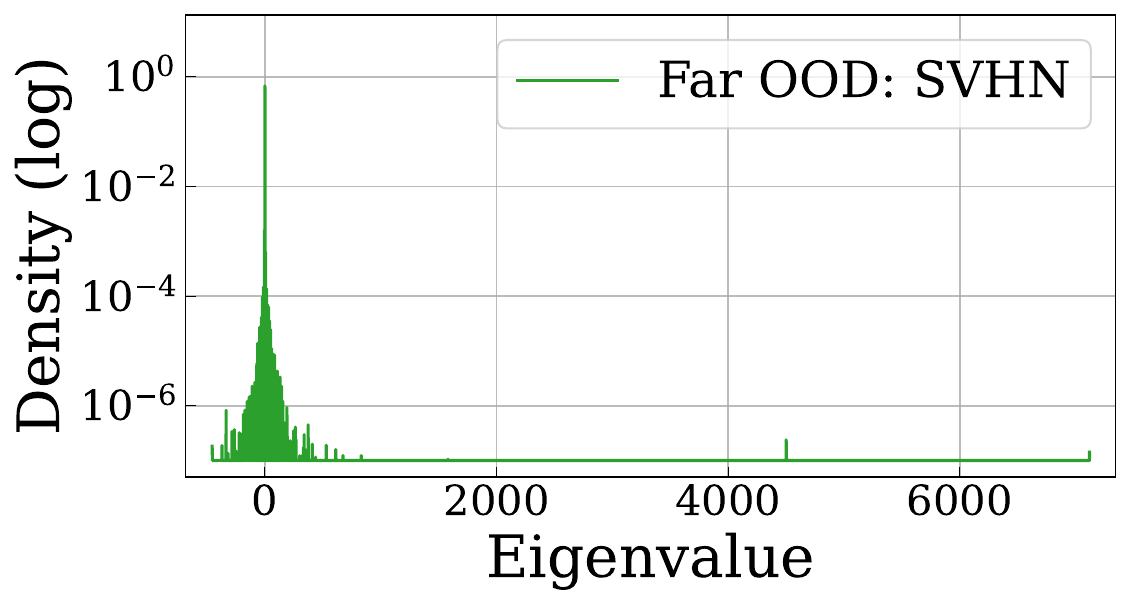}
        \subcaption{SVHN}
    \end{subfigure}
    \begin{subfigure}{0.245\linewidth}
        \centering
        \includegraphics[width=\linewidth]{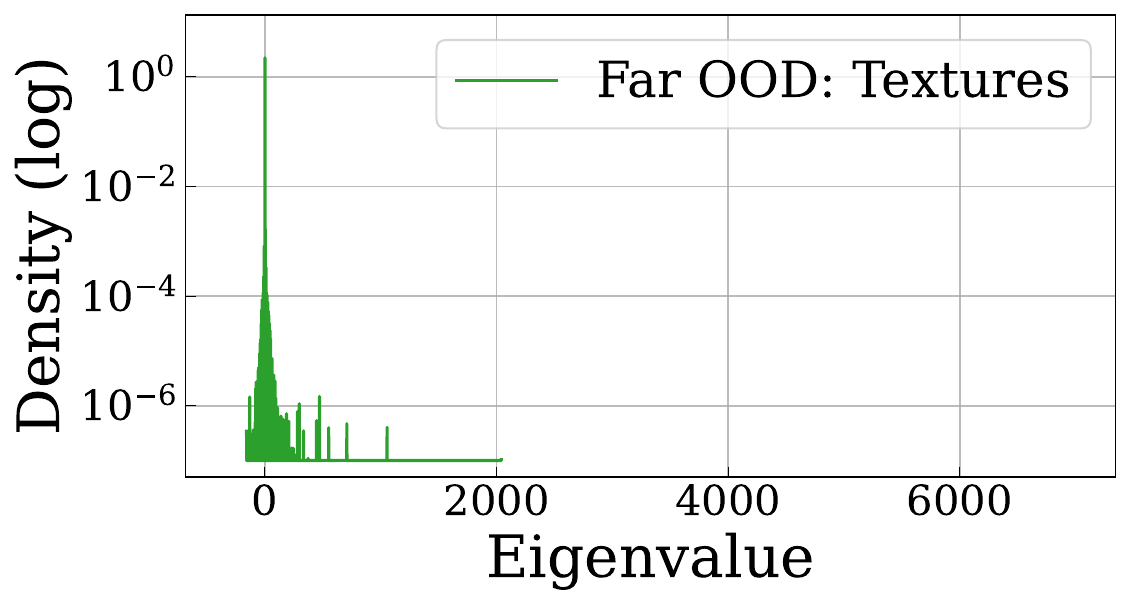}
        \subcaption{Textures}
    \end{subfigure}
    \begin{subfigure}{0.245\linewidth}
        \centering
        \includegraphics[width=\linewidth]{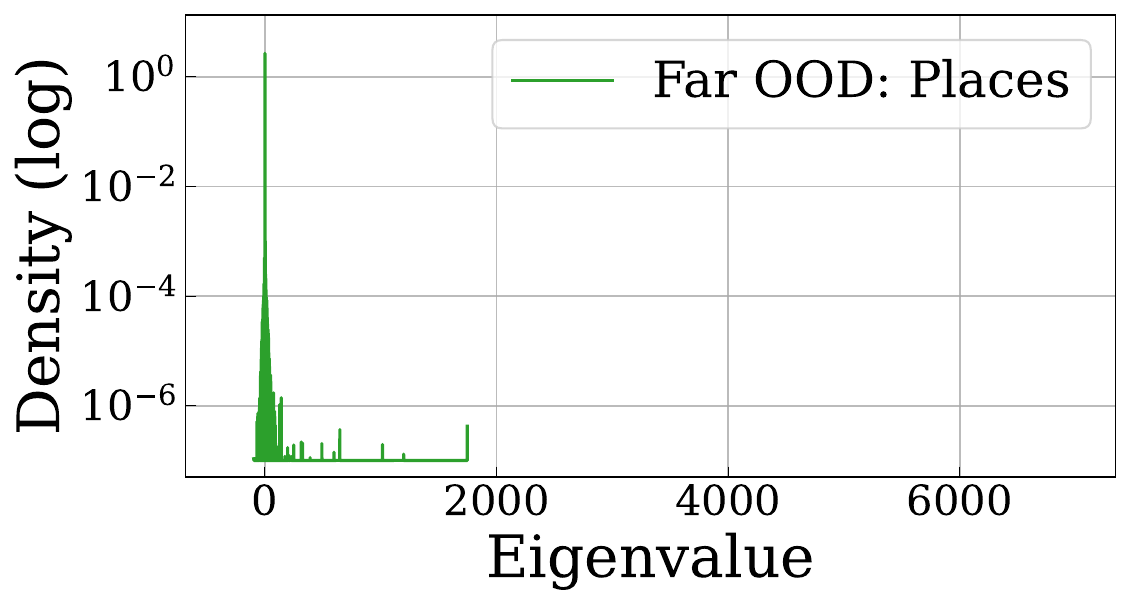}
        \subcaption{Places}
    \end{subfigure}
    \caption{
        \textit{Parameter-space Hessian} ESD of CIFAR-10–trained models across multiple OOD datasets.
        Colors indicate ID (\textcolor{dblue}{blue}), near-OOD (\textcolor{dorange}{orange}), and far-OOD (\textcolor{dgreen}{green}).
        }
    \label{fig:Hessian ESD for CIFAR-10 across datasets}
\end{figure*}

\begin{figure*}[h!]
    \centering
    \begin{subfigure}{0.245\linewidth}
        \centering
        \includegraphics[width=\linewidth]{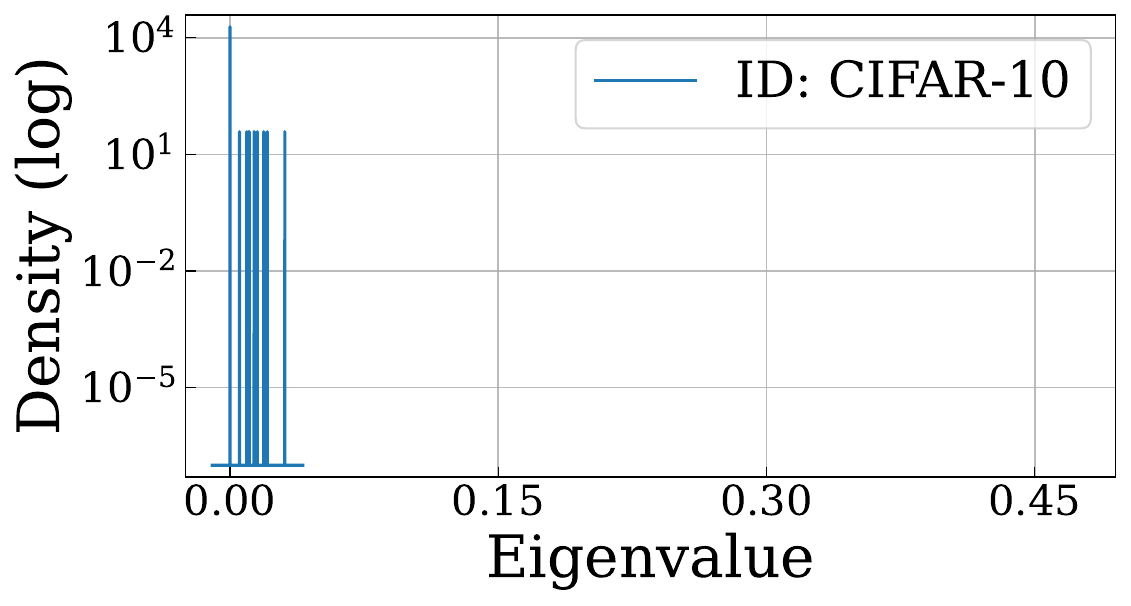}
        \subcaption{CIFAR-10}
    \end{subfigure}
    \begin{subfigure}{0.245\linewidth}
        \centering
        \includegraphics[width=\linewidth]{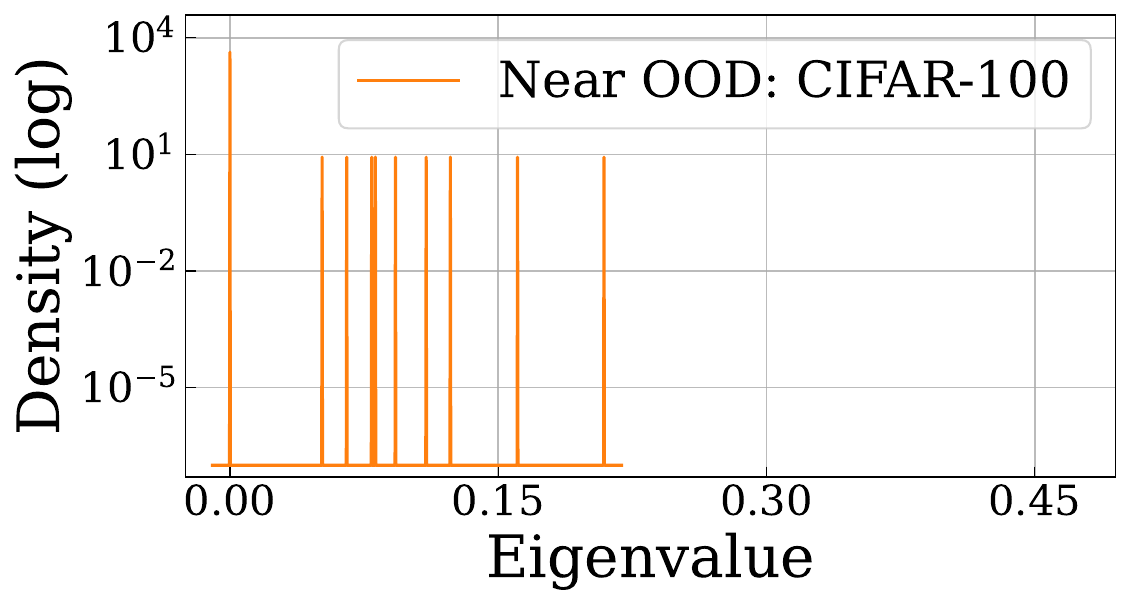}
        \subcaption{CIFAR-100}
    \end{subfigure}
    \begin{subfigure}{0.245\linewidth}
        \centering
        \includegraphics[width=\linewidth]{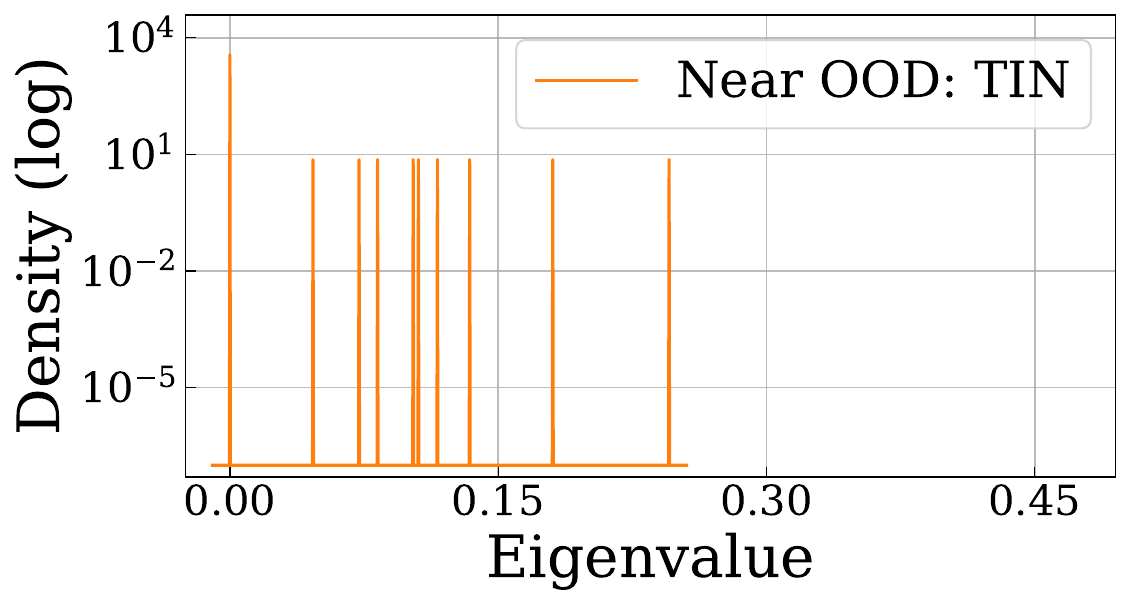}
        \subcaption{TIN}
    \end{subfigure}
    
    \begin{subfigure}{0.245\linewidth}
        \centering
        \includegraphics[width=\linewidth]{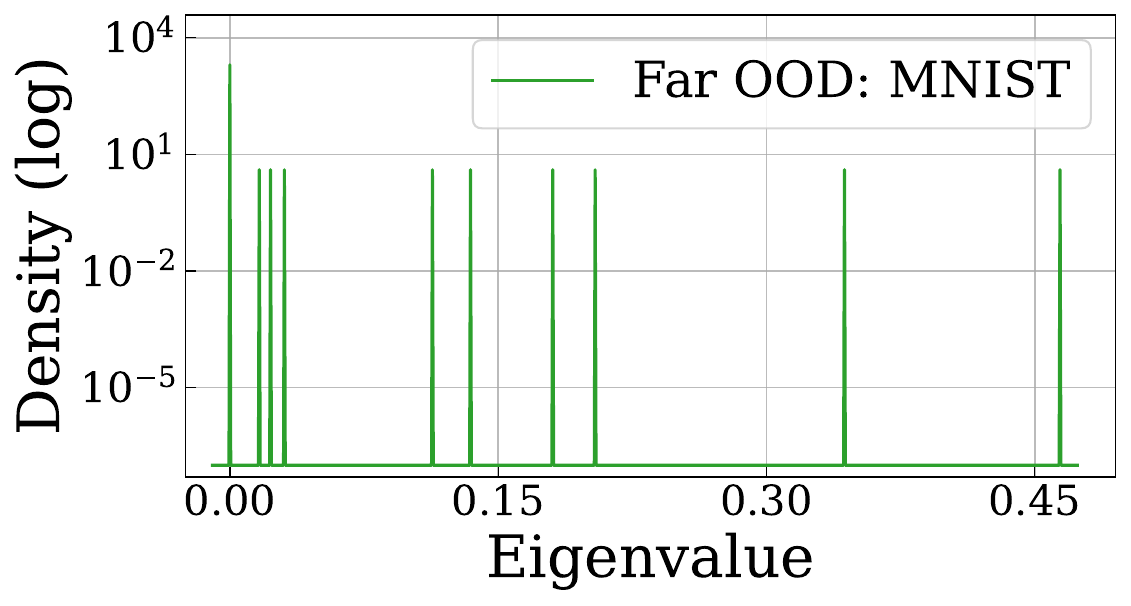}
        \subcaption{MNIST}
    \end{subfigure}
    \begin{subfigure}{0.245\linewidth}
        \centering
        \includegraphics[width=\linewidth]{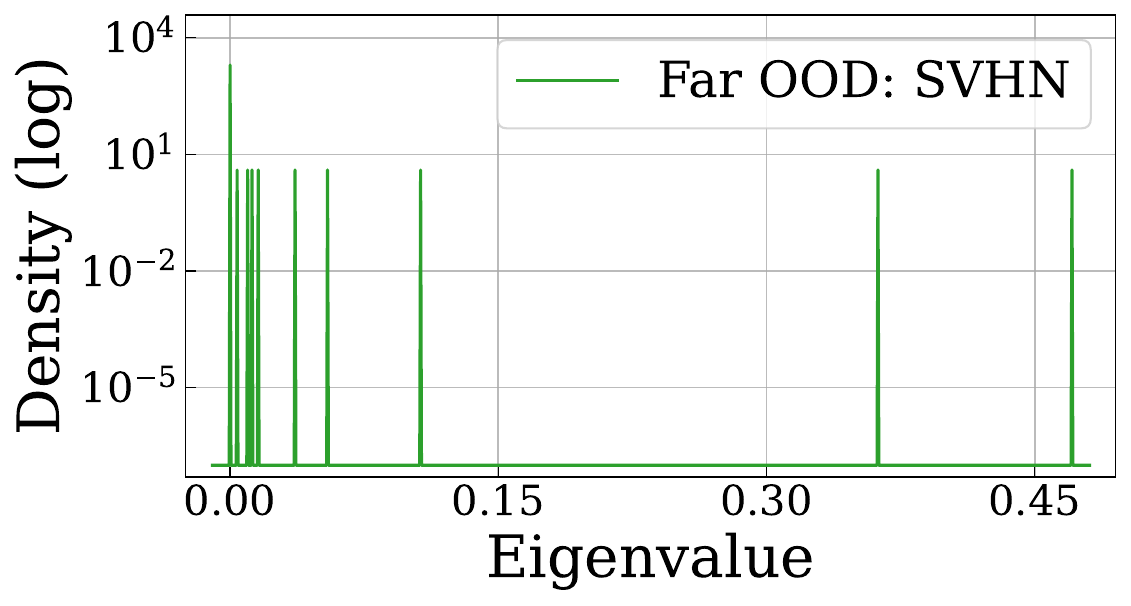}
        \subcaption{SVHN}
    \end{subfigure}
    \begin{subfigure}{0.245\linewidth}
        \centering
        \includegraphics[width=\linewidth]{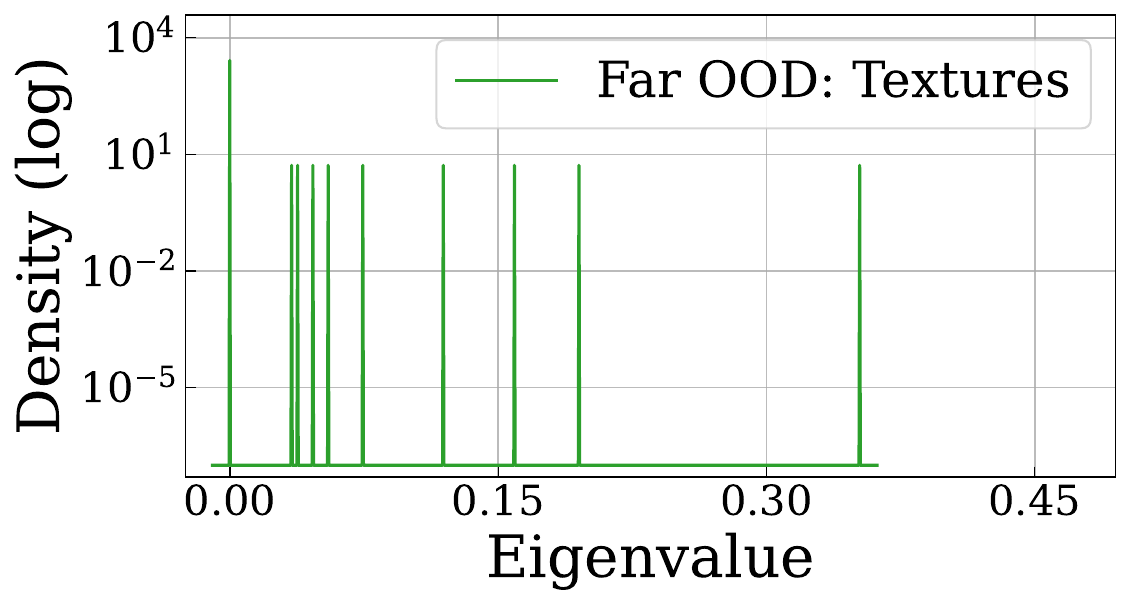}
        \subcaption{Textures}
    \end{subfigure}
    \begin{subfigure}{0.245\linewidth}
        \centering
        \includegraphics[width=\linewidth]{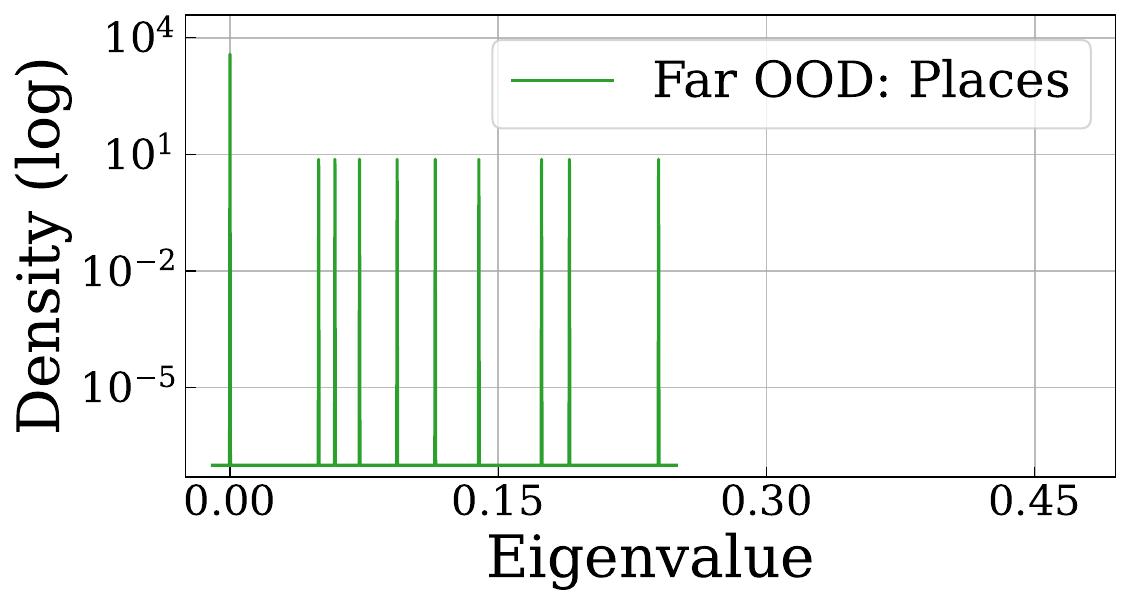}
        \subcaption{Places}
    \end{subfigure}
    \caption{
        \textit{Feature Hessian} ESD of CIFAR-10–trained models across multiple OOD datasets.
        Colors indicate ID (\textcolor{dblue}{blue}), near-OOD (\textcolor{dorange}{orange}), and far-OOD (\textcolor{dgreen}{green}).
        }
    \label{fig:Feature Hessian ESD for CIFAR-10 across datasets}
\end{figure*}

\clearpage

\begin{figure*}[t!]
    \centering
    \begin{subfigure}{0.245\linewidth}
        \centering
        \includegraphics[width=\linewidth]{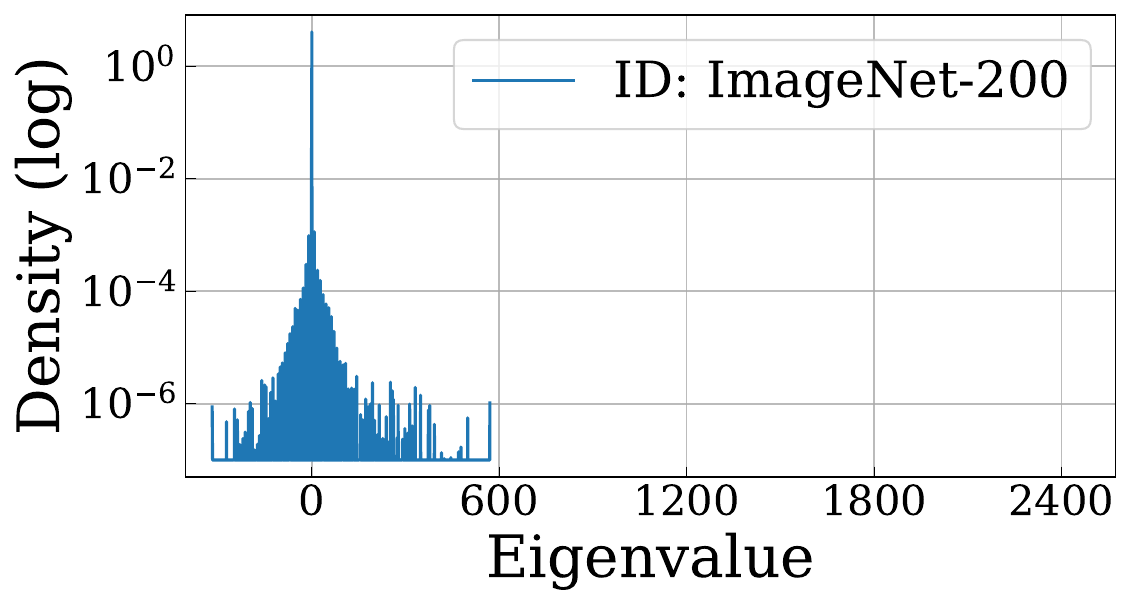}
        \subcaption{ImageNet-200}
    \end{subfigure}
    \begin{subfigure}{0.245\linewidth}
        \centering
        \includegraphics[width=\linewidth]{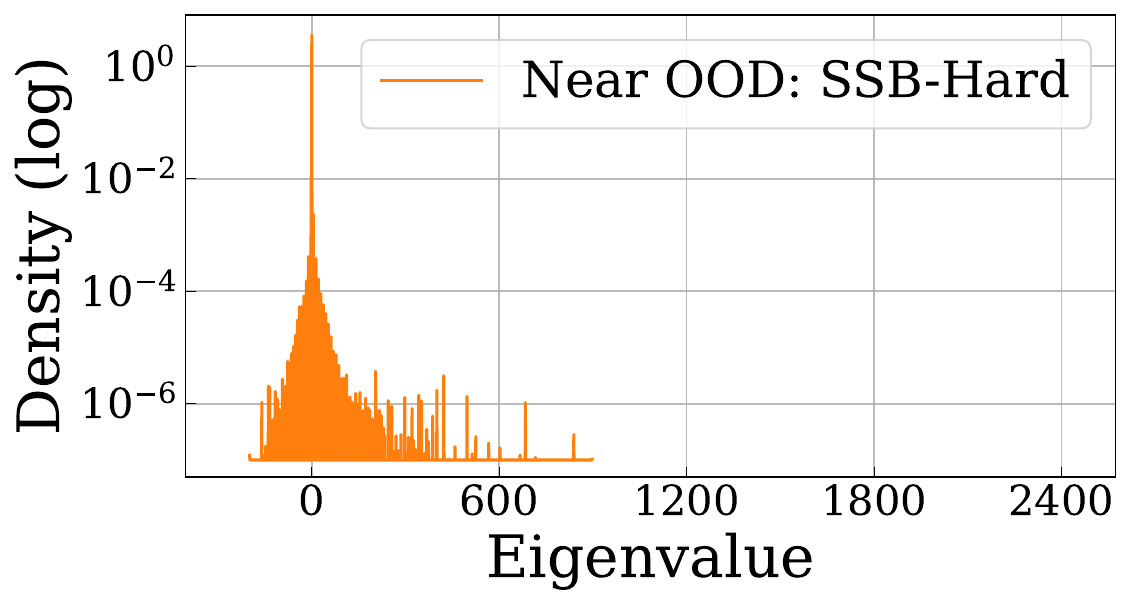}
        \subcaption{SSB-Hard}
    \end{subfigure}
    \begin{subfigure}{0.245\linewidth}
        \centering
        \includegraphics[width=\linewidth]{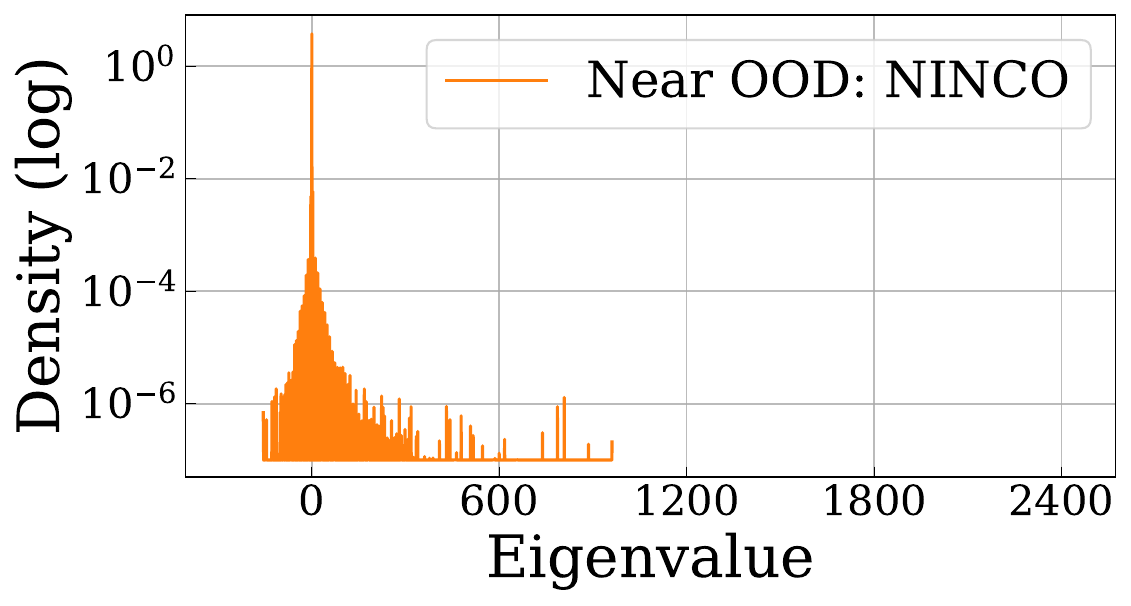}
        \subcaption{NINCO}
    \end{subfigure}
    
    \begin{subfigure}{0.245\linewidth}
        \centering
        \includegraphics[width=\linewidth]{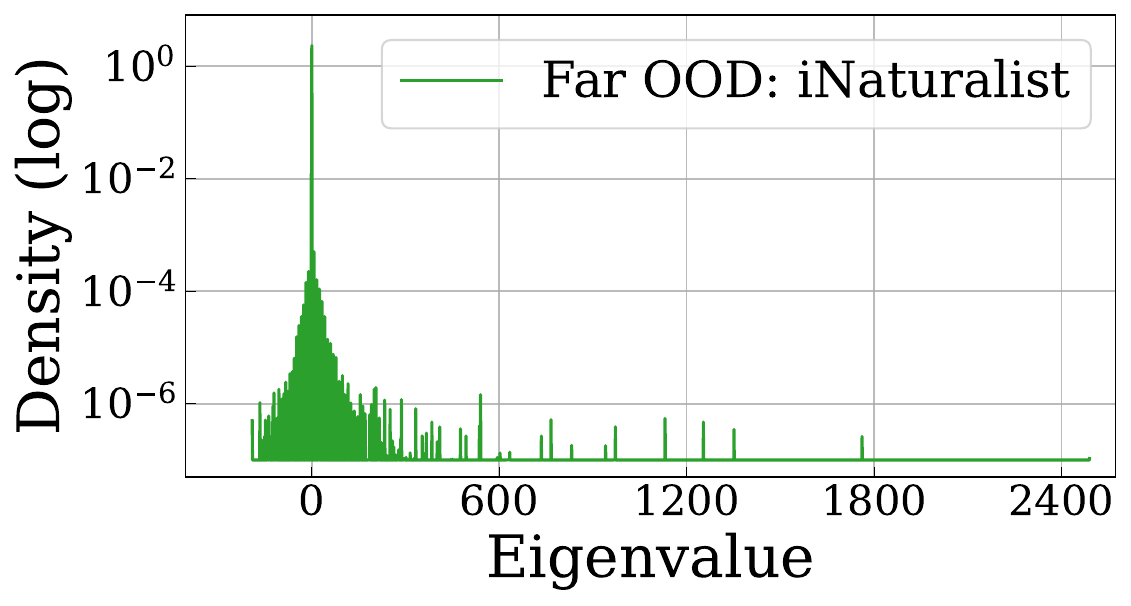}
        \subcaption{iNaturalist}
    \end{subfigure}
    \begin{subfigure}{0.245\linewidth}
        \centering
        \includegraphics[width=\linewidth]{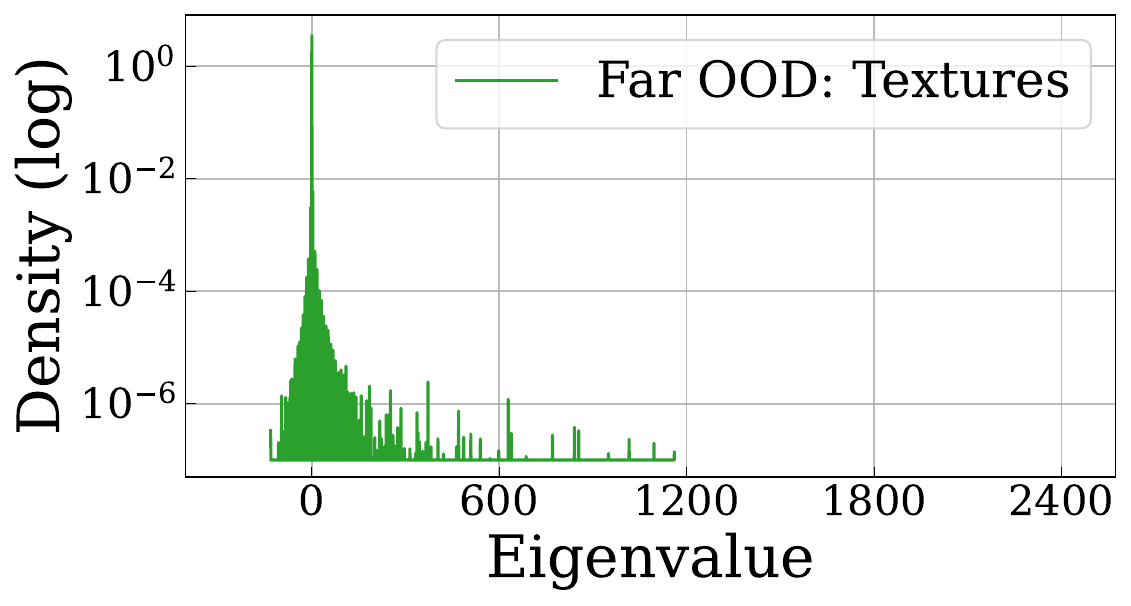}
        \subcaption{Textures}
    \end{subfigure}
    \begin{subfigure}{0.245\linewidth}
        \centering
        \includegraphics[width=\linewidth]{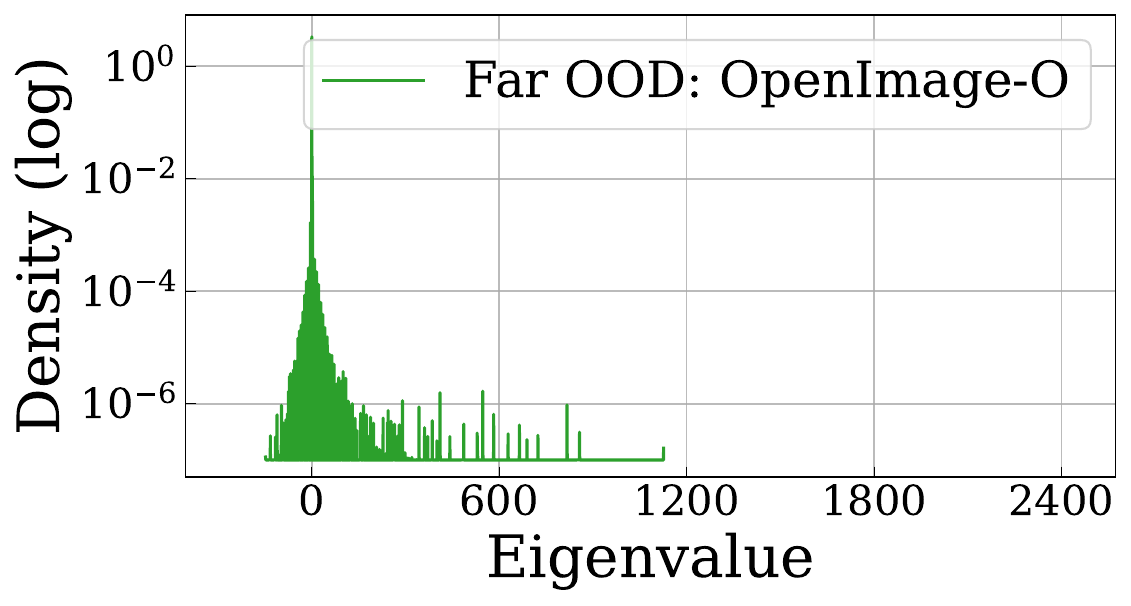}
        \subcaption{OpenImage-O}
    \end{subfigure}
    \caption{
        \textit{Parameter-space Hessian} ESD of ImageNet-200–trained models across multiple OOD datasets.
        Colors indicate ID (\textcolor{dblue}{blue}), near-OOD (\textcolor{dorange}{orange}), and far-OOD (\textcolor{dgreen}{green}).
        }
    \label{fig:Hessian ESD for ImageNet-200 across datasets}
\end{figure*}

\vspace{10em}

\begin{figure*}[t!]
    \centering
    \begin{subfigure}{0.245\linewidth}
        \centering
        \includegraphics[width=\linewidth]{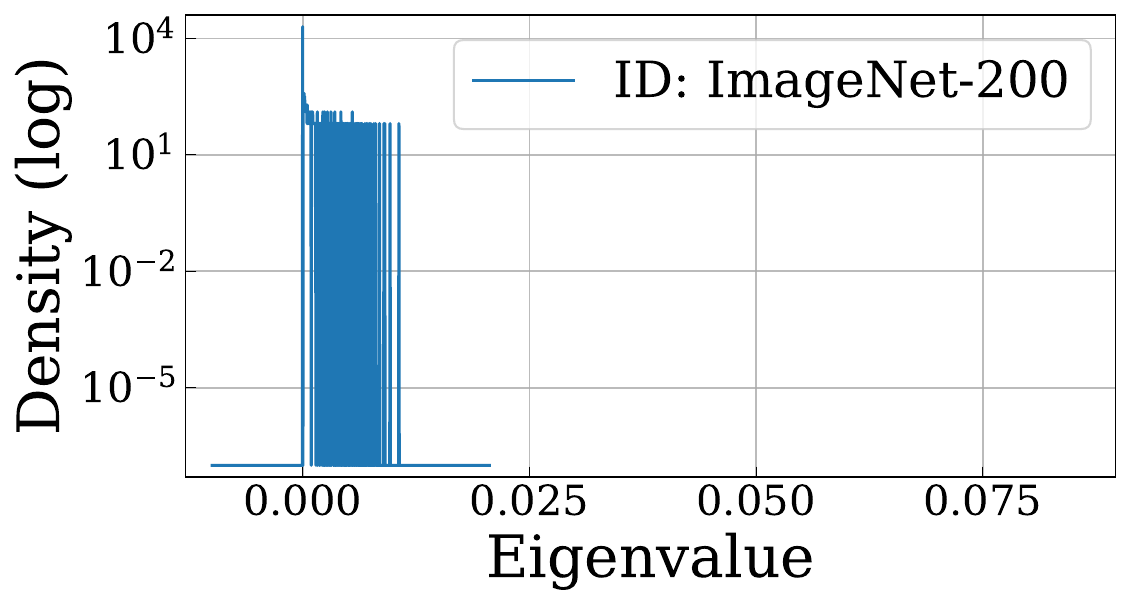}
        \subcaption{ImageNet-200}
    \end{subfigure}
    \begin{subfigure}{0.245\linewidth}
        \centering
        \includegraphics[width=\linewidth]{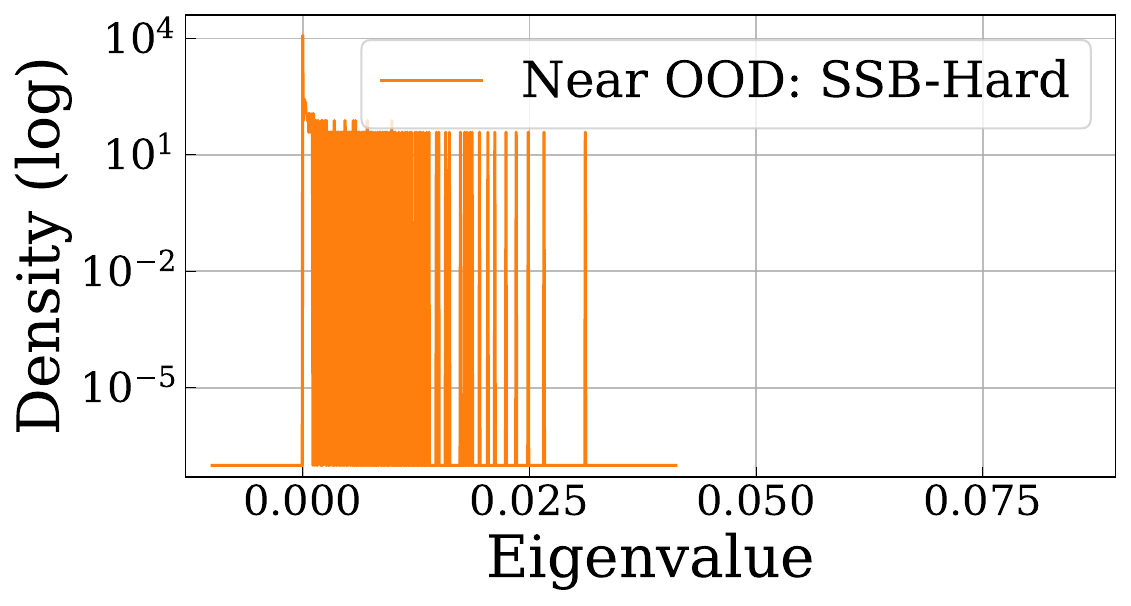}
        \subcaption{SSB-Hard}
    \end{subfigure}
    \begin{subfigure}{0.245\linewidth}
        \centering
        \includegraphics[width=\linewidth]{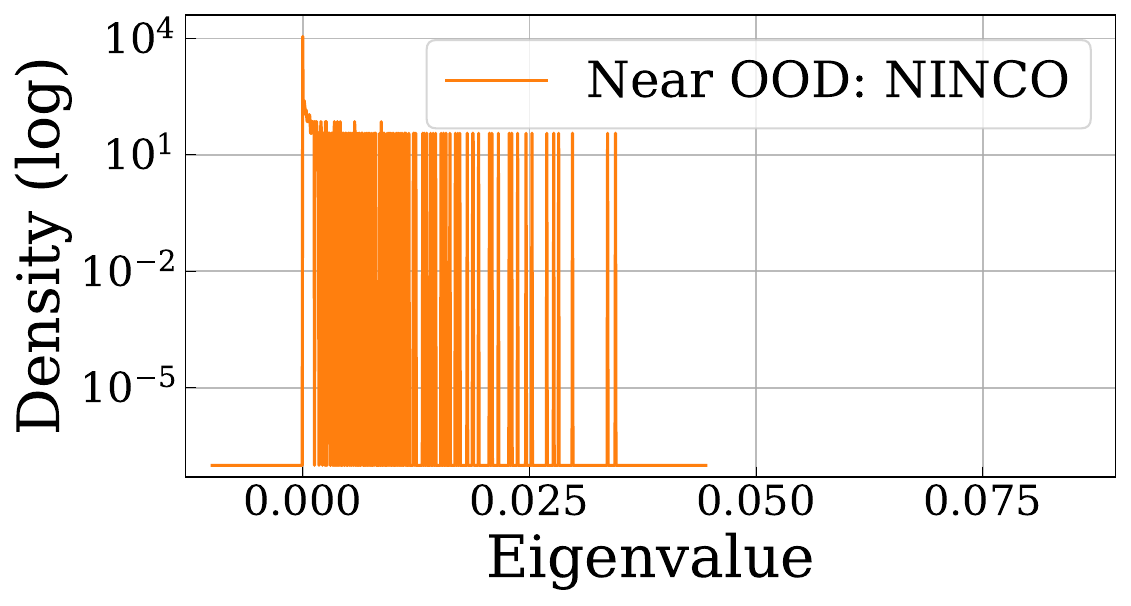}
        \subcaption{NINCO}
    \end{subfigure}
    
    \begin{subfigure}{0.245\linewidth}
        \centering
        \includegraphics[width=\linewidth]{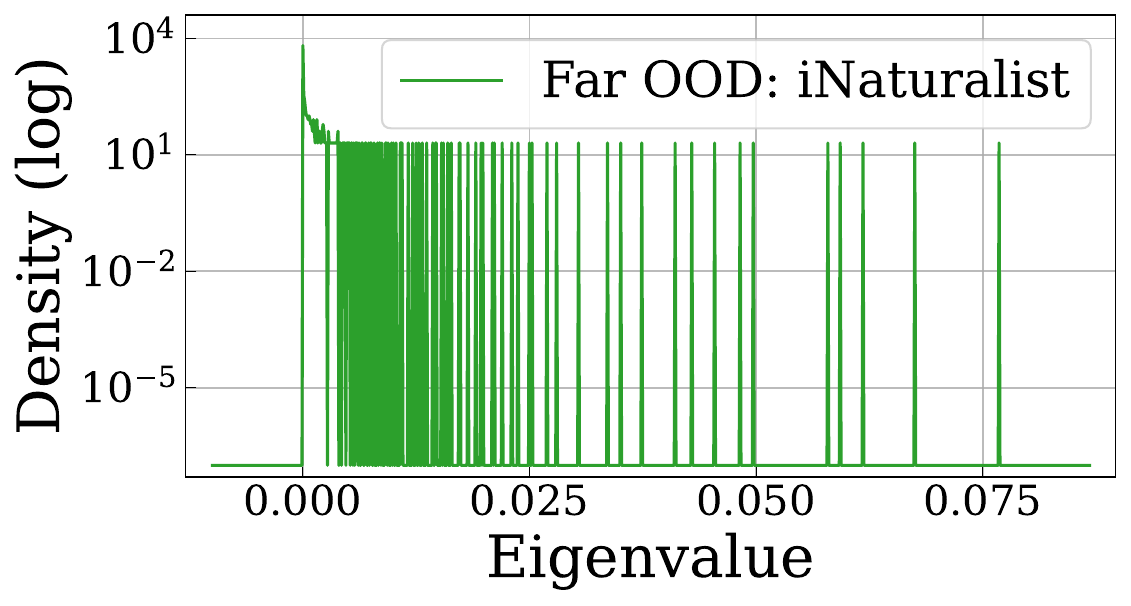}
        \subcaption{iNaturalist}
    \end{subfigure}
    \begin{subfigure}{0.245\linewidth}
        \centering
        \includegraphics[width=\linewidth]{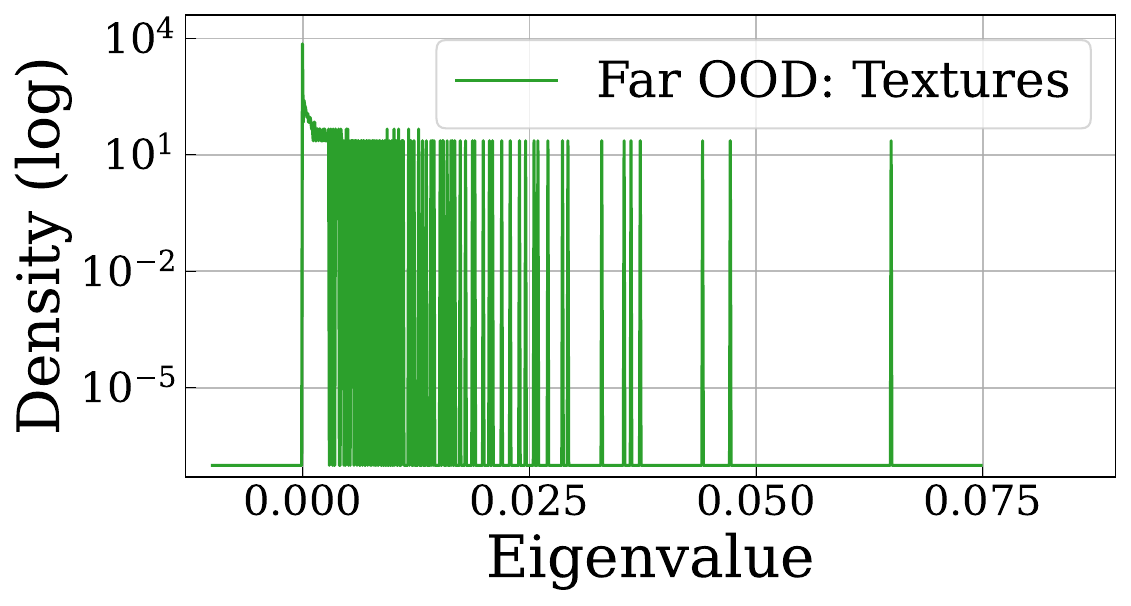}
        \subcaption{Textures}
    \end{subfigure}
    \begin{subfigure}{0.245\linewidth}
        \centering
        \includegraphics[width=\linewidth]{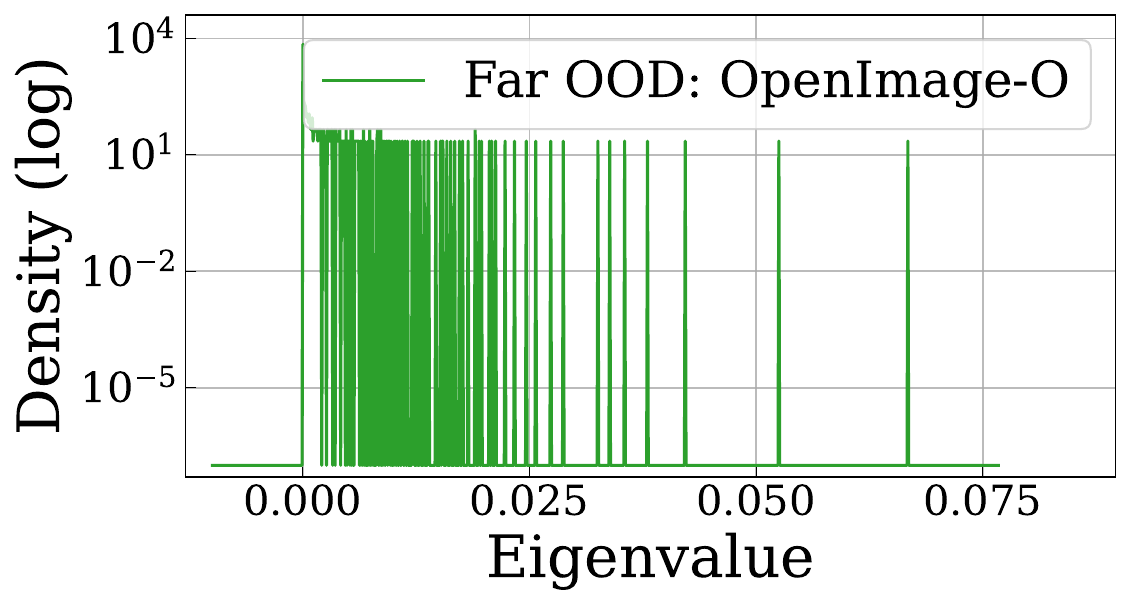}
        \subcaption{OpenImage-O}
    \end{subfigure}
    \caption{
        \textit{Feature Hessian} ESD of ImageNet-200–trained models across multiple OOD datasets.
        Colors indicate ID (\textcolor{dblue}{blue}), near-OOD (\textcolor{dorange}{orange}), and far-OOD (\textcolor{dgreen}{green}).
        }
    \label{fig:Feature Hessian ESD for ImageNet-200 across datasets}
\end{figure*}
\vspace{10em}

\subsection{Curvature Analysis}
Building upon the sample-level curvature analysis presented in the main text, we provide additional visualizations of the local loss landscape across multiple datasets.
Specifically, we plot the distributions of sample-wise Hessian trace values introduced in~\Cref{sec:observations}.
The results are shown in \Cref{fig:Sample-wise Hessian trace for CIFAR-10}, \ref{fig:Sample-wise Hessian trace for CIFAR-100}, \ref{fig:Sample-wise Hessian trace for ImageNet-200}, and \ref{fig:Sample-wise Hessian trace for ImageNet-1K}, corresponding to the CIFAR-10, CIFAR-100, ImageNet-200, and ImageNet-1K benchmarks.
Across these datasets, the kernel density estimates of the sample-wise trace distributions exhibit consistent patterns.
Although minor dataset-specific variations are observed, the overall geometric trend remains consistent across benchmarks, with Hessian trace values increasing as the data distribution shifts from ID to OOD.

\vspace{10em}


\begin{figure*}[h!]
    \centering
    \begin{subfigure}{0.245\linewidth}
        \centering
        \includegraphics[width=\linewidth]{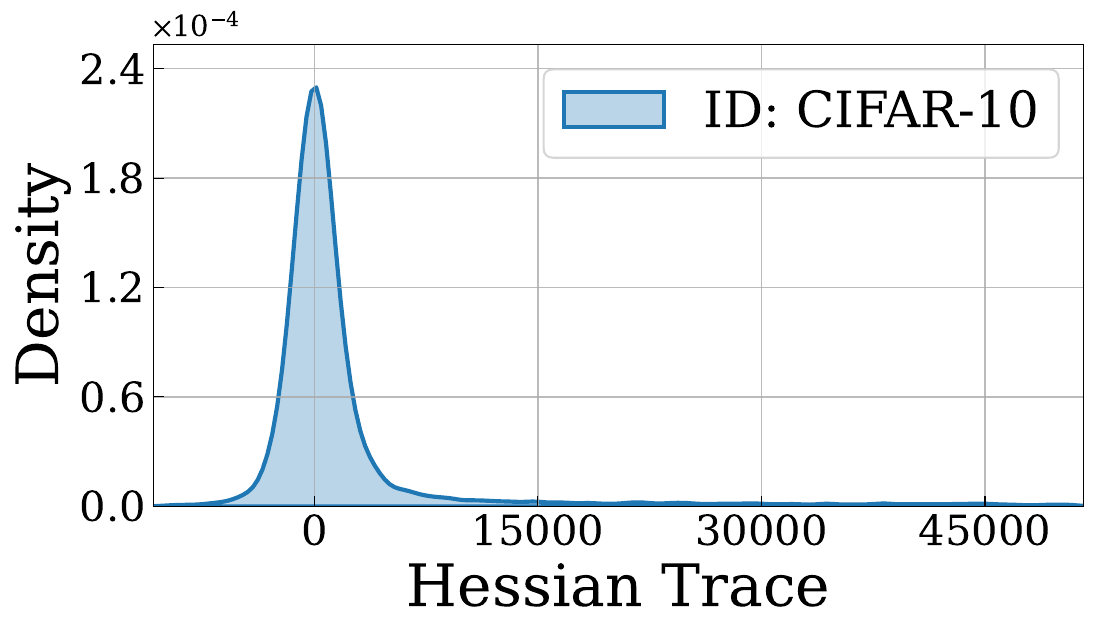}
        \subcaption{CIFAR-10}
    \end{subfigure}
    \begin{subfigure}{0.245\linewidth}
        \centering
        \includegraphics[width=\linewidth]{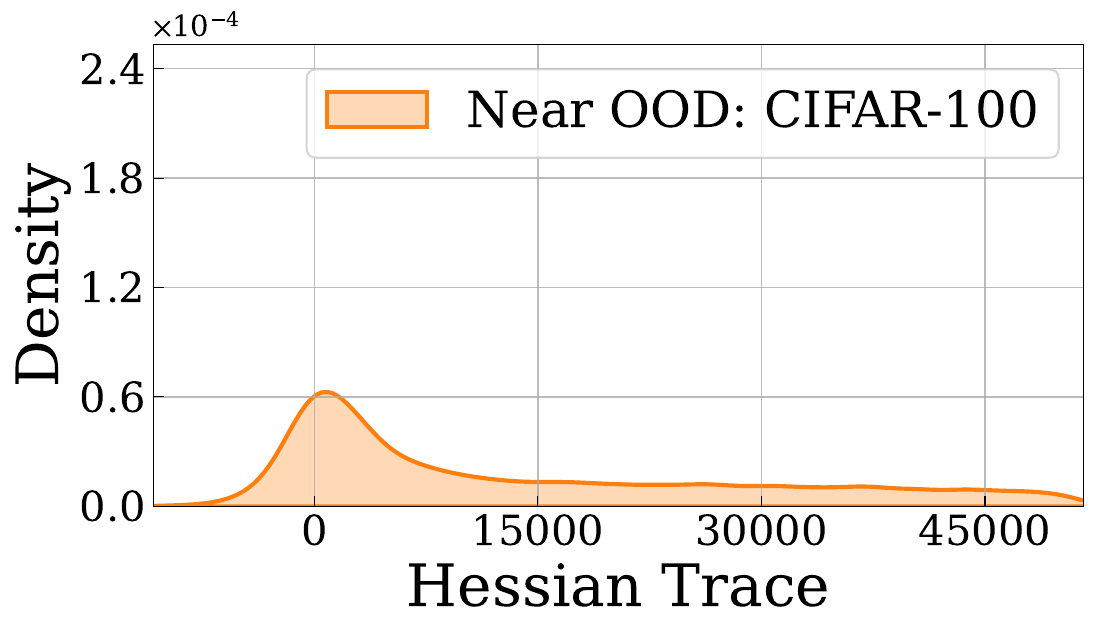}
        \subcaption{CIFAR-100}
    \end{subfigure}
    \begin{subfigure}{0.245\linewidth}
        \centering
        \includegraphics[width=\linewidth]{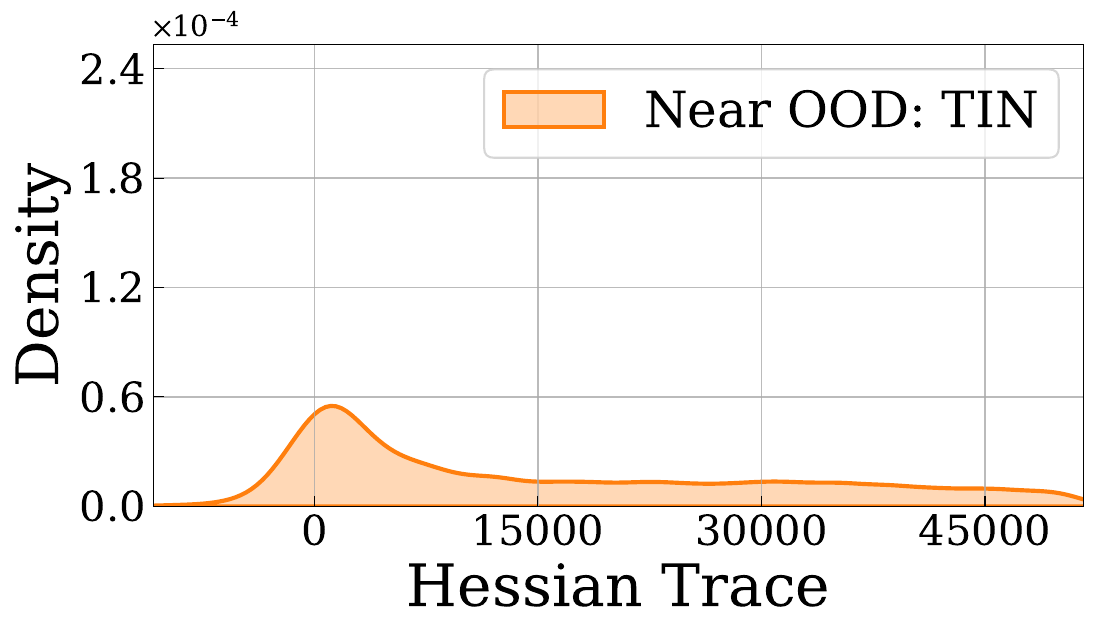}
        \subcaption{TIN}
    \end{subfigure}
    
    \begin{subfigure}{0.245\linewidth}
        \centering
        \includegraphics[width=\linewidth]{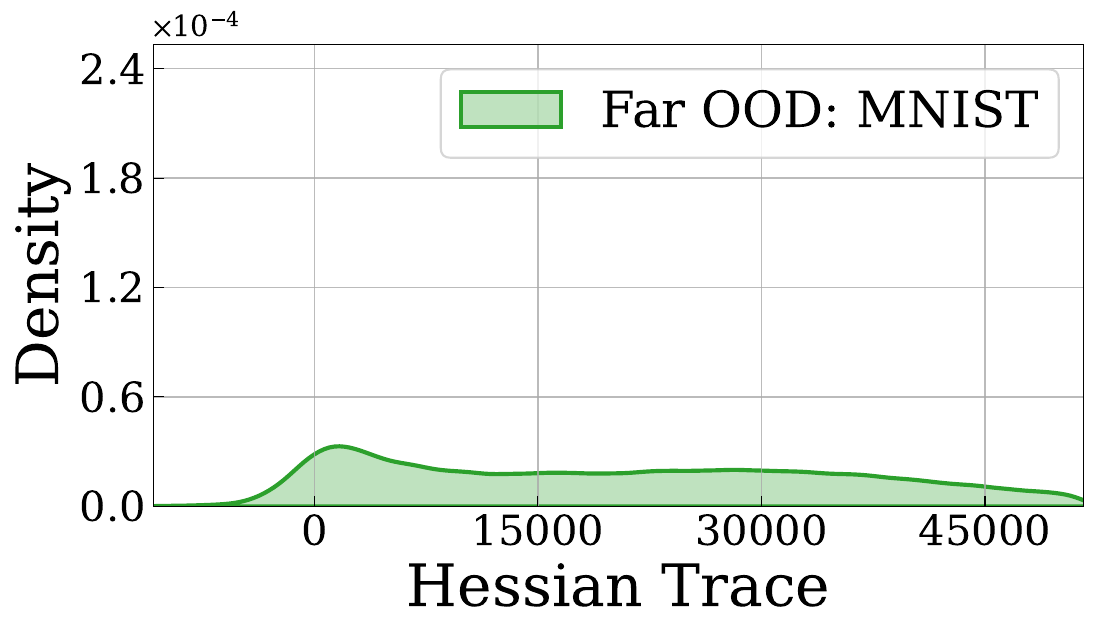}
        \subcaption{MNIST}
    \end{subfigure}
    \begin{subfigure}{0.245\linewidth}
        \centering
        \includegraphics[width=\linewidth]{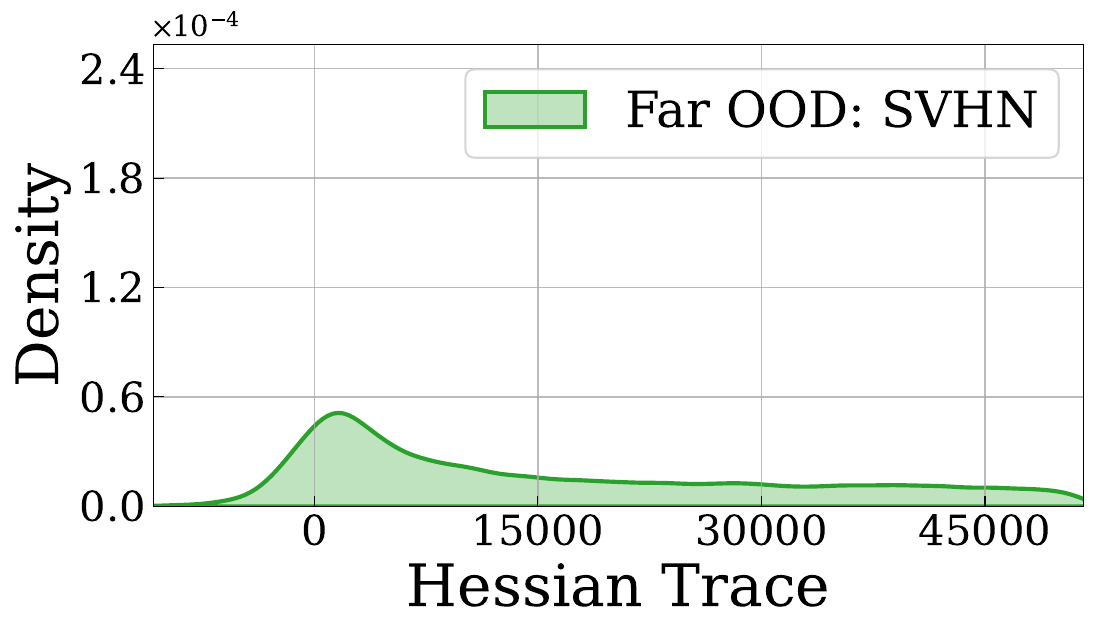}
        \subcaption{SVHN}
    \end{subfigure}
    \begin{subfigure}{0.245\linewidth}
        \centering
        \includegraphics[width=\linewidth]{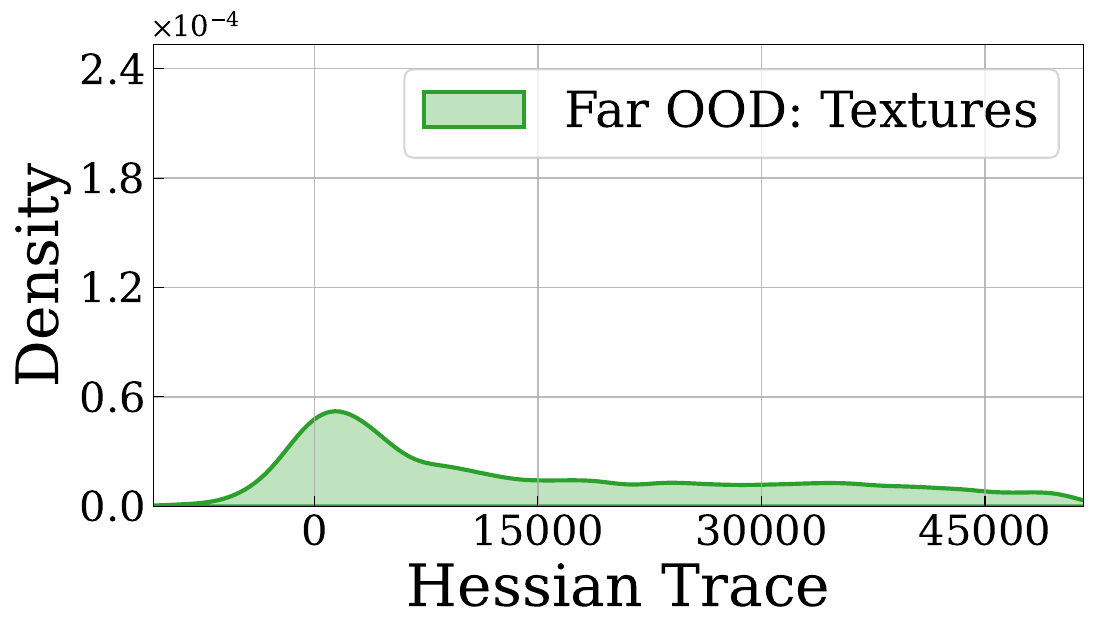}
        \subcaption{Textures}
    \end{subfigure}
    \begin{subfigure}{0.245\linewidth}
        \centering
        \includegraphics[width=\linewidth]{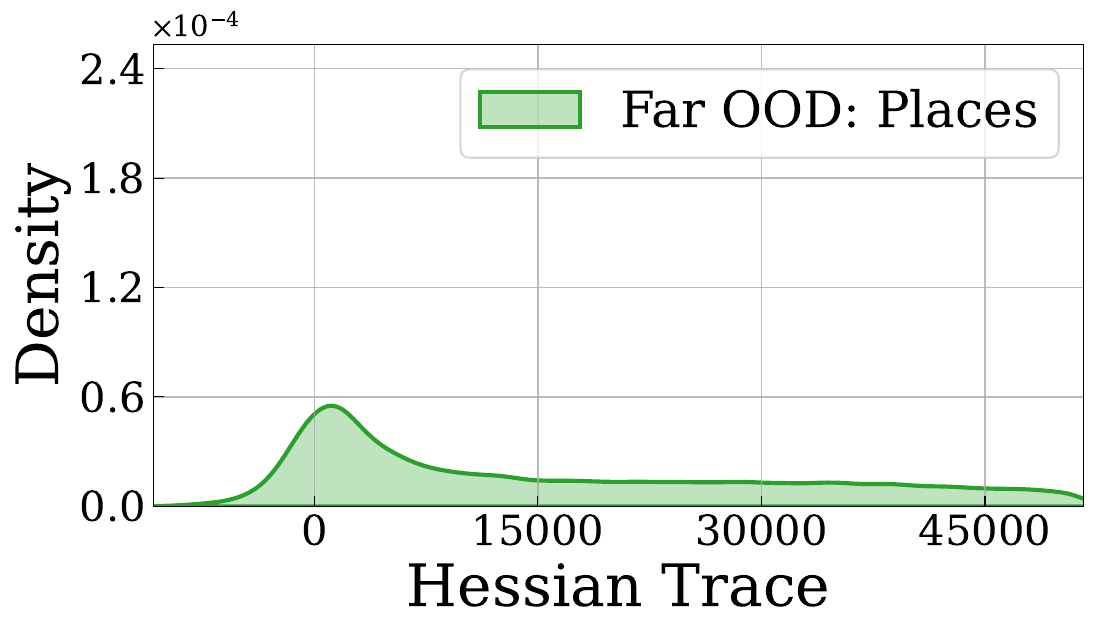}
        \subcaption{Places}
    \end{subfigure}
    \vspace{-1.5em}
    \caption{
        Per-sample Hessian trace distribution for a CIFAR-10–trained model evaluated on multiple OOD datasets. 
        Colors indicate ID (\textcolor{dblue}{blue}), near-OOD (\textcolor{dorange}{orange}), and far-OOD (\textcolor{dgreen}{green}).
        }
    \label{fig:Sample-wise Hessian trace for CIFAR-10}
\end{figure*}

\begin{figure*}[h!]
    \centering
    \begin{subfigure}{0.245\linewidth}
        \centering
        \includegraphics[width=\linewidth]{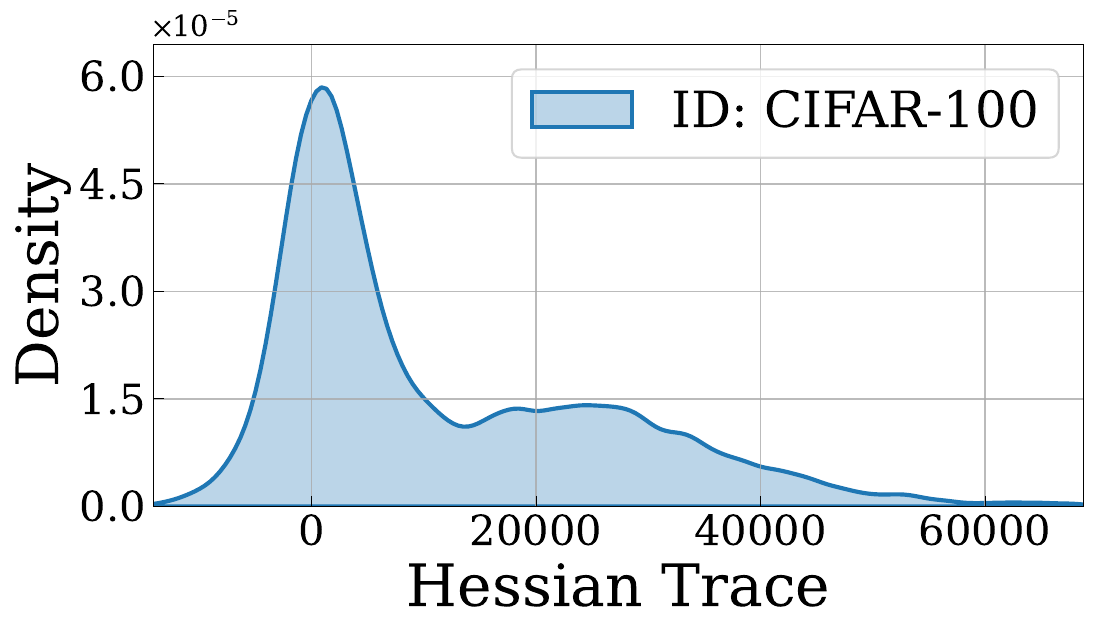}
        \subcaption{CIFAR-100}
    \end{subfigure}
    \begin{subfigure}{0.245\linewidth}
        \centering
        \includegraphics[width=\linewidth]{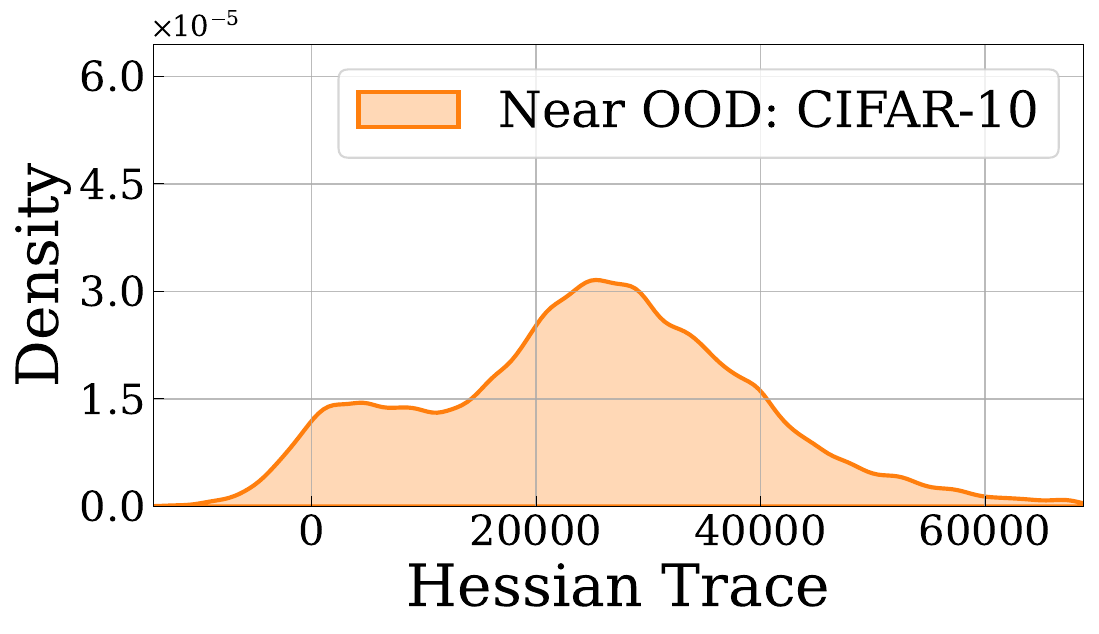}
        \subcaption{CIFAR-10}
    \end{subfigure}
    \begin{subfigure}{0.245\linewidth}
        \centering
        \includegraphics[width=\linewidth]{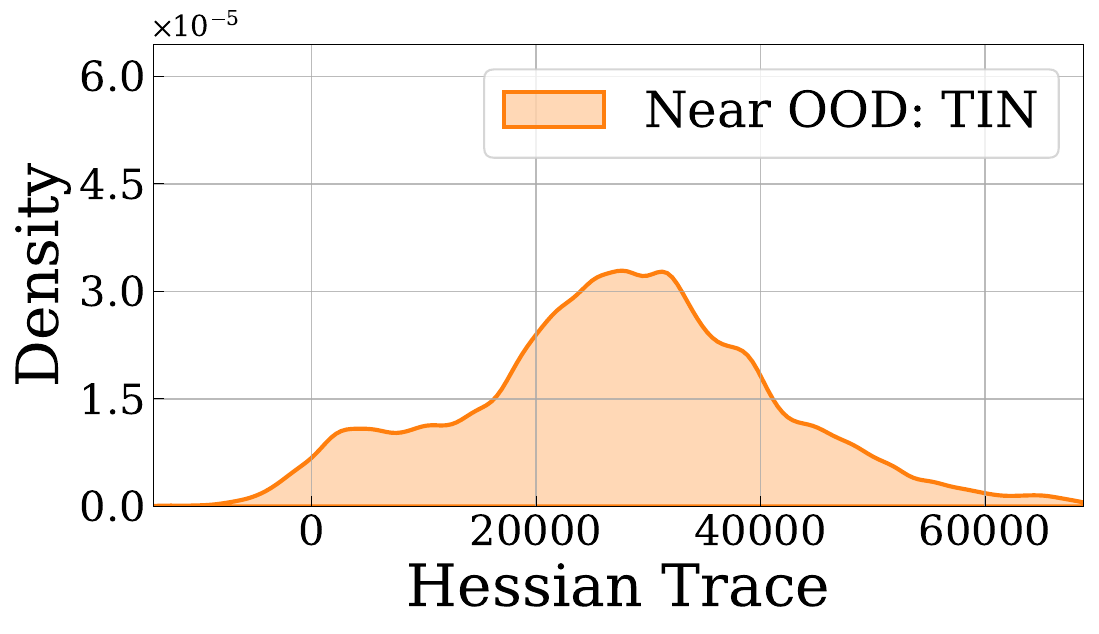}
        \subcaption{TIN}
    \end{subfigure}
    
    \begin{subfigure}{0.245\linewidth}
        \centering
        \includegraphics[width=\linewidth]{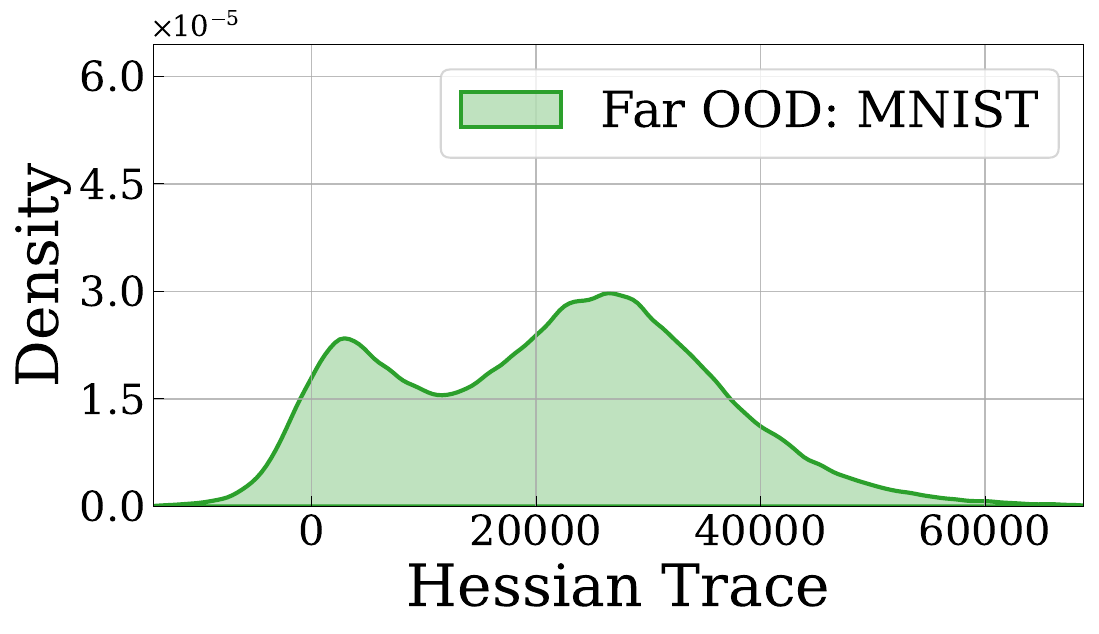}
        \subcaption{MNIST}
    \end{subfigure}
    \begin{subfigure}{0.245\linewidth}
        \centering
        \includegraphics[width=\linewidth]{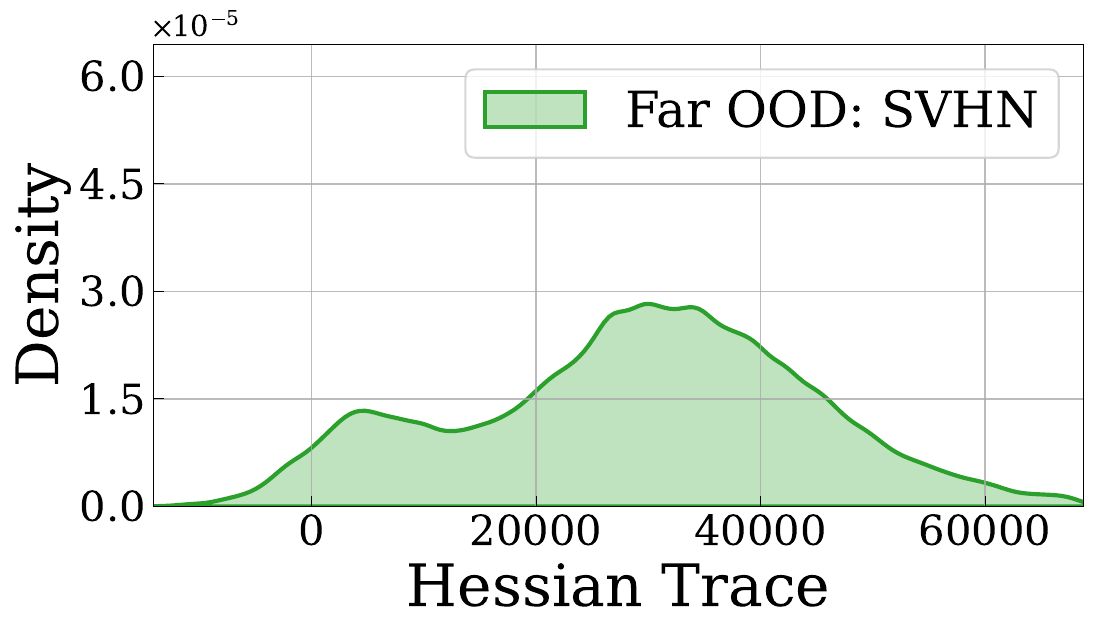}
        \subcaption{SVHN}
    \end{subfigure}
    \begin{subfigure}{0.245\linewidth}
        \centering
        \includegraphics[width=\linewidth]{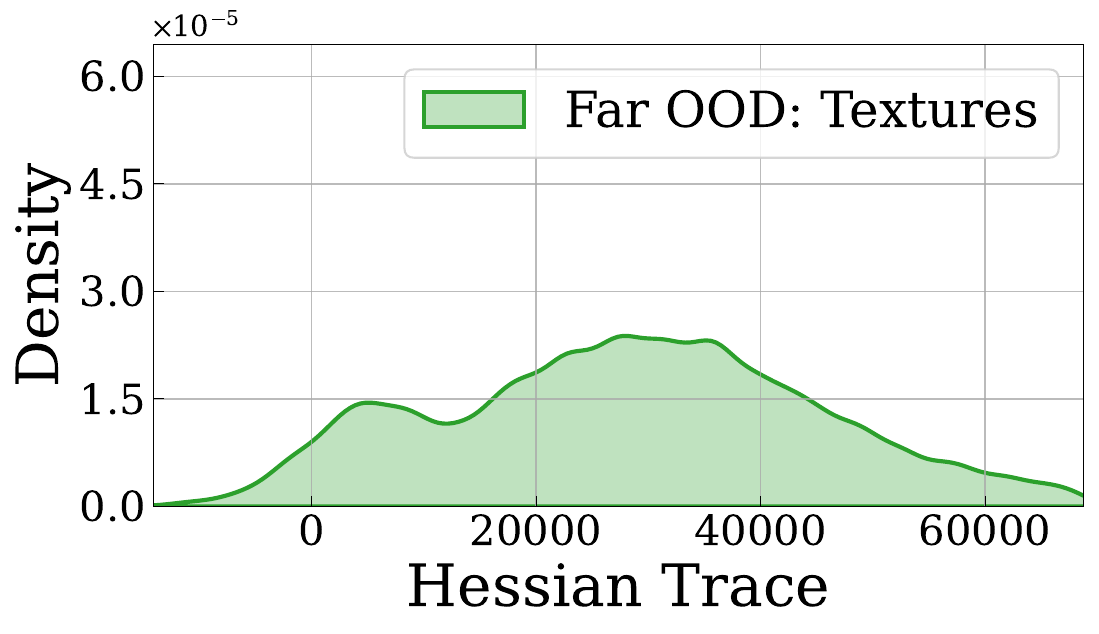}
        \subcaption{Textures}
    \end{subfigure}
    \begin{subfigure}{0.245\linewidth}
        \centering
        \includegraphics[width=\linewidth]{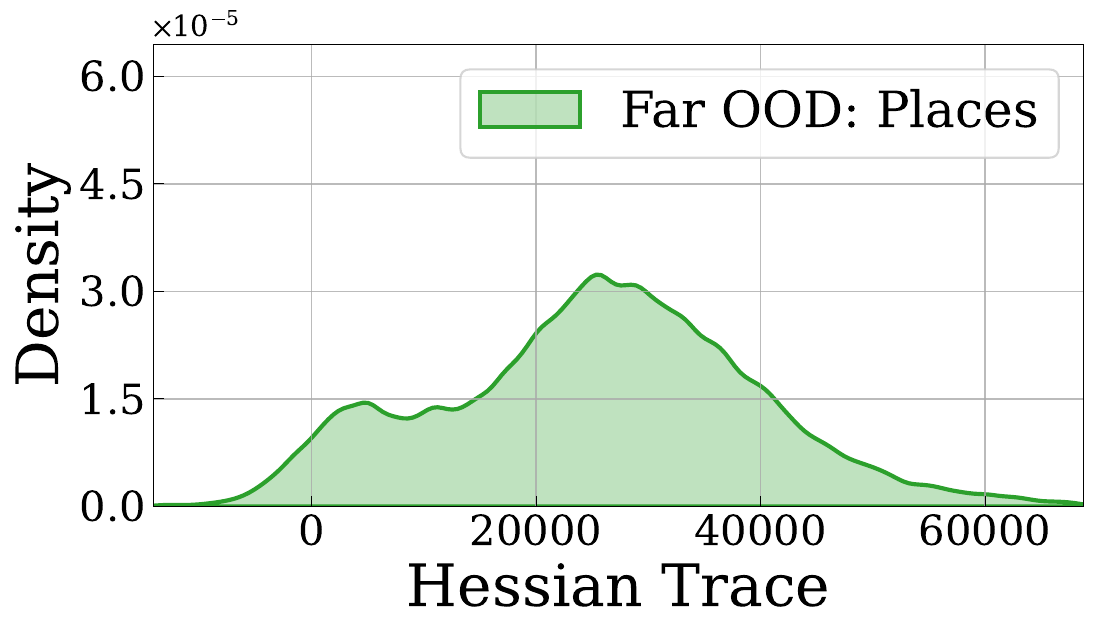}
        \subcaption{Places}
    \end{subfigure}
    \caption{
        Per-sample Hessian trace distribution for a CIFAR-100–trained model evaluated on multiple OOD datasets. 
        Colors indicate ID (\textcolor{dblue}{blue}), near-OOD (\textcolor{dorange}{orange}), and far-OOD (\textcolor{dgreen}{green}).
        }
    \label{fig:Sample-wise Hessian trace for CIFAR-100}
\end{figure*}

\begin{figure*}[h!]
    \centering
    \begin{subfigure}{0.245\linewidth}
        \centering
        \includegraphics[width=\linewidth]{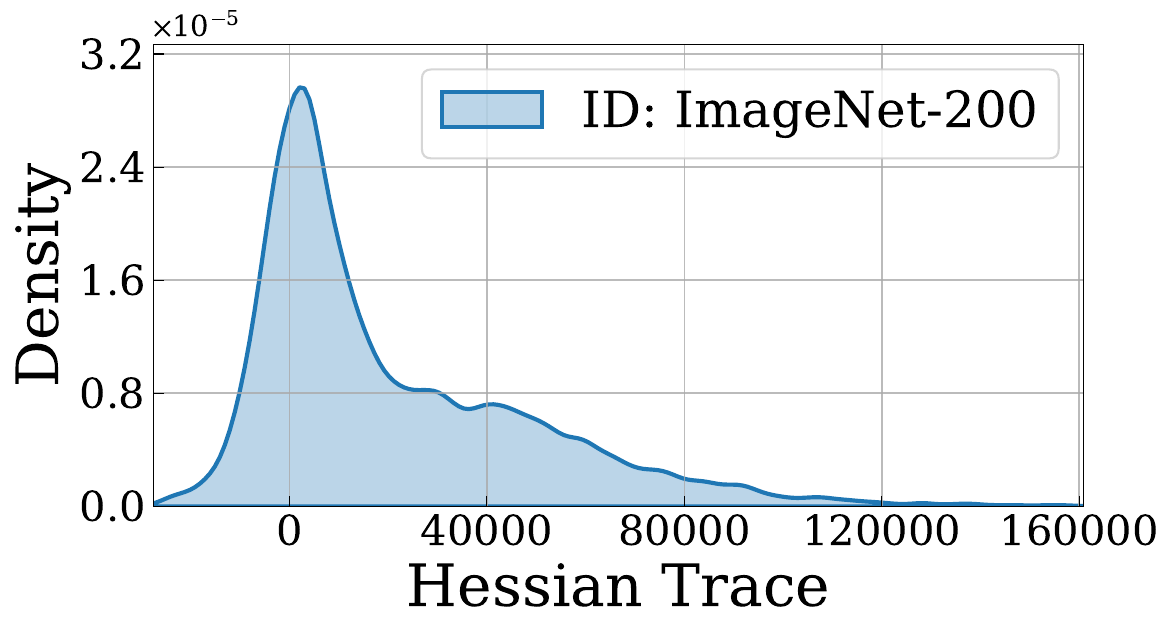}
        \subcaption{ImageNet-200}
    \end{subfigure}
    \begin{subfigure}{0.245\linewidth}
        \centering
        \includegraphics[width=\linewidth]{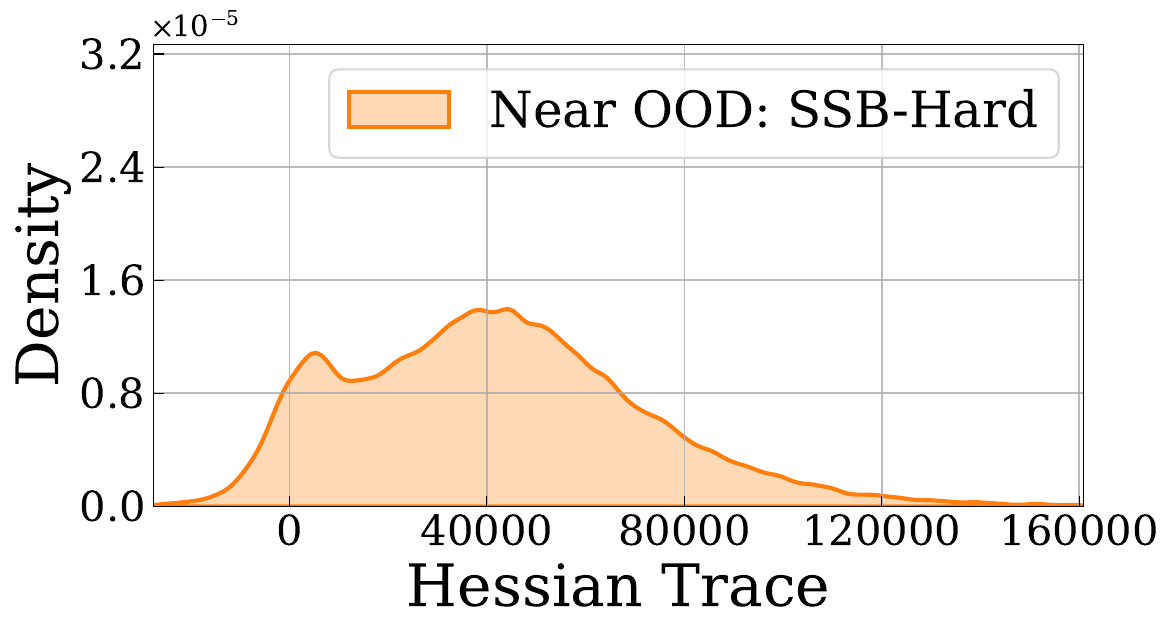}
        \subcaption{SSB-Hard}
    \end{subfigure}
    \begin{subfigure}{0.245\linewidth}
        \centering
        \includegraphics[width=\linewidth]{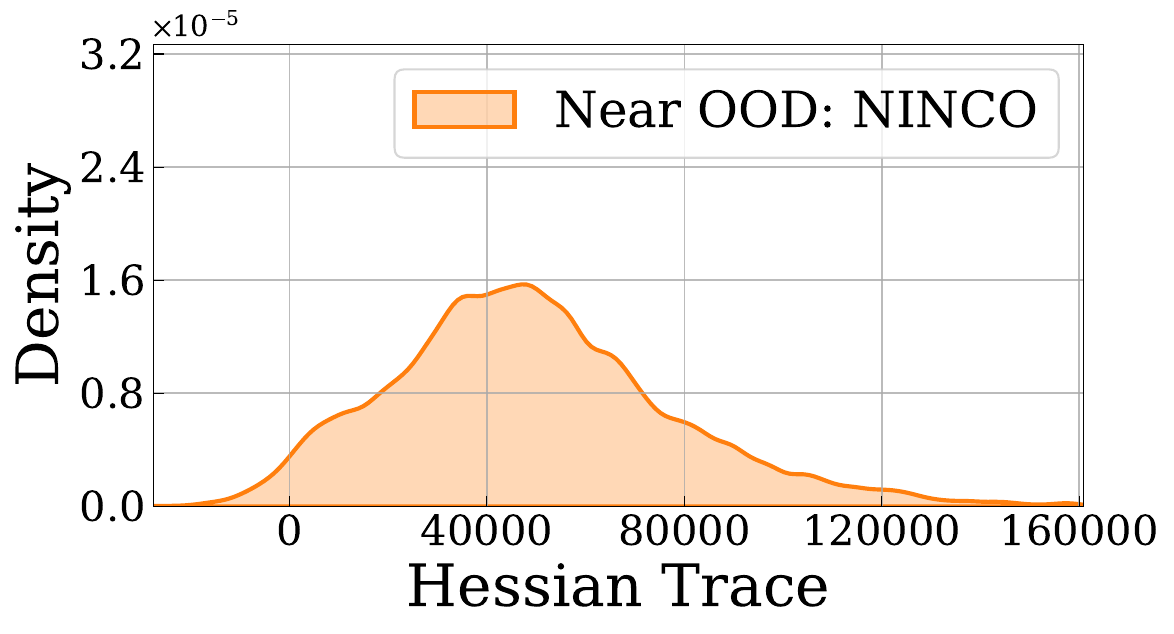}
        \subcaption{NINCO}
    \end{subfigure}
    
    \begin{subfigure}{0.245\linewidth}
        \centering
        \includegraphics[width=\linewidth]{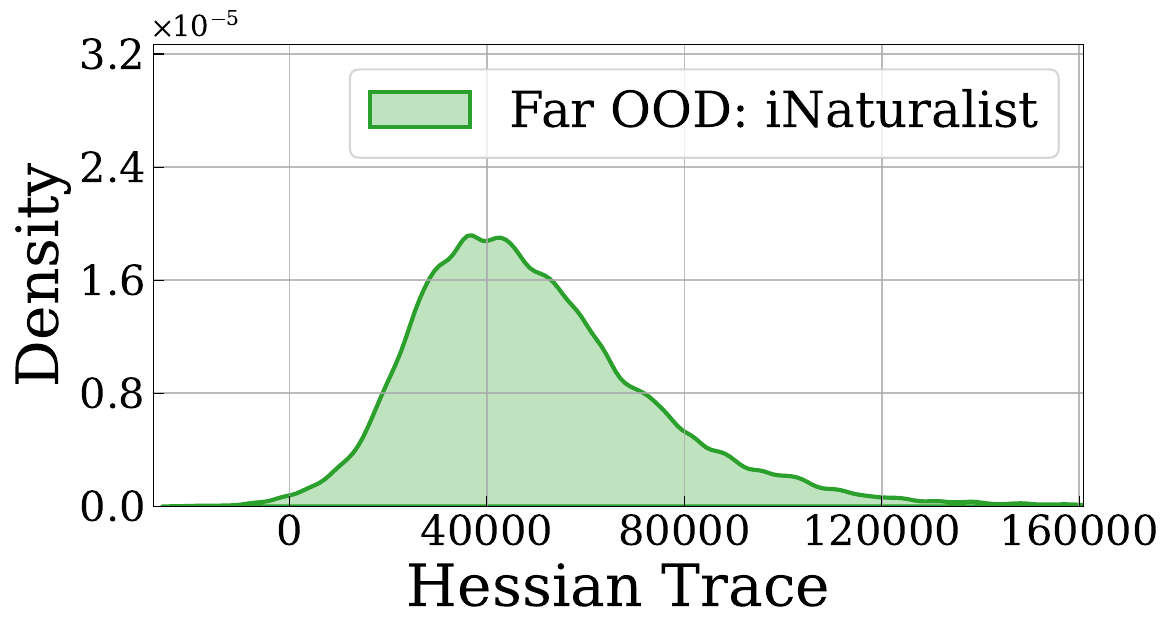}
        \subcaption{iNaturalist}
    \end{subfigure}
    \begin{subfigure}{0.245\linewidth}
        \centering
        \includegraphics[width=\linewidth]{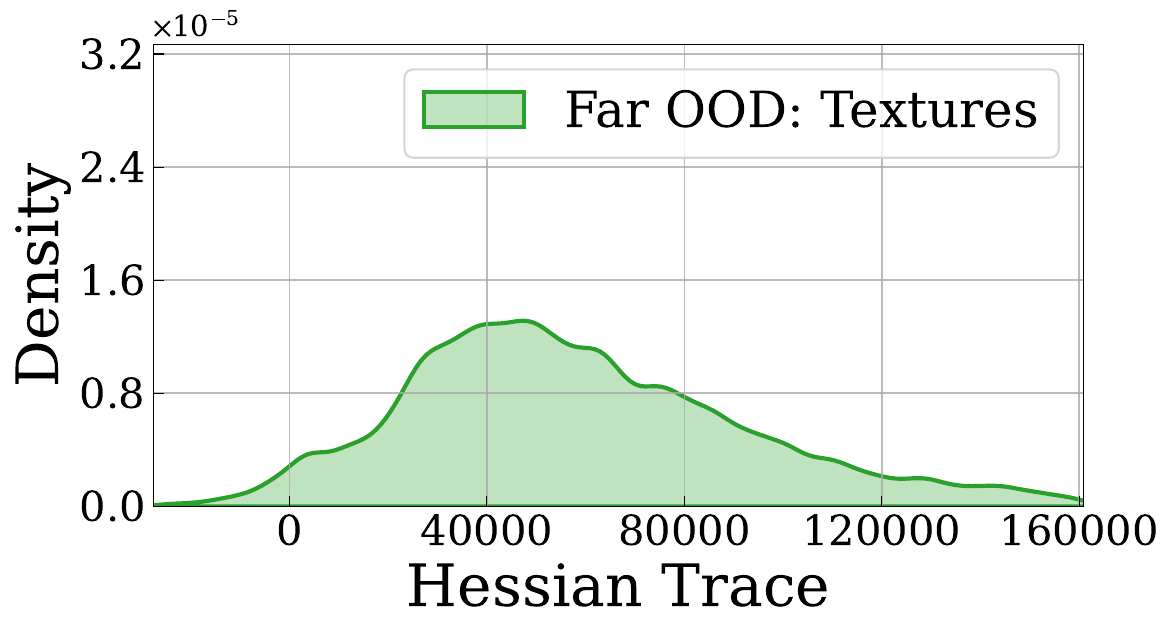}
        \subcaption{Textures}
    \end{subfigure}
    \begin{subfigure}{0.245\linewidth}
        \centering
        \includegraphics[width=\linewidth]{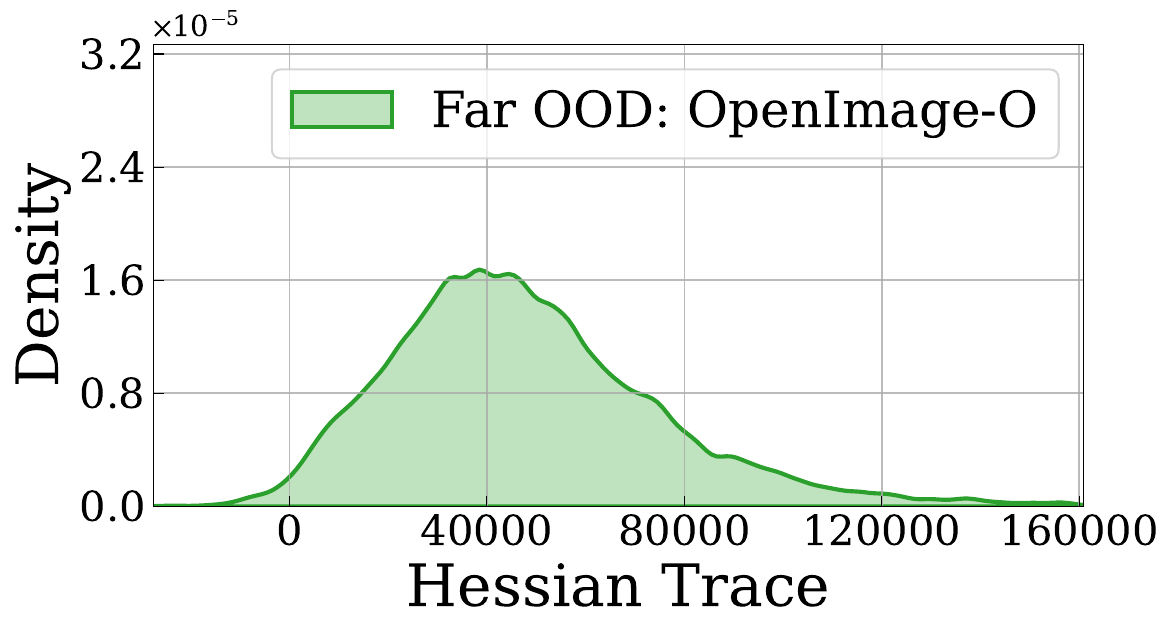}
        \subcaption{OpenImage-O}
    \end{subfigure}
    \caption{
        Per-sample Hessian trace distribution for a ImageNet-200–trained model evaluated on multiple OOD datasets. 
        Colors indicate ID (\textcolor{dblue}{blue}), near-OOD (\textcolor{dorange}{orange}), and far-OOD (\textcolor{dgreen}{green}).
        }
    \label{fig:Sample-wise Hessian trace for ImageNet-200}
\end{figure*}

\clearpage

\begin{figure*}[t]
    \centering
    \begin{subfigure}{0.245\linewidth}
        \centering
        \includegraphics[width=\linewidth]{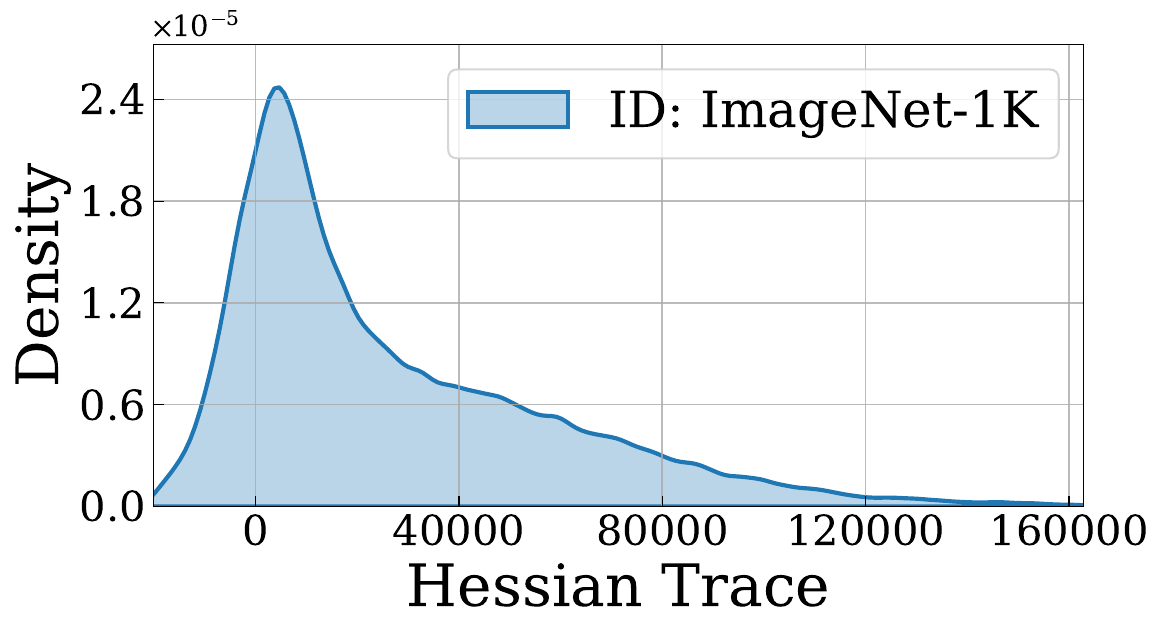}
        \subcaption{ImageNet-1K}
    \end{subfigure}
    \begin{subfigure}{0.245\linewidth}
        \centering
        \includegraphics[width=\linewidth]{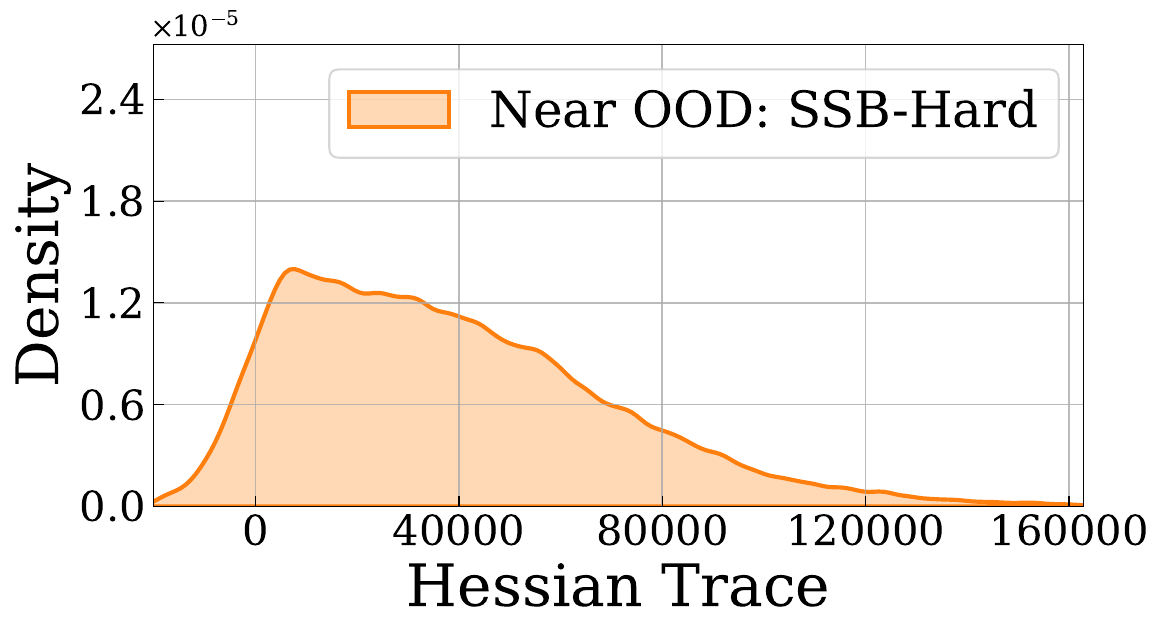}
        \subcaption{SSB-Hard}
    \end{subfigure}
    \begin{subfigure}{0.245\linewidth}
        \centering
        \includegraphics[width=\linewidth]{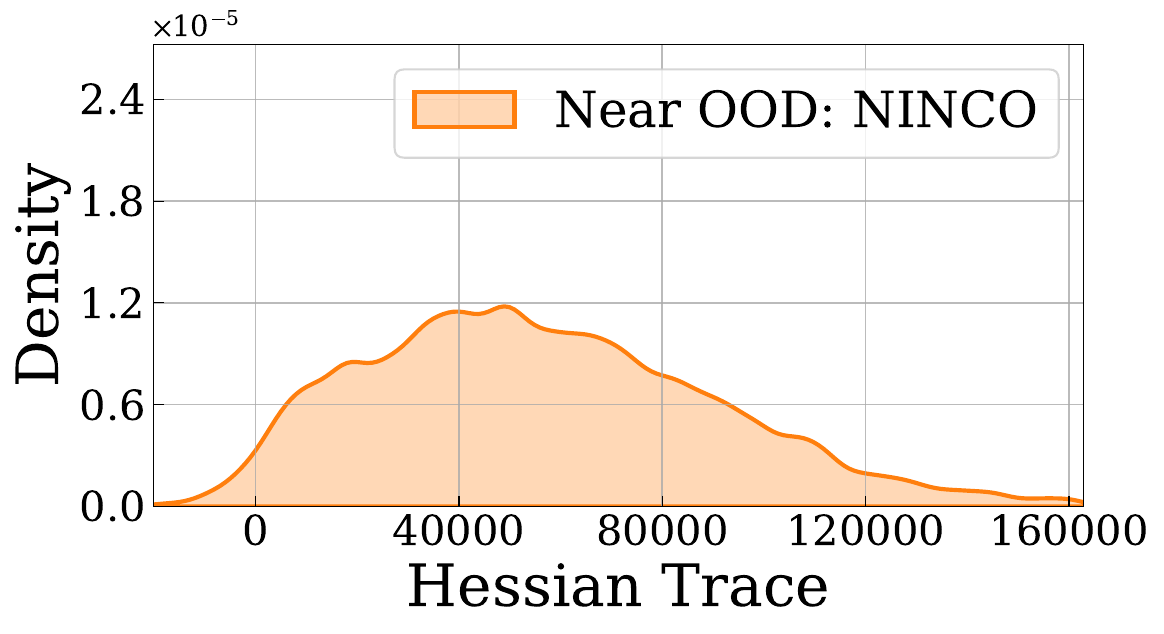}
        \subcaption{NINCO}
    \end{subfigure}
    
    \begin{subfigure}{0.245\linewidth}
        \centering
        \includegraphics[width=\linewidth]{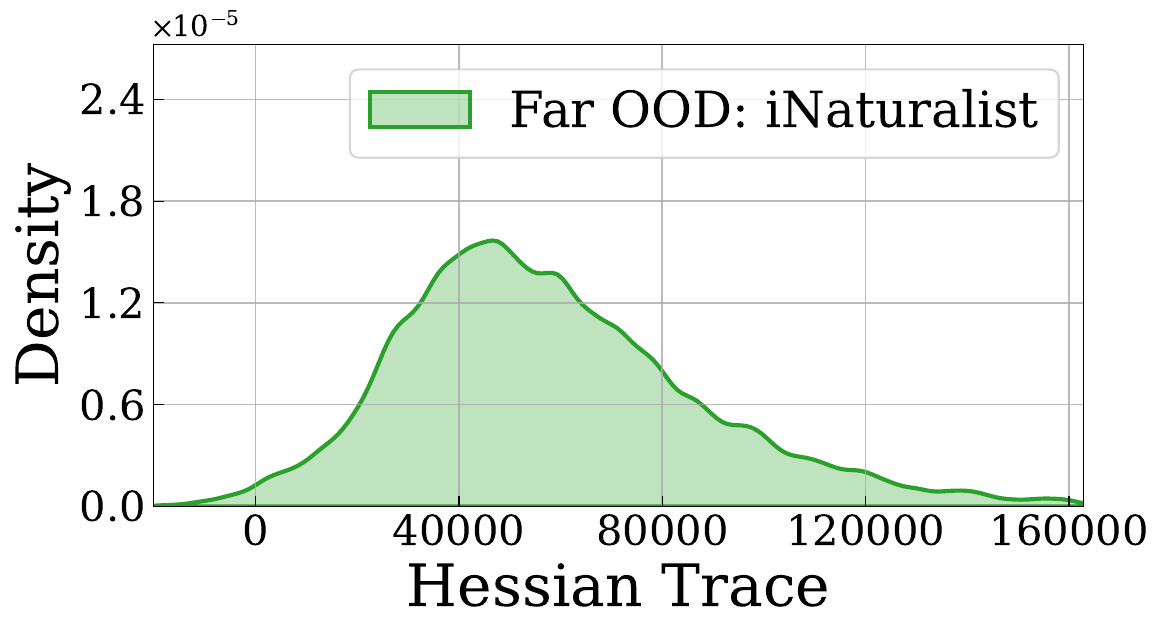}
        \subcaption{iNaturalist}
    \end{subfigure}
    \begin{subfigure}{0.245\linewidth}
        \centering
        \includegraphics[width=\linewidth]{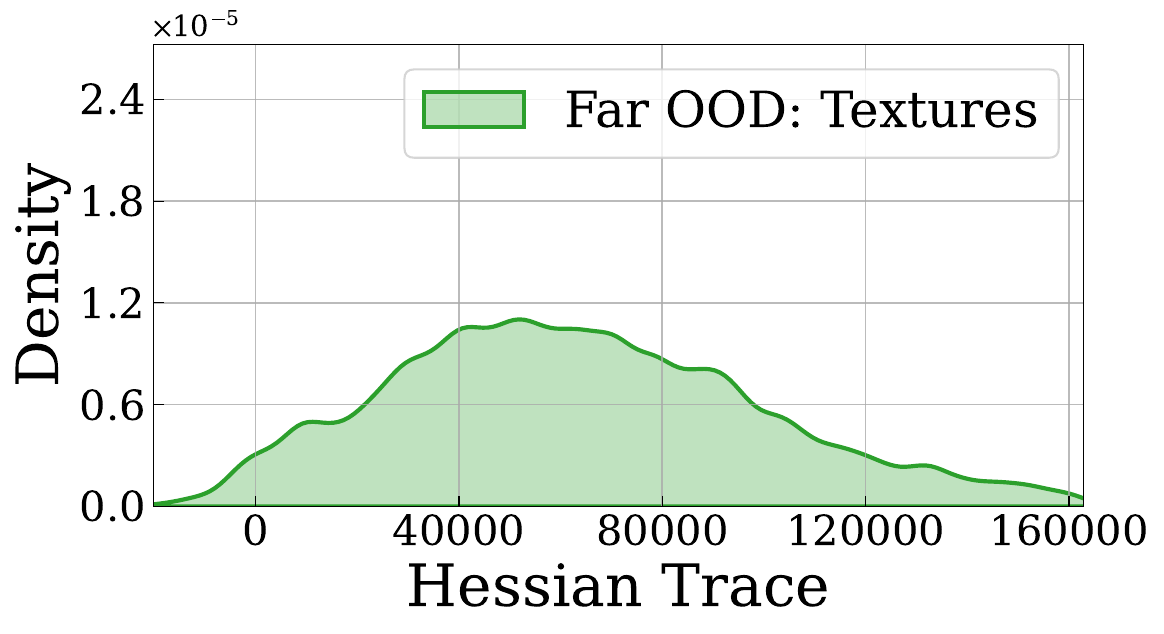}
        \subcaption{Textures}
    \end{subfigure}
    \begin{subfigure}{0.245\linewidth}
        \centering
        \includegraphics[width=\linewidth]{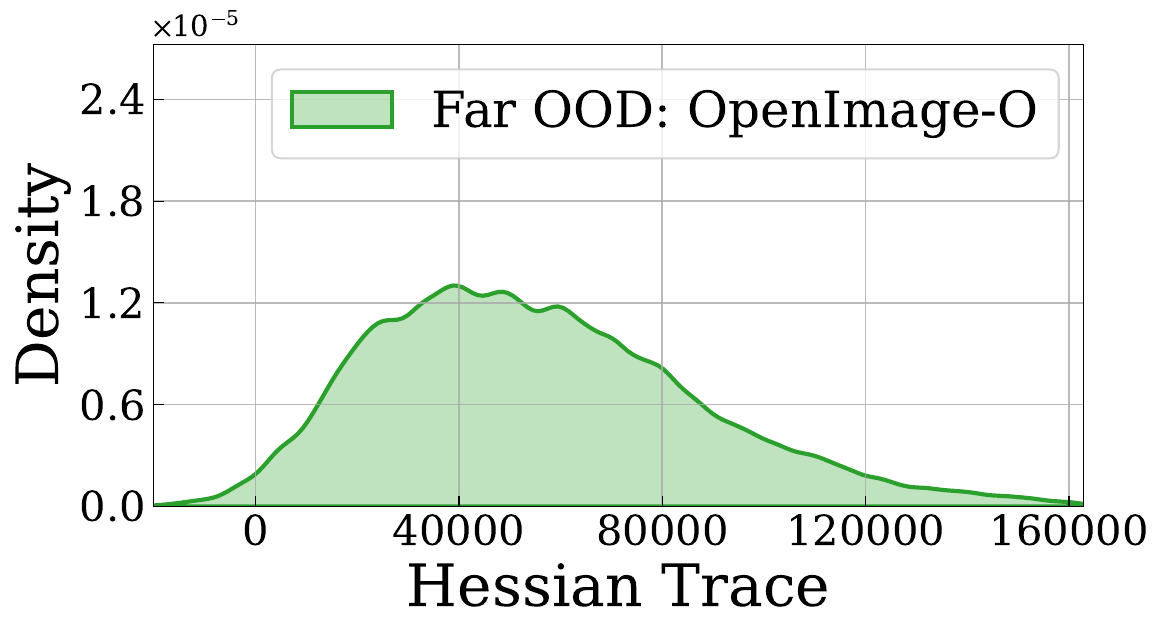}
        \subcaption{OpenImage-O}
    \end{subfigure}
    \caption{
        Per-sample Hessian trace distribution for a ImageNet-1K–trained model evaluated on multiple OOD datasets. 
        Colors indicate ID (\textcolor{dblue}{blue}), near-OOD (\textcolor{dorange}{orange}), and far-OOD (\textcolor{dgreen}{green}).
        }
    \label{fig:Sample-wise Hessian trace for ImageNet-1K}
\end{figure*}


\subsection{OOD Benchmark Results}
To provide a comprehensive evaluation of detection performance, we report AUROC and FPR95 across the OOD benchmarks introduced in~\Cref{sec:evaluations}.
Detailed per-dataset results are presented in \Cref{tab: comprehensive standard ood benchmark for cifar-10}, \ref{tab: comprehensive standard ood benchmark for cifar-100}, \ref{tab: comprehensive standard ood benchmark for imagenet-200}, and \ref{tab: comprehensive standard ood benchmark for imagenet-1k}, corresponding to CIFAR-10, CIFAR-100, ImageNet-200, and ImageNet-1K benchmarks.
Although our method does not necessarily achieve the best performance on every individual OOD dataset, it consistently demonstrates competitive and stable average performance across benchmarks.
In contrast to several existing baselines that exhibit strong dataset-specific biases or overfitting to particular visual statistics, our approach maintains robust detection capability across both smaller-scale CIFAR benchmarks and more complex ImageNet settings.
These results highlight the strong generalization and practical reliability of the proposed framework.
\begin{table*}[h]
    \centering
    \caption{
        OOD detection results of AUROC($\uparrow$) / FPR95($\downarrow$) on the CIFAR-10 benchmark.
        \textbf{Bold} entries indicate the best results; \underline{underlined} entries denote the second- and third-best results.
        }
    \vspace{-0.75em}
    \resizebox{\linewidth}{!}{%
        \begin{sc}
        \begin{tabular}{l @{\hspace{0.5em}}|@{\hspace{0.5em}} c@{\hspace{1em}}c@{\hspace{1em}}c@{\hspace{1em}}c@{\hspace{1em}}c@{\hspace{1em}}c @{\hspace{0.5em}}|@{\hspace{0.5em}} c} 
            \toprule
            \multirow{2}{*}{\textbf{Methods}} & \textbf{CIFAR-100} & \textbf{TIN} & \textbf{MNIST} & \textbf{SVHN} & \textbf{Textures} & \textbf{Places365} & \textbf{Average} \\
            \cmidrule(lr){2-2} \cmidrule(lr){3-3} \cmidrule(lr){4-4} \cmidrule(lr){5-5} \cmidrule(lr){6-6} \cmidrule(lr){7-7} \cmidrule(lr){8-8}
            & AUROC($\uparrow$) / FPR95($\downarrow$) & AUROC($\uparrow$) / FPR95($\downarrow$) & AUROC($\uparrow$) / FPR95($\downarrow$) & AUROC($\uparrow$) / FPR95($\downarrow$) & AUROC($\uparrow$) / FPR95($\downarrow$) & AUROC($\uparrow$) / FPR95($\downarrow$) & AUROC($\uparrow$) / FPR95($\downarrow$) \\ 
            \midrule
            
            \underline{\textbf{Baselines}}  \\
            
            OpenMax~\citep{openmax16cvpr} & 86.91 / 48.06 & 88.32 / 39.18 & 90.50 / 23.33 & 89.77 / 25.40 & 89.58 / 31.50 & 88.63 / 38.52 & 88.95 / 34.33 \\
            MSP~\citep{MSP} & 87.19 / 53.08 & 88.87 / 43.27 & 92.63 / 23.64 & 91.46 / 25.82 & 89.89 / 34.96 & 88.92 / 42.47 & 89.83 / 37.21 \\
            TempScale~\citep{tempscale} & 87.17 / 55.81 & 89.00 / 46.11 & 93.11 / 23.53 & 91.66 / 26.97 & 90.01 / 38.16 & 89.11 / 45.27 & 90.01 / 39.31 \\
            ODIN~\citep{ODIN} & 82.18 / 77.00 & 83.55 / 75.38 & \textbf{95.27} / 23.83 & 84.58 / 68.61 & 86.94 / 67.70 & 85.07 / 70.36 & 86.26 / 63.81 \\
            MDS~\citep{MDS} & 83.59 / 52.81 & 84.81 / 46.99 & 90.10 / 27.30 & 91.18 / 25.96 & 92.69 / 27.94 & 84.90 / 47.67 & 87.88 / 38.11 \\
            RMDS~\citep{RMDS} & 88.83 / \underline{43.86} & 90.76 / \underline{33.91} & 93.22 / 21.49 & \underline{91.84} / \underline{23.46} & 92.23 / 25.25 & 91.51 / 31.20 & 91.40 / \underline{29.86} \\
            EBO~\citep{Energy} & 86.36 / 66.60 & 88.80 / 56.08 & 94.32 / 24.99 & 91.79 / 35.12 & 89.47 / 51.82 & 89.25 / 54.85 & 90.00 / 48.24 \\
            ReAct~\citep{sun2021react} & 85.93 / 67.40 & 88.29 / 59.71 & 92.81 / 33.77 & 89.12 / 50.23 & 89.38 / 51.42 & 90.35 / 44.20 & 89.31 / 51.12 \\
            MLS~\citep{MLS} & 86.31 / 66.59 & 88.72 / 56.06 & 94.15 / 25.06 & 91.69 / 35.09 & 89.41 / 51.73 & 89.14 / 54.84 & 89.90 / 48.23 \\
            VIM~\citep{haoqi2022vim} & 87.75 / 49.19 & 89.62 / 40.49 & \underline{94.76} / \textbf{18.36} & \textbf{94.50} / \textbf{19.29} & \textbf{95.15} / \textbf{21.14} & 89.49 / 41.43 & 91.88 / 31.65 \\
            KNN~\citep{sun2022out} & \textbf{89.73} / \textbf{37.64} & \underline{91.56} / \textbf{30.37} & 94.26 / \underline{20.05} & \underline{92.67} / \underline{22.60} & 93.16 / \underline{24.06} & \underline{91.77} / \underline{30.38} & \underline{92.19} / \textbf{27.52} \\
            ASH~\citep{ASH} & 74.11 / 87.31 & 76.44 / 86.25 & 83.16 / 70.00 & 73.46 / 83.64 & 77.45 / 84.59 & 79.89 / 77.89 & 77.42 / 81.61 \\
            SHE~\citep{SHE} & 80.31 / 81.00 & 82.76 / 78.30 & 90.43 / 42.22 & 86.38 / 62.74 & 81.57 / 84.60 & 82.89 / 76.36 & 84.06 / 70.87 \\

            \midrule
            \textbf{\underline{Fixed $\alpha$}}\\

            \fold & \underline{89.45} / \underline{44.07} & \textbf{91.83} / \underline{32.77} & \underline{94.66} / \underline{21.18} & 91.56 / 28.50 & \underline{93.80} / \underline{24.17} & \textbf{93.46} / \textbf{25.87} & \textbf{92.46} / \underline{29.43} \\
            \fold-r & \underline{89.29} / 46.04 & \underline{91.66} / 33.96 & 94.26 / 23.08 & 91.47 / 30.36 & \underline{93.68} / 24.27 & \underline{93.42} / \underline{26.16} & \underline{92.30} / 30.65 \\      
            \fold-a & 86.31 / 62.94 & 88.94 / 53.52 & 91.44 / 45.77 & 86.79 / 57.51 & 92.21 / 34.17 & 90.91 / 41.19 & 89.43 / 49.18 \\      

            \midrule
            \textbf{\underline{Auto $\alpha$ tuning}}\\

            \autofold & \underline{89.45} / \underline{44.07} & \textbf{91.83} / \underline{32.77} & \underline{94.66} / \underline{21.18} & 91.56 / 28.50 & \underline{93.80} / \underline{24.17} & \textbf{93.46} / \textbf{25.87} & \textbf{92.46} / \underline{29.43} \\
            \autofold-r & \underline{89.29} / 46.04 & \underline{91.66} / 33.96 & 94.26 / 23.08 & 91.47 / 30.36 & \underline{93.68} / 24.27 & \underline{93.42} / \underline{26.16} & \underline{92.30} / 30.65 \\      
            \autofold-a & 86.31 / 62.94 & 88.94 / 53.52 & 91.44 / 45.77 & 86.79 / 57.51 & 92.21 / 34.17 & 90.91 / 41.19 & 89.43 / 49.18 \\        

            \bottomrule
        \end{tabular}
        \end{sc}
    }
    \label{tab: comprehensive standard ood benchmark for cifar-10}
\end{table*}

\begin{table*}[t]
    \centering
    \caption{
        OOD detection results of AUROC($\uparrow$) / FPR95($\downarrow$) on the CIFAR-100 benchmark.
        }
    \vspace{-0.75em}
    \resizebox{\linewidth}{!}{%
        \begin{sc}
        \begin{tabular}{l @{\hspace{0.5em}}|@{\hspace{0.5em}} c@{\hspace{1em}}c@{\hspace{1em}}c@{\hspace{1em}}c@{\hspace{1em}}c@{\hspace{1em}}c @{\hspace{0.5em}}|@{\hspace{0.5em}} c} 
            \toprule
            \multirow{2}{*}{\textbf{Methods}} & \textbf{CIFAR-10} & \textbf{TIN} & \textbf{MNIST} & \textbf{SVHN} & \textbf{Textures} & \textbf{Places365} & \textbf{Average} \\
            \cmidrule(lr){2-2} \cmidrule(lr){3-3} \cmidrule(lr){4-4} \cmidrule(lr){5-5} \cmidrule(lr){6-6} \cmidrule(lr){7-7} \cmidrule(lr){8-8}
            & AUROC($\uparrow$) / FPR95($\downarrow$) & AUROC($\uparrow$) / FPR95($\downarrow$) & AUROC($\uparrow$) / FPR95($\downarrow$) & AUROC($\uparrow$) / FPR95($\downarrow$) & AUROC($\uparrow$) / FPR95($\downarrow$) & AUROC($\uparrow$) / FPR95($\downarrow$) & AUROC($\uparrow$) / FPR95($\downarrow$) \\ 
            \midrule
            
            \underline{\textbf{Baselines}}  \\
            
            OpenMax~\citep{openmax16cvpr} & 74.38 / 60.17 & 78.44 / 52.99 & 76.01 / 53.82 & 82.07 / 53.20 & 80.56 / 56.12 & 79.29 / \underline{54.85} & 78.46 / 55.19 \\
            MSP~\citep{MSP} & 78.47 / \underline{58.91} & 82.07 / 50.70 & 76.08 / 57.23 & 78.42 / 59.07 & 77.32 / 61.88 & 79.22 / 56.62 & 78.60 / 57.40 \\
            TempScale~\citep{tempscale} & \underline{79.02} / \textbf{58.72} & 82.79 / \underline{50.26} & 77.27 / 56.05 & 79.79 / 57.71 & 78.11 / 61.56 & 79.80 / 56.46 & 79.46 / 56.79 \\
            ODIN~\citep{ODIN} & 78.18 / 60.64 & 81.63 / 55.19 & \underline{83.79} / \underline{45.94} & 74.54 / 67.41 & 79.33 / 62.37 & 79.45 / 59.71 & 79.49 / 58.54 \\
            MDS~\citep{MDS} & 55.87 / 88.00 & 61.50 / 79.05 & 67.47 / 71.72 & 70.68 / 67.21 & 76.26 / 70.49 & 63.15 / 79.61 & 65.82 / 76.01 \\
            RMDS~\citep{RMDS} & 77.75 / 61.37 & 82.55 / \textbf{49.56} & 79.74 / 52.05 & 84.89 / 51.65 & 83.65 / 53.99 & \textbf{83.40} / \textbf{53.57} & \underline{82.00} / \underline{53.70} \\
            EBO~\citep{Energy} & \underline{79.05} / 59.21 & 82.76 / 52.03 & 79.18 / 52.62 & 82.03 / 53.62 & 78.35 / 62.35 & 79.52 / 57.75 & 80.15 / 56.26 \\
            ReAct~\citep{sun2021react} & 78.65 / 61.30 & 82.88 / 51.47 & 78.37 / 56.04 & 83.01 / 50.41 & 80.15 / 55.04 & \underline{80.03} / 55.30 & 80.52 / 54.93 \\
            MLS~\citep{MLS} & \textbf{79.21} / \underline{59.11} & 82.90 / 51.83 & 78.91 / 52.95 & 81.65 / 53.90 & 78.39 / 62.39 & 79.75 / 57.68 & 80.14 / 56.31 \\
            VIM~\citep{haoqi2022vim} & 72.21 / 70.59 & 77.76 / 54.66 & 81.89 / 48.32 & 83.14 / 46.22 & \textbf{85.91} / \underline{46.86} & 75.85 / 61.57 & 79.46 / 54.70 \\
            KNN~\citep{sun2022out} & 77.02 / 72.80 & \textbf{83.34} / \underline{49.65} & 82.36 / 48.58 & 84.15 / 51.75 & 83.66 / 53.56 & 79.43 / 60.70 & 81.66 / 56.17 \\
            ASH~\citep{ASH} & 76.48 / 68.06 & 79.92 / 63.35 & 77.23 / 66.58 & 85.60 / 46.00 & 80.72 / 61.27 & 78.76 / 62.95 & 79.79 / 61.37 \\
            SHE~\citep{SHE} & 78.15 / 60.41 & 79.74 / 57.74 & 76.76 / 58.78 & 80.97 / 59.15 & 73.64 / 73.29 & 76.30 / 65.24 & 77.59 / 62.44 \\

            \midrule
            \textbf{\underline{Fixed $\alpha$}}\\

            \fold & 78.05 / 60.70 & \underline{83.26} / 50.51 & 82.27 / 47.04 & 85.65 / 44.54 & 83.05 / 50.81 & 79.37 / 56.09 & 81.94 / \underline{51.62} \\
            \fold-r & 77.49 / 63.67 & \underline{83.27} / 50.77 & 81.84 / 49.13 & \underline{86.58} / 40.53 & \underline{84.58} / \textbf{44.21} & 79.58 / 54.97 & \textbf{82.22} / \textbf{50.55} \\   
            \fold-a & 77.08 / 68.28 & 82.45 / 55.64 & 80.64 / 56.50 & \underline{87.16} / \underline{39.91} & 84.34 / 48.95 & \underline{80.21} / 56.29 & \underline{81.98} / 54.26 \\      

            \midrule
            \textbf{\underline{Auto $\alpha$ tuning}}\\

            \autofold & 70.96 / 79.66 & 78.80 / 68.66 & \underline{83.77} / \underline{42.72} & \textbf{87.42} / \textbf{32.93} & \underline{85.18} / \underline{45.14} & 72.42 / 76.03 & 79.76 / 57.52 \\      
            \autofold-r & 77.73 / 63.88 & 82.81 / 51.13 & 74.84 / 61.67 & 80.64 / 51.81 & 80.95 / 51.81 & 79.41 / \underline{53.97} & 79.40 / 55.71 \\      
            \autofold-a & 70.73 / 86.15 & 77.43 / 76.44 & \textbf{85.50} / \textbf{41.54} & 85.19 / \underline{39.73} & 82.31 / 56.20 & 69.49 / 82.98 & 78.44 / 63.84 \\      
            \bottomrule
        \end{tabular}
        \end{sc}
    }
    \label{tab: comprehensive standard ood benchmark for cifar-100}
\end{table*}


\begin{table*}[t]
    \centering
    \caption{
        OOD detection results of AUROC($\uparrow$) / FPR95($\downarrow$) on the ImageNet-200 benchmark. 
    }
    \vspace{-0.75em}
    \resizebox{\linewidth}{!}{%
        \begin{sc}
        \begin{tabular}{l @{\hspace{0.5em}}|@{\hspace{0.5em}} c@{\hspace{1em}}c@{\hspace{1em}}c@{\hspace{1em}}c@{\hspace{1em}}c @{\hspace{0.5em}}|@{\hspace{0.5em}} c} 
            \toprule
            \multirow{2}{*}{\textbf{Methods}} & \textbf{SSB-Hard} & \textbf{NINCO} & \textbf{iNaturalist} & \textbf{Textures} & \textbf{OpenImage-O} & \textbf{Average} \\
            \cmidrule(lr){2-2} \cmidrule(lr){3-3} \cmidrule(lr){4-4} \cmidrule(lr){5-5} \cmidrule(lr){6-6} \cmidrule(lr){7-7}
            & AUROC($\uparrow$) / FPR95($\downarrow$) & AUROC($\uparrow$) / FPR95($\downarrow$) & AUROC($\uparrow$) / FPR95($\downarrow$) & AUROC($\uparrow$) / FPR95($\downarrow$) & AUROC($\uparrow$) / FPR95($\downarrow$) & AUROC($\uparrow$) / FPR95($\downarrow$) \\ 
            \midrule
            
            \underline{\textbf{Baselines}}  \\
            
            OpenMax \citep{openmax16cvpr} & 77.53 / 72.37 & 83.01 / 54.59 & 92.32 / 24.53 & 90.21 / 36.80 & 88.07 / 38.03 & 86.23 / 45.26 \\
            MSP \citep{MSP} & 80.38 / \underline{66.00} & 86.29 / 43.65 & 92.80 / 26.48 & 88.36 / 44.58 & 89.24 / 35.23 & 87.41 / 43.19 \\
            TempScale \citep{tempscale} & \underline{80.71} / \underline{66.43} & \underline{86.67} / \underline{43.21} & 93.39 / 24.39 & 89.24 / 43.57 & 89.84 / 34.04 & 87.97 / 42.33 \\
            ODIN \citep{ODIN} & 77.19 / 73.51 & 83.34 / 60.00 & 94.37 / 22.39 & 90.65 / 42.99 & 90.11 / 37.30 & 87.13 / 47.24 \\
            MDS \citep{MDS} & 58.38 / 83.65 & 65.48 / 74.57 & 75.03 / 58.53 & 79.25 / 58.16 & 69.87 / 68.29 & 69.60 / 68.64 \\
            RMDS \citep{RMDS} & 80.20 / \textbf{65.91} & 84.94 / \textbf{42.13} & 90.64 / 24.70 & 86.77 / 37.80 & 86.77 / 34.85 & 85.86 / 41.08 \\
            EBO \citep{Energy} & 79.83 / 69.77 & 85.17 / 50.70 & 92.55 / 26.41 & 90.79 / 41.43 & 89.23 / 36.74 & 87.51 / 45.01 \\
            ReAct \citep{sun2021react} & 78.97 / 71.51 & 84.76 / 53.47 & 93.65 / 22.97 & 92.86 / 29.67 & 90.40 / 32.86 & 88.13 / 42.10 \\
            MLS \citep{MLS} & 80.15 / 69.64 & 85.65 / 49.87 & 93.12 / 25.09 & 90.60 / 41.25 & 89.62 / 35.76 & 87.83 / 44.32 \\
            VIM \citep{haoqi2022vim} & 74.04 / 71.28 & 83.32 / 47.10 & 90.96 / 27.34 & 94.61 / \underline{20.39} & 88.20 / 33.86 & 86.23 / 39.99 \\
            KNN \citep{sun2022out} & 77.03 / 73.71 & 86.10 / 46.64 & 93.99 / 24.46 & \underline{95.29} / 24.45 & 90.19 / 32.90 & 88.52 / 40.43 \\
            ASH \citep{ASH} & 79.52 / 72.14 & 85.24 / 57.63 & \textbf{95.10} / 22.49 & 94.77 / 25.65 & \underline{91.82} / 33.72 & \underline{89.29} / 42.33 \\
            SHE \citep{SHE} & 78.30 / 72.64 & 82.07 / 60.96 & 91.43 / 34.38 & 90.51 / 45.58 & 87.49 / 46.54 & 85.96 / 52.02 \\

            \midrule
            \textbf{\underline{Fixed $\alpha$}}\\

            \fold & \textbf{80.92} / \underline{66.43} & \textbf{87.14} / \underline{42.23} & 94.02 / 20.89 & 90.74 / 40.79 & 90.45 / 31.18 & 88.65 / 40.30 \\
            \fold-r & 79.89 / 69.85 & 86.36 / 47.96 & 94.79 / \textbf{18.76} & 92.93 / 29.05 & 91.38 / \underline{28.99} & 89.07 / \underline{38.92} \\        
            \fold-a & 79.25 / 70.38 & 86.49 / 47.20 & \underline{95.09} / \underline{19.14} & \underline{95.45} / \underline{21.28} & \textbf{92.14} / \textbf{27.29} & \textbf{89.68} / \textbf{37.06} \\      

            \midrule
            \textbf{\underline{Auto $\alpha$ tuning}}\\

            \autofold & \underline{80.46} / 67.62 & 86.58 / 44.97 & 93.49 / 22.43 & 91.46 / 37.40 & 89.86 / 31.68 & 88.37 / 40.82 \\      
            \autofold-r & 79.48 / 69.36 & \underline{86.63} / 45.10 & 94.87 / 20.77 & 92.83 / 31.90 & 91.34 / 30.11 & 89.03 / 39.45 \\      
            \autofold-a & 78.67 / 70.76 & 85.86 / 50.20 & \underline{94.96} / \underline{20.28} & \textbf{95.76} / \textbf{20.15} & \underline{91.95} / \underline{28.44} & \underline{89.44} / \underline{37.97} \\      

            \bottomrule
        \end{tabular}
        \end{sc}
    }
    \label{tab: comprehensive standard ood benchmark for imagenet-200}
\end{table*}

\begin{table*}[t]
    \centering
    \caption{
        OOD detection results of AUROC($\uparrow$) / FPR95($\downarrow$) on the ImageNet-1K benchmark.
    }
    \vspace{-0.75em}
    \resizebox{\linewidth}{!}{%
        \begin{sc}
        \begin{tabular}{l @{\hspace{0.5em}}|@{\hspace{0.5em}} c@{\hspace{1em}}c@{\hspace{1em}}c@{\hspace{1em}}c@{\hspace{1em}}c @{\hspace{0.5em}}|@{\hspace{0.5em}} c} 
            \toprule
            \multirow{2}{*}{\textbf{Methods}} & \textbf{SSB-Hard} & \textbf{NINCO} & \textbf{iNaturalist} & \textbf{Textures} & \textbf{OpenImage-O} & \textbf{Average} \\
            \cmidrule(lr){2-2} \cmidrule(lr){3-3} \cmidrule(lr){4-4} \cmidrule(lr){5-5} \cmidrule(lr){6-6} \cmidrule(lr){7-7}
            & AUROC($\uparrow$) / FPR95($\downarrow$) & AUROC($\uparrow$) / FPR95($\downarrow$) & AUROC($\uparrow$) / FPR95($\downarrow$) & AUROC($\uparrow$) / FPR95($\downarrow$) & AUROC($\uparrow$) / FPR95($\downarrow$) & AUROC($\uparrow$) / FPR95($\downarrow$) \\ 
            \midrule
            
            \underline{\textbf{Baselines}}  \\
            
            OpenMax~\citep{openmax16cvpr} & 71.37 / 77.33 & 78.17 / 60.81 & 92.05 / 25.29 & 88.10 / 40.26 & 87.62 / 37.39 & 83.46 / 48.22 \\
            MSP~\citep{MSP} & 72.09 / 74.49 & 79.95 / 56.88 & 88.41 / 43.34 & 82.43 / 60.87 & 84.86 / 50.13 & 81.55 / 57.14 \\
            TempScale~\citep{tempscale} & 72.87 / \underline{73.90} & 81.41 / 55.10 & 90.50 / 37.63 & 84.95 / 56.90 & 87.22 / 45.40 & 83.39 / 53.79 \\
            ODIN~\citep{ODIN} & 71.74 / 76.83 & 77.77 / 68.16 & 91.17 / 35.98 & 89.00 / 49.24 & 88.23 / 46.67 & 83.58 / 55.38 \\
            MDS~\citep{MDS} & 48.50 / 92.10 & 62.38 / 78.80 & 63.67 / 73.81 & 89.80 / 42.79 & 69.27 / 72.15 & 66.72 / 71.93 \\
            RMDS~\citep{RMDS} & 71.77 / 77.88 & 82.22 / 52.20 & 87.24 / 33.67 & 86.08 / 48.80 & 85.84 / 40.27 & 82.63 / 50.56 \\
            EBO~\citep{Energy} & 72.08 / 76.54 & 79.70 / 60.58 & 90.63 / 31.30 & 88.70 / 45.77 & 89.06 / 38.09 & 84.03 / 50.46 \\
            ReAct~\citep{sun2021react} & \underline{73.03} / 77.55 & 81.73 / 55.82 & 96.34 / 16.72 & 92.79 / 29.64 & 91.87 / 32.58 & 87.15 / 42.46 \\
            MLS~\citep{MLS} & 72.51 / 76.20 & 80.41 / 59.44 & 91.17 / 30.61 & 88.39 / 46.17 & 89.17 / 37.88 & 84.33 / 50.06 \\
            VIM~\citep{haoqi2022vim} & 65.54 / 80.41 & 78.63 / 62.29 & 89.56 / 30.68 & \textbf{97.97} / \textbf{10.51} & 90.50 / 32.82 & 84.44 / 43.34 \\
            KNN~\citep{sun2022out} & 62.57 / 83.36 & 79.64 / 58.39 & 86.41 / 40.80 & \underline{97.09} / 17.31 & 87.04 / 44.27 & 82.55 / 48.83 \\
            ASH~\citep{ASH} & 72.89 / \underline{73.66} & \underline{83.45} / 52.97 & \textbf{97.07} / \textbf{14.04} & \underline{96.90} / \underline{15.26} & \underline{93.26} / \underline{29.15} & \textbf{88.71} / \underline{37.02} \\
            SHE~\citep{SHE} & 71.08 / 76.30 & 76.49 / 69.72 & 92.65 / 34.06 & 93.60 / 35.27 & 86.52 / 55.02 & 84.07 / 54.07 \\

            \midrule
            \textbf{\underline{Fixed $\alpha$}}\\

            \fold & \underline{73.95} / \textbf{73.37} & 82.83 / 52.16 & 91.96 / 29.92 & 89.93 / 49.70 & 90.04 / 36.62 & 85.74 / 48.35 \\
            \fold-r & \textbf{74.45} / 76.60 & \underline{83.45} / \underline{49.69} & 95.99 / 18.09 & 94.52 / 25.74 & 91.78 / 31.68 & \underline{88.04} / 40.36 \\    
            \fold-a & 71.19 / 76.66 & \textbf{83.56} / \textbf{48.78} & \underline{96.87} / \underline{14.67} & 96.38 / 16.97 & \textbf{93.38} / \underline{28.07} & \underline{88.28} / \underline{37.03} \\      

            \midrule
            \textbf{\underline{Auto $\alpha$ tuning}}\\

            \autofold & 71.74 / 74.10 & 81.38 / 52.68 & 91.06 / 26.72 & 88.39 / 46.51 & 89.35 / 34.63 & 84.38 / 46.93 \\      
            \autofold-r & 69.17 / 78.49 & 81.35 / 50.77 & 88.68 / 24.74 & 91.09 / 31.74 & 88.95 / 31.21 & 83.85 / 43.39 \\      
            \autofold-a & 69.53 / 76.44 & 82.03 / \underline{50.46} & \underline{96.42} / \underline{15.15} & 96.83 / \underline{14.20} & \underline{93.21} / \textbf{27.78} & 87.60 / \textbf{36.81} \\      
            
            \bottomrule
        \end{tabular}
        \end{sc}
    }
    \label{tab: comprehensive standard ood benchmark for imagenet-1k}
\end{table*}
\end{document}